\documentclass[english]{article}
\usepackage{geometry}

\usepackage[T1]{fontenc}
\usepackage[latin9]{inputenc}
\usepackage{bm}
\usepackage{amsmath,mathtools}
\usepackage{amssymb,xspace}
\usepackage[unicode=true,
 bookmarks=false,
 breaklinks=false,pdfborder={0 0 1},colorlinks=false]
 {hyperref}
\hypersetup{
 colorlinks,citecolor=blue,filecolor=blue,linkcolor=blue,urlcolor=blue}

\makeatletter
\usepackage{amsthm}
\usepackage{comment}
\usepackage{natbib}
\usepackage{booktabs}

\usepackage{graphicx}

\usepackage{algorithmic}
\usepackage[linesnumbered,ruled,vlined]{algorithm2e}

\usepackage{float}
\usepackage{multirow}
 
\usepackage{dsfont}
\usepackage{tcolorbox}
\usepackage{makecell}

\usepackage{color}
\definecolor{yxc}{RGB}{255,0,0}
\definecolor{yjc}{RGB}{125,0,0}
\definecolor{ytw}{RGB}{255,69,0}
\definecolor{gen}{RGB}{0,0,200}

\allowdisplaybreaks

\newcommand{\defn}{\coloneqq}

\newcommand{\pddim}{\texttt{Posterior-DDIM}\xspace}

\newcommand{\St}{\mathcal S_t^{\mathsf{meas}}\xspace}
\newcommand{\Sti}[1]{\mathcal{S}_{#1}^{\mathsf{meas}}\xspace}

\newcommand{\Stc}{\mathcal S_t^{\mathsf{prior}}\xspace}
\newcommand{\Stci}[1]{\mathcal{S}_{#1}^{\mathsf{prior}}\xspace}

\newcommand{\cS}{\mathcal{S}}

\newcommand{\bR}{\mathbb{R}}

\newcommand{\mymid}{\,|\,}

\newcommand{\veps}{\varepsilon}
\newcommand{\disteq}{\overset{\mathrm{d}}{=}}
\newcommand{\wh}{\widehat}

\definecolor{yanxi}{RGB}{0,200,100}

\definecolor{yuchen}{RGB}{0,200,100}

\title{Provable diffusion-based posterior sampling \\ for linear inverse problems via DDIM}

\author{
Yuchen Jiao\footnote{The authors contributed equally. Corresponding author: Gen Li.}~\thanks{Department of Statistics and Data Science, Chinese University of Hong Kong, Hong Kong
}
\and 
Na Li\footnotemark[1]~\footnotemark[2] 
\and 
Changxiao Cai\thanks{Department of Industrial and Operations Engineering, University of Michigan, Ann Arbor, USA 
}
\and Yuxin Chen\thanks{Department of Statistics and Data Science, the Wharton School, University of Pennsylvania,  PA, USA
}
\and
Gen Li\footnotemark[2]
}
\date{\today}

\makeatother

\begin{document}

\theoremstyle{plain} \newtheorem{lemma}{\textbf{Lemma}}\newtheorem{proposition}{\textbf{Proposition}}\newtheorem{theorem}{\textbf{Theorem}}

\theoremstyle{definition}\newtheorem{definition}{\textbf{Definition}}
\theoremstyle{remark}\newtheorem{remark}{\textbf{Remark}}
\theoremstyle{plain}\newtheorem{corollary}{\textbf{Corollary}}

\maketitle 

\begin{abstract}
    Diffusion-based methods have achieved remarkable empirical success in solving inverse problems.  However, many existing posterior samplers either lack rigorous theoretical guarantees or incur substantial computational overhead. We propose a simple and efficient algorithm, called \pddim, for solving linear inverse problems with diffusion priors via a DDIM-type sampler. Our method requires only lightweight, coordinate-wise modifications to the standard DDIM update, while explicitly incorporating the measurement model. The key idea is to perform posterior sampling separately along each singular direction of the measurement operator: for each direction, the sampler follows the learned diffusion prior when the observation signal-to-noise ratio (SNR) is below the corresponding diffusion SNR, and switches to a calibrated measurement-based predictor otherwise. We prove that the proposed sampler converges to the Bayesian posterior conditioned on the measurements. 
    Empirical results show that the proposed sampler performs favorably against existing diffusion-based posterior samplers across a range of image restoration tasks, achieving the best performance on the majority of evaluation metrics considered. Overall, our results convert posterior sampling for noisy linear inverse problems to simple coordinate-wise DDIM updates, yielding an efficient, easy-to-implement algorithm with provable posterior consistency.

\end{abstract}

\noindent \textbf{Keywords:} diffusion models; linear inverse problems; DDIM; posterior sampling

    %

\tableofcontents

\section{Introduction}

Diffusion models \citep{sohl2015deep,ho2020denoising,song2020score,lai2025principles} have emerged as one of the most successful paradigms for modern generative modeling, achieving remarkable performance across a wide range of domains, from image synthesis, video generation, to scientific data generation. Beyond sample generation, their ability to capture complex, high-dimensional data distributions makes them particularly well-suited as flexible, expressive, and fully data-driven priors over the signals of interest, offering a versatile framework for Bayesian inference in inverse problems \citep{song2021solving,chung2022come,chung2022diffusion,kawar2022denoising,song2023pseudoinverse,rout2023solving,daras2024survey}.

\subsection{Motivation: inverse problems with diffusion priors}

In an inverse problem, one is asked to infer an unknown random signal $X_0\in \mathbb{R}^d$ 
 from noisy observations
\begin{align}
Y= \mathcal{A}(X_0)+\varepsilon,
\end{align}
where $\mathcal{A}: \mathbb{R}^d\to \mathbb{R}^n$ is a known measurement operator,  and $\varepsilon\in \mathbb{R}^n$ represents the measurement noise. Such problems are often ill-posed: even in the absence of noise, the observations $y$ alone may be insufficient to uniquely determine the underlying signal $X_0$, particularly when $n<d$. A principled approach to resolving this ambiguity and improving statistical performance is provided by Bayesian inference, which combines the measurement model with a prior distribution over the unknown signal $X_0$. Specifically, given a prior $p_{X_0}$, the central objective of Bayesian inference is to characterize, or sample from, the posterior distribution $p_{X_0\mid y}$.

The effectiveness of Bayesian inference therefore depends largely on the choice of prior. 
Classical approaches typically impose hand-crafted priors, such as sparsity or low-rank structure \citep{daubechies2004iterative,fan2001variable,candes2006robust,donoho2006compressed,park2008bayesian,ji2008bayesian,babacan2009bayesian,donoho2009message,candes2012exact,candes2014mathematics,rovckova2018spike,chi2019nonconvex}.  While these priors have led to many important advances, they often oversimplify the structure of natural signals. Diffusion models offer a fundamentally different alternative: rather than manually specifying a prior or regularizer, they use a pre-trained generative model as an implicit prior over the signal distribution.  This perspective has inspired a rapidly growing line of work on posterior sampling for inverse problems with diffusion priors  \citep{song2021solving,chung2022diffusion,kawar2022denoising,song2023pseudoinverse,xu2024provably,wang2024dmplug,zhang2025improving}. Moreover, because pre-trained diffusion models can be used as ``plug-and-play'' priors \citep{venkatakrishnan2013plug,sreehari2016plug,chan2017plug,xu2024provably}, they often require no task-specific retraining and readily transfer across sensing modalities and application domains.

Using diffusion models as priors, however, presents a new suite of algorithmic challenges.  Rather than assuming access to an explicit, closed-form expression for the prior distribution, the diffusion-prior framework models the prior implicitly through a pre-trained diffusion model that provides score estimates (i.e., estimates of the Stein score functions) for progressively diffused versions of the prior.  A central algorithmic challenge is therefore to combine the learned diffusion prior with the measurement model \eqref{eq:inverse} in order to generate samples from the posterior distribution. Existing diffusion-based methods typically address this through heuristic correction steps, projection operators, likelihood-gradient guidance, approximate posterior updates, and so on  \citep{chung2022come,chung2022score,choi2021ilvr,kawar2022denoising}. Although these methods have demonstrated strong empirical performance, many lack rigorous guarantees that the generated samples converge to the true Bayesian posterior. Furthermore, some require computationally intensive modifications to the underlying sampling procedure. 
For example, \citet{chung2022diffusion,song2023pseudoinverse} required computing the gradient of the score functions, while the sequential Monte Carlo approach of \citet{cardoso2023monte} may require a prohibitively large number of particles. These considerations can limit their practical applicability in settings where both computational efficiency and statistical reliability are important.

These issues highlight the need for diffusion-based inverse problem solvers that achieve rigorous theoretical guarantees and computational efficiency at once. Such properties are particularly important in high-stakes applications such as medical imaging and scientific discovery. At the same time, practical deployment often favors algorithms that preserve the simplicity and efficiency of standard diffusion samplers.



\subsection{An overview of our contributions}

In this paper, we make progress towards developing a theoretically principled, and computationally efficient framework for posterior sampling in {\em linear} inverse problems with diffusion priors. Specifically, we consider the following noisy linear measurement model: 
\begin{align}
    Y = A X_0 + \varepsilon,
    \label{eq:inverse}
\end{align}
where $X_0\in\mathbb R^d$ is the unknown signal (generated randomly from a data distribution $p_0$), $A:\mathbb R^d\to\mathbb R^n$ represents a linear measurement operator known {\em a priori},  and $\varepsilon \sim \mathcal N(0,\sigma^2 I_n)$ stands for Gaussian measurement noise with $\sigma>0$ the noise standard deviation. Throughout, we assume that the prior distribution of $X_0$ can be represented implicitly by a pre-trained diffusion model.  The objective is to generate samples from the posterior distribution of $X_0$ conditioned on the observation $y$. 

Our starting point is the observation that, under a linear measurement model, different singular directions of the measurement operator carry different amounts of statistical information. Rather than treating all directions uniformly, we design a DDIM-type posterior sampler, called \pddim,  that adaptively determines, for each singular direction, whether the update should be based primarily on the learned diffusion prior or the observed data. This results in a simple, lightweight, coordinate-wise modification of the standard DDIM sampler that explicitly exploits the structure of the measurement operator, while preserving the efficiency of DDIM-type sampling. Along directions where the measurements have sufficient signal-to-noise ratio (SNR), the algorithm injects a carefully calibrated noisy version of the observations so that the resulting coordinates have the correct forward-diffusion marginal distribution. In contrast, directions that are either observed with low SNR or entirely unobserved continue to evolve according to the standard DDIM dynamics induced by the prior. Consequently, posterior sampling reduces to a sequence of simple, direction-wise DDIM updates that are both computationally efficient and easy to implement.

Theoretically, our main contribution is to establish posterior consistency of the proposed \pddim sampler. Under an exact diffusion prior and a sufficiently fine time discretization, we prove that the sampler converges to the posterior distribution of the unknown signal conditioned on the measurements. This result provides rigorous theoretical support for DDIM-type samplers in linear inverse problems, and suggests a principled approach to integrating measurement information with diffusion priors in noisy, potentially ill-conditioned or rank-deficient settings.

Empirically, we evaluate the proposed sampler on a range of image restoration tasks, including inpainting, super-resolution, Gaussian deblurring, and compressed sensing, using PSNR, SSIM, LPIPS, and FID as evaluation metrics. Across these tasks, our \pddim algorithm performs favorably against existing diffusion-based posterior samplers: it achieves the best performance on at least three of the four metrics in most tasks, and on at least two metrics in every task considered.

\subsection{Related work}

Despite their significant empirical success, existing diffusion-based inverse solvers still leave room for improvement. 
One prominent class of methods \citep{song2020score,chung2022come,chung2022score,choi2021ilvr,kawar2022denoising,chung2023fast} enforces measurement consistency by incorporating projection steps into unconditional diffusion sampling procedures. While these approaches have demonstrated strong empirical performance, they do not provide theoretical guarantees on the convergence or posterior correctness of the resulting sampling procedures.
A second line of work \citep{chung2022diffusion,song2023pseudoinverse,kawar2021snips,cardoso2023monte,song2023loss} attempted to approximate the (\textit{intractable}) posterior score and subsequently apply standard diffusion samplers for posterior sampling. Although empirically effective,  it remains unclear whether the resulting sampler faithfully yields the true posterior distribution. 
In addition, \citet{song2021solving} proposed \texttt{ReSample}, which promotes measurement consistency by repeatedly optimizing the clean image estimate at intermediate time steps and then re-noising it to continue the reserve diffusion process. 
Similarly, \citet{graikos2023diffusion} proposed to iteratively optimize the clean image via gradient descent in the reverse diffusion process. 
A different perspective was introduced by \citet{mardani2023variational}, who formulated inverse problems with diffusion priors as a variational inference problem via direct KL minimization rather than reverse diffusion sampling. Like the aforementioned approaches, these methods currently lack rigorous theoretical guarantees. 
Another strand of work \citet{trippe2022diffusion,wu2023practical,dou2024diffusion,cardoso2023monte} combined Sequential Monte Carlo (MC) methods \citep{doucet2001introduction} with (unconditional) scores to sample from the correct posterior; in particular, \citet{cardoso2023monte,dou2024diffusion} established convergence guarantees as the number of particles tends to infinity.
However, these methods incur a prohibitive computational cost.
In particular, \citep{gupta2024diffusion} showed that polynomially many particles are in general insufficient for the above Sequential MC approaches to converge.
More recently, \citet{xu2024provably} proposed an alternating framework that accommodates even nonlinear inverse problems by interleaving two samplers: a proximal consistency sampler based on the likelihood function of the forward model and a denoising diffusion sampler driven by the learned score function of the prior. They also established rigorous theoretical guarantees for the resulting algorithm. In contrast, our proposed method is built almost entirely upon DDIM-type sampling.


Finally, we note that the convergence theory of DDPM-type samplers has been studied extensively in recent years (see, e.g., \citet{chen2023restoration,li2024sharp,huang2024convergence,liang2025low,tang2026adaptivity,cai2026diffusion,jiao2025optimal,li2024improved,jiao2024instance}), validating the efficacy of DDIM-type sampling and paving the way for more refined theoretical analyses in future work.

\section{Preliminaries}

In this section, we briefly review the background on diffusion models and their connections to inverse problems, which will be useful throughout the paper.

\paragraph{Diffusion models.}
We begin by reviewing the basics of diffusion models. Let \(X_0\sim p_0\) denote a random sample in $\mathbb{R}^d$ drawn from an underlying data distribution \(p_0\). A diffusion model defines a forward stochastic process \((X_t)_{t\in[0,T]}\) that progressively corrupts \(X_0\) by injecting Gaussian noise according to:
\begin{align}
    X_t \disteq \alpha_t X_0 + \sigma_t Z_t,
    \qquad
    Z_t\sim \mathcal N(0,I_d)
    \label{eq:forward-process}
\end{align}
for all $0\leq t\leq T$, where the coefficients \(\alpha_t\geq 0\) and \(\sigma_t\geq 0\) specify the noise schedule. Throughout this paper, we assume that the SNR $\alpha_t/\sigma_t$ decreasing monotonically and continuously as $t$ grows, and satisfies  $\alpha_0/\sigma_0 = \infty$ and $\alpha_\infty/\sigma_\infty = 0$. As \(t\) increases, the signal component gradually diminishes while the noise level increases, so that, under an appropriate noise schedule, the distribution of \(X_t\) transitions smoothly from the data distribution toward an approximately Gaussian distribution. A central goal of a diffusion model is to learn the corresponding reverse process, which starts from a Gaussian sample and progressively removes the injected noise to yield a sample from the data distribution of interest.

\paragraph{Score functions.}
A fundamental object governing the reverse process is the (Stein) score function
\begin{subequations}
\begin{align}
    s_t(x)
    \coloneqq 
    \nabla\log p_{X_t}(x),
    \qquad 0\leq t\leq T,
    \label{eq:defn-stein-score-function}
\end{align}
where \(p_{X_t}\) denotes the marginal distribution of \(X_t\). The score function specifies the direction in which the log-density of \(X_t\) increases most rapidly and plays a central role in reversing the forward process  \citep{anderson1982reverse,haussmann1986time}. In practice, the score functions are often parameterized through through equivalent prediction targets, such as the \emph{noise predictor} 
\begin{align}
    \epsilon_t(x)
    \coloneqq 
    \mathbb E[Z_t \mid X_t=x]   
\end{align}
and the {\em data predictor} 
\begin{align}
    \mu_t(x)
    \coloneqq 
    \mathbb E[X_0\mid X_t=x].
\end{align}
\end{subequations}
These objects are linked through Tweedie's formula \citep{efron2011tweedie}, which yields the identities
\begin{align}
    \epsilon_t(x) = -\sigma_t s_t(x)
    \qquad \text{and}
   \qquad 
    \mu_t(x)
    &=
    \frac{x-\sigma_t\epsilon_t(x)}{\alpha_t}.
    \label{eq:data-noise-predictor-relation}
\end{align}

\paragraph{Reparameterzed reverse process and  DDIM-type sampler.}
The forward diffusion process \eqref{eq:forward-process} admits a reverse-time stochastic differential equation (SDE), originally characterized by \citet{anderson1982reverse}, which forms the basis for diffusion-based generative sampling.  
It is sometimes convenient to reparameterize time $t$ using the following parameter concerning the log SNR:
\begin{align}
    \lambda_t \coloneqq \log \frac{\alpha_t}{\sigma_t}.
    \label{eq:connection-lambda-t}
\end{align}
Under this parameterization, the reverse SDE admits the following simple form using the $\lambda$-variable (i.e., the log SNR variable) as the ``time variable'' in place of $t$:
\begin{align}
    \mathrm d \bigg( \frac{X}{\alpha} \bigg)
    =
    -2 e^{-\lambda} \epsilon_{\lambda}(X)\,\mathrm d\lambda
    +
    \sqrt{2}e^{-\lambda}\,\mathrm d B_{\lambda},
    \label{eq:reverse-sde}
\end{align}
where \(B_\lambda\) denotes a standard Brownian motion indexed by $\lambda$, and we abuse the notation by letting $\epsilon_{\lambda}$ be a shorthand for $\epsilon_t$ with $\lambda = \lambda_t$.

Practical sampling procedures like DDPM-type samplers \citep{ho2020denoising} can be obtained by approximating this reverse SDE \eqref{eq:reverse-sde} through suitable time discretization. More precisely, for a sequence of decreasing time points
\[
    t_0 > t_1 > \cdots > t_M
\]
with corresponding log SNR values $\lambda_{t_0},\cdots,\lambda_{t_M}$ (cf.~\eqref{eq:connection-lambda-t}), 
a first-order discretization yields the DDPM-type update
\begin{align}
    \frac{X_{t_{i+1}}}{\alpha_{t_{i+1}}}
    &=
    \frac{X_{t_i}}{\alpha_{t_i}}
    -
    2\big(
        e^{-\lambda_{t_i}}
        -
        e^{-\lambda_{t_{i+1}}}
    \big)
    \epsilon_{t_i}(X_{t_i})
    +
    \sqrt{
        e^{-2\lambda_{t_i}}
        -
        e^{-2\lambda_{t_{i+1}}}
    }\,Z_{t_i},
    \label{eq:ddpm-noise-predictor}
\end{align}
where \(z_{t_i} \overset{\mathrm{i.i.d.}}{\sim}\mathcal N(0,I)\). In light of the relation~\eqref{eq:data-noise-predictor-relation} between the noise predictor and the data predictor, the same update can be equivalently expressed in terms of \(\mu_t\).  Letting
\begin{align}
    \delta_{i+1}
    \coloneqq
    \lambda_{t_i}-\lambda_{t_{i+1}}
    \label{eq:defn-delta-i}
\end{align}
with $\lambda_t$ defined in \eqref{eq:reverse-sde}, 
one obtains the first-order SDE-DPMSolver-1 update  \citep{lu2022dpm++}
\begin{align}
    \frac{X_{t_{i+1}}}{\sigma_{t_{i+1}}}
    &=
    e^{-\delta_{i+1}}
    \frac{X_{t_i}}{\sigma_{t_i}}
    +
    e^{\lambda_{t_{i+1}}}
    \left(
        1-e^{-2\delta_{i+1}}
    \right)
    \mu_{t_i}(X_{t_i})
    +
    \sqrt{
        1-e^{-2\delta_{i+1}}
    }\,Z_{t_i}.
    \label{eq:sde-dpm-solver-1}
\end{align}

More generally, DDIM-type samplers \citep{song2021denoising} interpolate between deterministic and stochastic reverse updates. For any given parameter \(0\le\eta\le1\), the (general) DDIM-type update can be written as
\begin{align}
\text{(DDIM)}:\quad \frac{X_{t_{i+1}}}{\sigma_{t_{i+1}}}
&=
\frac{X_{t_i}}{\sigma_{t_i}}
\sqrt{
    1-\eta\big(1-e^{-2\delta_{i+1}}\big)
}
+
e^{\lambda_{t_{i+1}}}
\Big(
    1
    -
    e^{-\delta_{i+1}}
    \sqrt{
        1-\eta\big(1-e^{-2\delta_{i+1}}\big)
    }
\Big)
\mu_{t_i}(X_{t_i})
\notag\\
&\quad
+
\sqrt{
    \eta\big(1-e^{-2\delta_{i+1}}\big)
}
\,Z_{t_i}.
\label{eq:ddim}
\end{align}
Notably, \(\eta\) controls the stochasticity of the update: the choice \(\eta=0\) yields the deterministic DDIM-type update, while \(\eta=1\) recovers the stochastic SDE-DPMSolver-1 update in \eqref{eq:sde-dpm-solver-1}.



\paragraph{Diffusion priors for inverse problems.}
Next, we briefly describe how diffusion models are commonly used as priors for linear inverse problems. 
Given a pre-trained diffusion model representing the prior distribution of \(X_0\), the aim is to generate samples from the posterior distribution
$p_{X_0 \mid Y=y}$.  
In analogy with the unconditional case, consider the conditional data predictor
\begin{align}
    \mu_{t,y}(x)
    \coloneqq 
    \mathbb E[X_0\mid X_t=x,\, AX_0+\veps=y].
    \label{eq:conditional-data-predictor}
\end{align}
If this conditional predictor were available, then one could simply replace the unconditional predictor \(\mu_t\) by \(\mu_{t,y}\) in the DDIM-type sampler \eqref{eq:ddim}, thereby directly obtaining samples from the posterior distribution \(p_{X_0\mid y}\). Consequently, a central challenge in diffusion-based inverse problems is to approximate the conditional predictor \(\mu_{t,y}\) using only the pre-trained unconditional predictor \(\mu_t\) together with the measurement \(y\).

Many existing methods are built upon this idea. Starting from the unconditional prediction \(\mu_t(x)\), they seek to compute a measurement-compatible data predictor $\mu$ by balancing two competing objectives: remaining faithful to
to the diffusion prior that yields the unconditional data predictor \(\mu_t(\cdot)\), while simultaneously enforcing consistency with the observed measurements.
To approximate the conditional data predictor, existing approaches often rely on heuristic approximations. A common strategy is to combine the unconditional predictor  $\mu_t(\cdot)$ and the measurement $Y$ in a simple --- often linear --- manner \citep{wang2022diffusion,kawar2022denoising,song2021solving}, while another is to use $\mu_t(\cdot)$ directly as a proxy for the unknown signal $X_0$ \citep{chung2022diffusion}.

\section{Main results}

In this section, we develop a posterior-consistent DDIM-type sampler for noisy linear inverse problems, and establish its theoretical guarantees. 

Before proceeding, we find it convenient to introduce notation associated with the singular value decomposition (SVD) of the measurement matrix  \(A \in \mathbb{R}^{n\times d}\). Suppose that $A$ has rank $r$, and let 
\begin{equation}
A
=
U\Sigma V^\top
=
U_{\mathcal S}\Sigma_{\mathcal S}V_{\mathcal S}^\top
\label{eq:A-svd}
\end{equation}
represent the full and compact singular value decompositions (SVDs) of \(A\), where $U \in \mathbb{R}^{n\times n}$ and $V \in \mathbb{R}^{d\times d}$ are orthonormal matrices, 
$U_{\mathcal{S}} \in \mathbb{R}^{n\times r}$ and $V_{\mathcal{S}} \in \mathbb{R}^{d\times r}$ consist of the left and right singular vectors associated with the nonzero singular values, and $\Sigma \in \mathbb{R}^{n\times d}$ and $\Sigma_{\mathcal{S}} \in \mathbb{R}^{r\times r}$ are the corresponding rectangular diagonal and diagonal matrices of singular values,  respectively. We denote by 
\begin{align}
\mathcal{S} \coloneqq \{s \in [\min\{n,d\}] : \Sigma_{ss} > 0\}
\end{align}
the indices of  the nonzero singular values, and hence the notation $U_{\mathcal{S}}, V_{\mathcal{S}}$ and $\Sigma_{\mathcal{S}}$.We also write
\begin{align}
U_{\cS} \defn [u_s]_{s\in\cS} \in \bR^{n\times r}, \quad V_{\cS} \defn [v_s]_{s\in\cS} \in \bR^{d\times r}, \quad \text{and} \quad \Sigma_{\cS} \defn \mathsf{diag}([\Sigma_{ss}]_{s\in\cS}) \in \bR^{ r\times r}.
\end{align}

\subsection{Motivation: SNR-guided partition of singular directions}
\label{sec:partition-direction}

The SVD~\eqref{eq:A-svd} of the measurement matrix $A$ naturally decouples the inverse problem into a sequence of scalar inference problems along its singular directions. Indeed, expressing the measurement model $Y = AX_0 + \veps$ (i.e., \eqref{eq:inverse}) in the SVD basis yields
\begin{align}
\Sigma_{\mathcal{S}}^{-1}U_{\mathcal{S}}^{\top}Y = V_{\mathcal{S}}^{\top}X_0 +  \varepsilon'  \qquad \text{with } \varepsilon' \defn \Sigma_{\mathcal{S}}^{-1}U_{\mathcal{S}}^{\top} \varepsilon \sim \mathcal{N}\big(0, \sigma^2\Sigma_{\mathcal{S}}^{-2}\big); \label{eq:obs-model}
\end{align}
or equivalently, for every $s\in \mathcal{S}$,
\begin{align}
\Sigma_{ss}^{-1}u_{s}^{\top}Y = v_{s}^{\top}X_0 +  \varepsilon_s'  \qquad \text{with } \varepsilon_s' \sim \mathcal{N}\bigg(0, \frac{\sigma^2}{\Sigma_{ss}^2}\bigg). \label{eq:obs-model-ss}
\end{align}
Thus, after transforming into the singular-vector basis, the effective {\em observation SNR} along the \(s\)-th singular direction is naturally measured by \(\Sigma_{ss}/\sigma\). 
Meanwhile, projecting the forward diffusion process \eqref{eq:forward-process} along the same singular direction gives
\begin{align}
v_s^\top X_t = \alpha_t v_s^\top X_0 + \sigma_t Z_s\qquad \text{with } Z_s \sim \mathcal{N}(0,1),
\end{align} 
whose SNR --- which we often refer to as {\em diffusion SNR} --- is given by \(\alpha_t/\sigma_t\). 
Consequently, the relative magnitudes of these two SNRs dictate whether the posterior update should rely primarily on the measurements or on the diffusion prior.

This comparison naturally leads to a partition of the singular directions.  For each diffusion time \(t\), we partition the singular directions according to
\begin{align}
\St \defn \{s\in \mathcal{S}:\alpha_{t}\sigma<\Sigma_{ss}\sigma_{t}\} \qquad \text{and} \qquad \Stc \defn \mathcal{S}\setminus \St. 
\end{align}
For each $s \in \mathcal{S}$, we define the threshold time $\tau_s \in \bR_{\geq 0} \cup \{\infty\}$ such that
\begin{align*}
\frac{\alpha_{\tau_s}}{\sigma_{\tau_s}} = \frac{\Sigma_{ss}}{\sigma}, 
\end{align*}
namely, the diffusion time at which the diffusion SNR matches the observation SNR along the \(s\)-th singular direction. When $\alpha_t/\sigma_t$ is continuous and strictly decreasing in \(t\),  with  $\alpha_0/\sigma_0 = \infty$ and $\alpha_\infty/\sigma_\infty = 0$, the threshold \(\tau_s\) is well defined. It can be easily verified that 
\begin{equation}
s\in \St
\qquad \Longleftrightarrow \qquad 
t \geq \tau_s.
\label{eq:equiv-taus-St}
\end{equation}

Therefore, at each diffusion time \(t\), it is natural to partition the singular directions into three groups:
\begin{itemize}

\item {\em Measurement-dominated directions}: $s\in \St$,  or equivalently, directions obeying $t \geq \tau_s$;

\item {\em Prior-dominated observed directions}:  $s\in \Stc$, or equivalently, directions in $\mathcal{S}$ obeying  $t < \tau_s$;

\item {\em Unobserved directions}: $s\in \mathcal{S}^{\mathrm{c}}\coloneqq [d]\setminus \mathcal{S}$.
\end{itemize}
Our sampler treats these three groups differently. As we shall see, each case admits a simple DDIM-type update, resulting in an efficient posterior sampler that adapts to the SNR along each singular direction.



\subsection{Algorithm: \pddim}


We now describe the proposed sampler, based on the partition of singular directions developed in Section~\ref{sec:partition-direction}. The central idea is to exploit the measurements whenever the corresponding observation SNR exceeds the diffusion SNR, and otherwise rely on the learned diffusion prior. This idea gives rise to distinct update rules for the three groups of singular directions, which we describe below.  
The complete algorithm of the proposed DDIM-based posterior sampler, hereafter referred to as \pddim, is summarized in Algorithm~\ref{alg:main}.

\begin{algorithm}[t]
	\DontPrintSemicolon   \caption{\pddim sampler for the linear inverse problem \eqref{eq:inverse} \label{alg:main}}
\textbf{input:} time points $t_0 > t_1 > \cdots > t_M$, learned data predictor $\{\mu_{t
_i}(\cdot)\}_{i=1}^M$, schedule $\{(\alpha_{t_i},\sigma_{t_i})\}_{i=1}^M$, stochasticity parameters $\eta_0, \eta_1$. \\
\textbf{initialize}  $\widehat{X}_{t_0}\sim\mathcal{N}(0,I_d)$, $Z=[Z_s]_{1\leq s\leq d}\sim\mathcal{N}(0,I_d)$. \\
compute the SVD of $A$ (cf.~\eqref{eq:A-svd}); set $\mathcal{S} \gets \{i:\Sigma_{ii}>0\}$. \\
%
{\color{blue}\tcc{initialize DDPM updates in all directions.}}
\For{$s=1,\cdots,d$}{
\If{$s\le \min\{n,d\}$ and $\alpha_{t_0}\sigma<\Sigma_{ss}\sigma_{t_0}$}{
set $v_{s}^{\top}\widehat{X}_{t_0} \gets \alpha_{t_0} \Sigma_{ss}^{-1}u_s^{\top}Y + \sqrt{\sigma_{t_0}^2 - \alpha_{t_0}^2\sigma^2\Sigma_{ss}^{-2}} Z_s$.}
\Else{
set $v_{s}^{\top}\widehat{X}_{t_0} \leftarrow Z_s$.
}
}
%
\For{$i = 0,\dots,M-1$}
{ 
set $\Sti{t_{i+1}}\gets \{s:\alpha_{t_{i+1}}\sigma<\Sigma_{ss}\sigma_{t_{i+1}}\}$ and $\Stci{t_{i+1}}\gets \mathcal{S}\setminus \mathcal{S}_{t_{i+1}}^{\mathsf{meas}}$. \\
 draw $Z_{t_{i+1}}\sim\mathcal{N}(0,I)$. \\
    \For{$s\in\Sti{t_{i+1}}$}{
{\color{blue}\tcc{updates for measurement-dominated directions.}}
    compute 
    $v_{s}^{\top}\widehat{X}_{t_{i+1}} \leftarrow  \alpha_{t_{i+1}} \Sigma_{ss}^{-1}u_s^{\top}Y + \sqrt{\sigma_{t_{i+1}}^2 - \alpha_{t_{i+1}}^2\sigma^2\Sigma_{ss}^{-2}} Z_s$. 
    }
     \For{$s\in\Stci{t_{i+1}}$}{
     {\color{blue}\tcc{DDIM for prior-dominated directions.}}
     \If{all nonzero singular values of $A$ are equal}{compute $v_{s}^{\top}\widehat{X}_{t_{i+1}}$ via \eqref{eq:ddim-prior-dominated} with $\eta=\eta_0$. 
}
\Else{compute $v_{s}^{\top}\widehat{X}_{t_{i+1}}$ via \eqref{eq:ddim-prior-dominated-extension} with $\eta=\eta_0$. {\color{blue}\tcp{add correction for general case.}}
    }
    }

\For{ $s\in\mathcal{S}^{\mathsf{c}}$}{compute $v_{s}^{\top}\widehat{X}_{t_{i+1}}$ via \eqref{eq:ddim-prior-dominated} with $\eta=\eta_1$. {\color{blue}\tcp{DDIM for unobserved directions.}}
}
} 
\textbf{output:} the generated sample $\widehat{X}_{t_M}=\sum_{s}v_s v_s^{\top}\widehat{X}_{t_M}$.
\end{algorithm}

\paragraph{Measurement-dominated directions.}
For directions \(s\in \St \), the measurements are sufficiently informative to be used directly. To this end, consider the auxiliary forward process
\begin{align}
\xi_{t,s} \defn \alpha_{t} \xi_{0,s} + \sqrt{\sigma_{t}^2 - \alpha_{t}^2\sigma^2\Sigma_{ss}^{-2}} Z_s, \quad Z_s \sim \mathcal{N}(0,1),\label{eq:def-xhat-t}
\end{align}
where $\xi_{0,s}$ is taken to be
\begin{align}\label{eq:def-sigmahat-alphahat}
\xi_{0,s} \defn \Sigma_{ss}^{-1}u_s^{\top}Y.
\end{align}
By construction, this auxiliary forward process $\xi_{t,s}$ has the same marginal distribution as the forward diffusion process $X_t$ (cf.~\eqref{eq:forward-process}) projected onto the direction the $s$-th singular direction $v_s$, namely, 
\begin{align*}
\xi_{t,s} \disteq v_{s}^\top X_t = v_{s}^\top (\alpha_{t} X_0 + \sigma_{t} Z).
\end{align*}
Consequently, sampling 
$v_{s}^{\top}\widehat{X}_{t_i}$ at the time point $t_i$ can be accomplished by generating \begin{align}
v_{s}^{\top}\widehat{X}_{t_i} & = \alpha_{t_i} \Sigma_{ss}^{-1}u_s^{\top}Y + \sqrt{\sigma_{t_i}^2 - \alpha_{t_i}^2\sigma^2\Sigma_{ss}^{-2}} Z_{t_i} \quad\text{with } Z_{t_i} \sim \mathcal{N}(0,1),
\label{eq:DDIM-St}
\end{align}
provided that \(s\in \mathcal{S}_{t_i}^{\mathsf{meas}} \). 
A formal proof is given in Appendix \ref{app:proof-eq-ddim-st}.


\paragraph{Prior-dominated observed directions.}
For directions \(s\in \Stc \), the measurements are less informative than the diffusion prior. Consequently, rather than directly incorporating the observations, we update these directions using the standard DDIM sampler driven primarily by the learned data predictor.

We begin with the special case in which all nonzero singular values of \(A\) are identical. In this setting, for \(s\in\mathcal S_{t_{i+1}}^{\mathsf{prior}}\), we update $v_s^\top \wh X_{t}$ according to the standard DDIM-type update rule (cf.~\eqref{eq:ddim-prior-dominated}) as follows
\begin{align}
\frac{v_s^\top \wh X_{t_{i+1}}}{\sigma_{t_{i+1}}}
&=\frac{v_s^\top \wh X_{t_i}}{\sigma_{t_i}}\sqrt{1-\eta(1-e^{-2\delta_{i+1}})}
+e^{\lambda_{t_{i+1}}}
\left(1-e^{-\delta_{i+1}}\sqrt{1-\eta(1-e^{-2\delta_{i+1}})}\right)v_s^\top \mu_{t_i}(\wh X_{t_i})\notag\\
&\quad+
\sqrt{\eta(1-e^{-2\delta_{i+1}})}
\,v_s^\top Z_{t_{i+1}}, \qquad Z_{t_{i+1}}\sim\mathcal N(0,I_d),
\label{eq:ddim-prior-dominated}
\end{align}
where the stochasticity parameter is set to \(\eta=\eta_0\).

The above update extends naturally to the general setting where \(A\) has arbitrary nonzero singular values. The only modification is the addition of a correction term. Specifically, for \(s\in\Stci{t_i}\), we replace \eqref{eq:ddim-prior-dominated} with
\begin{align}
\frac{v_s^\top \widehat{X}_{t_{i+1}}}{\sigma_{t_{i+1}}}
&=\frac{v_s^\top \widehat{X}_{t_i}}{\sigma_{t_i}}\sqrt{1-\eta(1-e^{-2\delta_{i+1}})}
+\sqrt{\eta(1-e^{-2\delta_{i+1}})}
\,v_s^\top Z_{t_{i+1}}\notag\\
&\quad+e^{\lambda_{t_{i+1}}}
\left(1-e^{-\delta_{i+1}}\sqrt{1-\eta(1-e^{-2\delta_{i+1}})}\right)v_s^\top \Bigg(\mu_{t_i}(\widehat{X}_{t_i}) + \underbrace{\frac{\alpha_{t_i}\widehat{X}_{\tau_s}-\alpha_{\tau_s}\widehat{X}_{t_i}}{\alpha_{\tau_s}\alpha_{t_i}({\rm e}^{2(\lambda_{t_i} - \lambda_{\tau_s})}-1)}}_{\text{correction term}}\Bigg),
\label{eq:ddim-prior-dominated-extension}
\end{align}
where  $Z_{t_{i+1}}\sim\mathcal N(0,I_d)$. 
The derivation of this modified update rule is deferred to Appendix~\ref{subsec:extension}.

\paragraph{Unobserved directions.}
For directions \(s\in\mathcal S^{\mathsf c}\), the measurements contain no information about \(v_s^\top X_0\). Accordingly, these coordinates should be sampled entirely from the learned diffusion prior, given the coordinates that are observed. As in the prior-dominated observed case, we employ the DDIM update \eqref{eq:ddim-prior-dominated}. However, we choose a larger stochasticity parameter \(\eta=\eta_1\) to encourage sufficient exploration of the unobserved subspace. This is important because the posterior uncertainty along \(\mathcal S^{\mathsf c}\) is determined entirely by the learned diffusion prior, rather than by the measurement model.

\paragraph{Choices of the stochasticity parameters $\eta_0$ and $\eta_1$.} 
We take a moment to discuss the roles of the two DDIM stochasticity parameters \(\eta_0\) and \(\eta_1\). Although both parameters control the amount of randomness injected into the DDIM update, they serve fundamentally different purposes because the target conditional distributions differ across the two subspaces. For directions \(s\in \Stc\), the corresponding coordinates are observed, although the measurements are not yet sufficiently informative to be imposed directly. The update should therefore approximate the conditional reverse transition given the current diffusion state and the measurements. This motivates using a relatively small value of \(\eta_0\), so as to keep the update closer to the deterministic DDIM sampler. In contrast, for directions \(s\in\mathcal S^{\mathsf c}\), no measurements are available. These coordinates must be sampled entirely from the learned conditional prior given the observed coordinates. A larger value of \(\eta_1\) makes the DDIM update more Langevin-like, thereby promoting exploration of the unobserved subspace

\subsection{Theoretical guarantees: asymptotic posterior correctness}
\label{subsec:thm}

In this subsection, we establish the posterior correctness of the proposed \pddim sampler. We begin with the special case in which all nonzero singular values of the measurement matrix $A$ are identical, as this setting captures the main ideas while allowing for cleaner algorithm design analysis.  Here and throughout, $\mathsf{TV}(p,q)$ denotes the total-variation (TV) distance between two probability measures $p$ and $q$. 

\begin{theorem}\label{thm:main}
Consider a sequence of time points $t_0>\cdots>t_M$. Recall the definition of $\delta_i$ in \eqref{eq:defn-delta-i}, and set 
\begin{equation}
\delta\coloneqq  \max\Big\{\max_{1\leq i\leq M}\delta_{i},\alpha_{t_0},|1-\sigma_{t_0}|\Big\}.
\end{equation}
Assume that the support of $X_0$ is bounded, and that all nonzero singular values of $A$ are equal. Suppose further that $\eta_0=1$. When $\eta_1=\delta^{-p}$ with $1/2<p<2/3$, 
the \pddim sampler (see Algorithm~\ref{alg:main}) achieves 
\begin{align}\label{eq:thm-main}
\lim_{t_M\to 0} \lim_{\delta\to 0} \mathsf{TV}\big(p_{\widehat{X}_{t_M}\mid  Y},p_{X_0\mid Y}\big) = 0.
\end{align}
\end{theorem}
%

%
%


In words, Theorem~\ref{thm:main} establishes that, under the assumption that all nonzero singular values of \(A\) are identical, the distribution of the output generated by Algorithm~\ref{alg:main} converges to the posterior distribution of \(X_0\) conditioned on the measurements. This guarantee is asymptotic and holds in the regime where (i) the terminal time \(t_M\) tends to \(0\), approaching the initial time of the forward diffusion process; (ii) the initial sampling time \(t_0\) is chosen such that $\alpha_{t_0}\to0$ and $\sigma_{t_0}\to 1$, ensuring that the SNR at the initialization stage vanishes; and (iii) the maximum gap between the log-SNRs of consecutive time points vanishes, which in particular requires the number of sampling steps \(M\) to tend to infinity. 
 The proof is deferred to Appendix~\ref{sec:analysis}.

The restriction that all nonzero singular values of $A$ are identical is not essential. The proposed approach extends naturally to the general setting with arbitrary nonzero singular values, requiring only a slight modification of the sampling procedure (as already described in \eqref{eq:ddim-prior-dominated-extension}).  The corresponding theoretical guarantee is stated below. Compared to Theorem~\ref{thm:main}, the only additional change is a different choice of \(\eta_0\). The proof of this theorem is postponed to Appendix~\ref{subsec:extension}. 
\begin{theorem}\label{cor:main}
Under the assumptions of Theorem~\ref{thm:main}, except that the nonzero singular values of \(A\) are allowed to be arbitrary and \(\eta_0=\eta_1=\delta^{-p}\) for some \(1/2<p<2/3\), the conclusion of Theorem~\ref{thm:main}, namely \eqref{eq:thm-main}, continues to hold.
\end{theorem}

    

\section{Experiments}
\label{sec:exp}

In this section, we conduct extensive empirical evaluations of the proposed \pddim algorithm on a diverse range of image restoration tasks, including inpainting, super-resolution, Gaussian deblurring, and compressed sensing.

We begin by describing the common experimental setup. Unless otherwise stated, the experiments are run with the NFE (number of function evaluations) set to 100, except for \texttt{DPS}, for which we follow \citet{zhang2024unleashing} and set the NFE to be $1000$. To justify our default choice,  we also investigate the effect of NFE and find that the reconstruction quality largely saturates once the NFE exceeds 80 (see Figure~\ref{fig:imgnet_sr}). 
All experiments are conducted on a server equipped with NVIDIA A100-PCIE-40GB GPUs and Intel Xeon (Cascade Lake) CPUs.


\subsection{Choices of  $\eta_0$ and $\eta_1$}
\label{subsec:exp-eta}

We evaluate the influence of the hyperparameters  $\eta_0$ and $\eta_1$ in \pddim, which govern the DDIM updates on $\mathcal{S}^{\mathrm{c}}$ and $\Stc$, respectively. Experiments are conducted on two image restoration tasks: 50\% random inpainting and $4\times$ super-resolution with average pooling, using the CelebA \citep{liu2015deep} and ImageNet \citep{deng2009imagenet} datasets. 
We set the noise standard deviation to $\sigma=0.05$ (doubled when pixels are rescaled to $[-1,1]$).
For ImageNet, we use the pretrained $256 \times 256$ model without classifier guidance from \citet{dhariwal2021diffusion}.
For CelebA, we use the pretrained model released with SDEdit \citep{meng2021sdedit}.
Following \citet{wang2022zero,zhang2024unleashing}, we evaluate on the 1K ImageNet sub-test set and the CelebA test set provided by \citet{wang2022zero}. Performance is measured using PSNR (peak SNR), SSIM \citep{wang2004image}, LPIPS \citep{zhang2018unreasonable}, and FID \citep{heusel2017gans}. For readability, all reported LPIPS values are multiplied by 100.

We vary $\eta_1$ from $0$ to $32$ while fixing $\eta_0=0$, and vary $\eta_0$ from $0$ to $2$ while fixing $\eta_1=16$.
The results for inpainting and super-resolution are reported in Tables \ref{tab:result_eta_inp} and \ref{tab:result_eta_sr}, respectively.
For fixed $\eta_0$, increasing $\eta_1$ generally improves performance across most evaluation metrics, although the improvement is not strictly monotone and small fluctuations are observed. This overall trend is consistent with our theoretical analysis that requires $\eta_1$ to be sufficiently large.
Turning to the second set of experiments, where $\eta_1$ is fixed and $\eta_0$ 
 is varied, we  find that fine-tuning $\eta_0$ can yield additional, albeit modest, performance gains. Empirically, the best empirical performance is attained at $\eta_0=0$, which corresponds to the deterministic DDIM update.

%


\begin{table}[t]
\centering
\caption{Empirical results for inpainting on CelebA and ImageNet under different $\eta$. 
}
\label{tab:result_eta_inp}
\begin{tabular}{c|c|cccc|cccc}
\hline
\multirow{2}{*}{$\eta_0$} & \multirow{2}{*}{$\eta_1$}
& \multicolumn{4}{c|}{CelebA}
& \multicolumn{4}{c}{ImageNet} \\
\cline{3-10}

&
& PSNR$\uparrow$ & SSIM$\uparrow$ & LPIPS$\downarrow$ & FID$\downarrow$
& PSNR$\uparrow$ & SSIM$\uparrow$ & LPIPS$\downarrow$ & FID$\downarrow$ \\\hline
                           & 0                          & 27.22                     & 0.79                  &  27.11                &  61.46       &  {21.24}                        &  {0.54}                        &  {46.72}                        & 89.56                        \\  
                           & 1                                 &  31.32                &  0.88                 &  16.97                 &     30.24       &  {26.13}                        &  {0.78}                        &  {26.17}                        & 37.76                      \\  
                           & 2                                 &   32.55                 &   0.90                  &   15.44                 &    26.62         &  {28.76}                        &  {0.86}                        &  {16.74}                        & 18.60              \\  
                           & 4                                 & 33.16                &  \textbf{0.91}                   &  15.07                   &  25.42           &  {30.72}                        &  \textbf{0.89}                        &  14.45                        & \textbf{15.74}              \\  
                           & 8                                 &  \textbf{33.23}                 &  \textbf{0.91}                   &  \textbf{15.06}                   &   \textbf{25.33}        &  \textbf{31.06}                        &  \textbf{0.89}                        &  \textbf{14.31}                        & 15.86           \\  
                           & 16                                &   \textbf{33.23}                 &   \textbf{0.91}                  &   \textbf{15.06}                 &  \textbf{25.33}            &  {\textbf{31.06}} &  {\textbf{0.89}} &  {\textbf{14.31}} & {15.85} \\  
\multirow{-7}{*}{0}        & 32           &  \textbf{33.23}                    &  \textbf{0.91}                        &  \textbf{15.06}      &                           \textbf{25.33}                          &  {\textbf{31.06}} &  {\textbf{0.89}} &  {\textbf{14.31}} & {15.85} \\ \hline
0  & &  \textbf{33.23}                 &   \textbf{0.91}                  &   \textbf{15.06}                 &  \textbf{25.33}            &  {\textbf{31.06}} &  {\textbf{0.89}} &  {\textbf{14.31}} & {15.85}\\
0.5        &                                &    33.03                &  0.90                    &  15.66                                  &   26.60          &  {30.86}                        &  {0.88}                        &  {15.33}                        & 17.89           \\   
1                          &                                 &   32.84                   &   0.90                  &  16.36                   &  28.42          &  {30.67}                        &  {0.88}                        &  {16.61}                        & 20.85            \\   
2                                & \multirow{-4}{*}{16}                &   32.51                   &    0.89                 &   17.89 & 33.21             &  {30.26}                        &  {0.86}                        &  {19.52}                        & 28.71          \\ \hline
\end{tabular}
\end{table}

\begin{table}[t]
\centering
\caption{Empirical results for super-resolution on CelebA and ImageNet under different $\eta$.}
\label{tab:result_eta_sr}
\begin{tabular}{c|c|cccc|cccc}
\hline
\multirow{2}{*}{$\eta_0$} & \multirow{2}{*}{$\eta_1$}
& \multicolumn{4}{c|}{CelebA}
& \multicolumn{4}{c}{ImageNet} \\
\cline{3-10}

&
& PSNR$\uparrow$ & SSIM$\uparrow$ & LPIPS$\downarrow$ & FID$\downarrow$
& PSNR$\uparrow$ & SSIM$\uparrow$ & LPIPS$\downarrow$ & FID$\downarrow$ \\\hline
                           & 0                          &  {27.77}          &  {0.78}           &  {22.49}             & 32.97                  & 22.48                       & 0.56                    &  38.50                    &  53.08           \\  
                           & 1                          &   {28.39}          &  {0.80}           &  \textbf{21.09}             & \textbf{31.24}                  &    24.81                 &  0.70                   &  \textbf{31.22}                 &   \textbf{39.94}           \\  
                           & 2                          &   {28.92}          &  {0.82}           &  {21.31}             & 32.56                  &   25.67                   &  0.72                   &  32.03                  &  45.74              \\  
                           & 4                          &   {29.59}          &  {0.83}           &  {21.85}             & 33.49                  &  \textbf{25.93}                    & \textbf{0.73}                 &  32.73                   &  49.06            \\  
                           & 8                          &  {29.83}          &  \textbf{0.84}           &  {22.02}             & 34.13                  &  25.90                   &  \textbf{0.73}                    &  32.84                    &  49.99           \\  
                           & 16                         &  \textbf{29.84}          &  \textbf{0.84}           &  {22.02}             & 34.16                  &  25.90                    &  \textbf{0.73}                    & 32.84                    & 49.99              \\  
\multirow{-7}{*}{0}        & 32                         &  \textbf{29.84}          &  \textbf{0.84}           &  {22.02}             & 34.16                  &  25.90                    &  \textbf{0.73}                   &  32.84                    &  49.99             \\ \hline
0 & &  \textbf{29.84}          &  \textbf{0.84}           &  {22.02}             & 34.16                  &  25.90                    &  \textbf{0.73}                    & 32.84                    & 49.99 \\
0.5                        &                            &  {29.48}          &  {0.83}           &  {22.86}             & 35.35                  &  25.64                     &  0.72                   &  34.11                   &  52.27            \\   
1                          &                            &  {29.21}          &  {0.82}           &  {23.57}             & 37.15                  &    25.44                  & 0.70                     &   35.23                   &  55.03             \\   
2                          & \multirow{-4}{*}{16}       &  {28.82}          &  {0.81}           &  {24.89}             & 41.79                  &  25.13                    &   0.69                   &  37.24                   &   61.86          \\ \hline
\end{tabular}
\end{table}


\subsection{Empirical comparisons with reference algorithms}
\label{subsec:exp-comparison}

Next, we compare \pddim against several representative diffusion-based image restoration algorithms, including \texttt{\texttt{DDRM}} \citep{kawar2022denoising}, \texttt{\texttt{DDNM+}} \citep{wang2022zero}, \texttt{\texttt{DPS}} \citep{chung2022diffusion}, \texttt{\texttt{RED-diff}} \citep{mardani2023variational}, and \texttt{ProjDiff} \citep{zhang2024unleashing}.\footnote{In addition to our algorithm, we also implement \texttt{ProjDiff} and \texttt{DDNM+}.
The remaining results reported in this section are taken from \citet{zhang2024unleashing}.} As a classical baseline, we also include the least-squares estimator, $\widehat{X}^{\mathsf{ls}} = A^{\dagger}Y$. In addition to the inpainting and super-resolution tasks considered in Section~\ref{subsec:exp-eta}, we further evaluate all methods on Gaussian deblurring and compressed sensing.  For Gaussian deblurring, we use a one-dimensional Gaussian blur kernel of size \(5\) with standard deviation \(10\). For compressed sensing, we employ a Walsh--Hadamard sampling matrix with a sampling ratio of \(0.5\). Throughout all experiments, the observation noise standard deviation is fixed at \(\sigma=0.05\).

In addition, note that the posterior mean can be approximately computed via Monte Carlo sampling, i.e,
$$
\mathbb{E}[X_0\mymid Y] \approx \frac{1}{n}\sum_{i=1}^N \widehat{X}^{(i)},
$$
where $\widehat{X}^{(i)}\overset{\mathsf{i.i.d.}}{\sim} p_{X_0\mymid Y}$.
As is well-known,  the posterior mean achieves the optimal reconstruction accuracy under the mean squared estimation error metric. 
%
%
Consequently, reconstruction fidelity can be improved by averaging samples generated from independent noise realizations.
We observe, however, that this approach may sometimes degrade quality as measured by FID. 
To illustrate this effect, we report both single-sample results and posterior-mean results, where the latter are obtained by averaging 4 independent trials. 

The results are reported in Tables \ref{tab:sr}--\ref{tab:cs}. Boldface indicates the best overall result, while underlining indicates the best sampling-based result. 
Across all tasks and datasets, the proposed \pddim method consistently delivers strong performance. 
For super-resolution, it achieves the best performance on three and two of the four evaluation metrics on CelebA and ImageNet datasets respectively. For inpainting, it outperforms the reference algorithms on all four metrics for CelebA and on two metrics for ImageNet.
For Gaussian deblurring, \pddim achieves the best performance on three metrics for CelebA and two metrics for ImageNet.
For compressed sensing, \pddim outperforms \texttt{DDNM+} across all of the four evaluation metrics.
Overall, whereas existing methods typically achieve the best result on only one or two metrics in most settings, our method attains the best performance on \textbf{at least two} evaluation metrics across all tasks and datasets, demonstrating its efficacy and robustness across a diverse range of inverse problems.

\begin{table*}[t]
\centering
\caption{Empirical results for super-resolution on CelebA and ImageNet.}
\label{tab:sr}
\resizebox{\textwidth}{!}{
\begin{tabular}{c|c|cccc|cccc}
\hline
\multirow{2}{*}{} & \multirow{2}{*}{Method}
& \multicolumn{4}{c|}{CelebA}
& \multicolumn{4}{c}{ImageNet} \\
\cline{3-10}

&
& PSNR$\uparrow$ & SSIM$\uparrow$ & LPIPS$\downarrow$ & FID$\downarrow$
& PSNR$\uparrow$ & SSIM$\uparrow$ & LPIPS$\downarrow$ & FID$\downarrow$ \\
\hline

\multirow{8}{*}{\makecell{Single-\\sample}}
& $A^\dagger y$
& 23.64 & 0.49 & 68.72 & 147.89
& 21.85 & 0.45 & 65.34 & 183.32 \\

& \texttt{DPS}
& 27.98 & 0.78 & 23.10 & 39.91
& 24.44 & 0.67 & \underline{\textbf{31.81}} & \underline{\textbf{36.17}} \\

& \texttt{DDRM}
& 29.20 & 0.82 & 21.92 & 40.14
& 25.66 & 0.72 & 34.88 & 55.71 \\


& \texttt{RED-diff}
& 24.98 & 0.55 & 50.59 & 73.89
& 22.74 & 0.49 & 53.24 & 96.26 \\

& \texttt{DDNM+}
& 29.20 & 0.82 & 21.91 & 39.96
& 25.62 & 0.72 & 34.39 & 53.78 \\

& \texttt{ProjDiff}
& 29.49 & 0.83 & \textbf{\underline{20.89}} & 36.61
& 25.73 & 0.72 & 33.03 & 49.70 \\

& \textbf{\pddim (Ours)}
& \underline{29.84} & \underline{0.84} & 22.02 & \underline{\textbf{34.16}}
& \underline{25.90} & \underline{{0.73}} & 32.84 & 49.99 \\

\hline

\multirow{3}{*}{\makecell{Posterior\\ mean}}
& \texttt{DDNM+}
& 30.22 & \textbf{0.85} & 21.97 & 43.54
& 26.11 & 0.73 & 34.38 & 54.27 \\

& \texttt{ProjDiff}
& 30.31 & \textbf{0.85} & 22.55 & 38.34
& 26.10 & \textbf{0.74} & 33.13 & 50.27 \\

& \textbf{\pddim (Ours)}
& \textbf{30.35} & \textbf{0.85} & 22.55 & 40.62
& \textbf{26.13} & \textbf{0.74} & 33.35 & 52.91 \\

\hline
\end{tabular}
}
\end{table*}

\begin{table*}[t]
\centering
\caption{Empirical results for inpainting on CelebA and ImageNet. }
\label{tab:inpainting}
\resizebox{\textwidth}{!}{
\begin{tabular}{c|c|cccc|cccc}
\hline
\multirow{2}{*}{} & \multirow{2}{*}{Method}
& \multicolumn{4}{c|}{CelebA}
& \multicolumn{4}{c}{ImageNet} \\
\cline{3-10}

& 
& PSNR$\uparrow$ & SSIM$\uparrow$ & LPIPS$\downarrow$ & FID$\downarrow$
& PSNR$\uparrow$ & SSIM$\uparrow$ & LPIPS$\downarrow$ & FID$\downarrow$ \\
\hline

\multirow{8}{*}{\makecell{Single-\\sample}}
& $A^\dagger y$
& 13.70 & 0.19 & 76.07 & 226.28
& 14.21 & 0.24 & 67.42 & 176.52 \\

& \texttt{DPS}
& 32.80 & 0.90 & 16.32 & 32.80
& 30.15 & 0.86 & 17.76 & 22.03 \\

& \texttt{DDRM}
& 32.81 & 0.90 & 16.78 & 35.28
& 29.99 & 0.87 & 17.11 & 19.88 \\


& \texttt{RED-diff}
& 7.96 & 0.18 & 78.27 & 192.32
& 9.85 & 0.17 & 83.92 & 281.65 \\

& \texttt{DDNM+}
& 32.81 & 0.90 & 16.78 & 35.33
& 30.00 & 0.87 & 17.07 & 19.92 \\

& \texttt{ProjDiff}
& \underline{33.44} & \underline{0.91} & 15.32 & 31.13
& \underline{31.32} & \underline{0.89} & \underline{13.87} & \textbf{\underline{14.97}} \\

& \textbf{\pddim (Ours)}
& 33.23 & \underline{0.91} & \underline{15.06} & \underline{\textbf{25.33}}
& 31.06 & \underline{0.89} & 14.31 & 15.85 \\

\hline

\multirow{3}{*}{\makecell{Posterior\\ mean}}
& \texttt{DDNM+}
& 33.98 & \textbf{0.92} & 16.06 & 37.64
& 31.22 & 0.89 & 16.44 & 22.07 \\

& \texttt{ProjDiff}
& 34.22 & \textbf{0.92} & 14.41 & 32.40 
& \textbf{32.07} & \textbf{0.90} & 13.78 & 16.57 \\

& \textbf{\pddim (Ours)}
& \textbf{34.24} & \textbf{0.92} & \textbf{13.81} & 27.15
& 32.04 & \textbf{0.90} & \textbf{13.67} & 16.75 \\

\hline
\end{tabular}
}
\end{table*}

Figure~\ref{fig:celeba_inp} shows the performance of \pddim, \texttt{DDNM+}, and \texttt{ProjDiff} under different numbers of Monte Carlo samples used to estimate the posterior mean. As discussed previously, increasing the number of samples improves reconstruction fidelity, as reflected by PSNR, SSIM, and LPIPS, but comes at the expense of reduced sample diversity, leading to worse FID scores. 
Furthermore, among all reference algorithms considered, our method achieves the best trade-off between fidelity and diversity on this task.

Figure~\ref{fig:imgnet_sr} illustrates the effect of the NFE on the super-resolution task using the ImageNet dataset. As the NFE increases, the performance improves, and becomes stable once NFE exceeds $80$, suggesting that our default choice of NFE=100 provides a favorable balance between computational cost and reconstruction quality.

\begin{figure}[t]
    \centering
    \includegraphics[width=\textwidth]{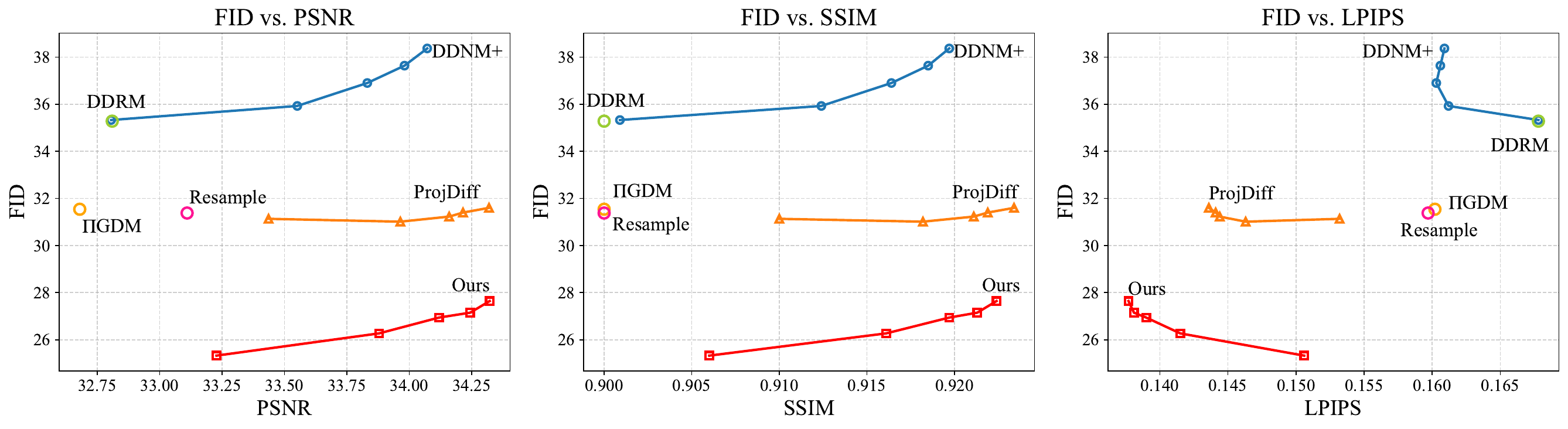}
    \caption{FID versus PSNR, SSIM, and LPIPS for different inpainting algorithms on the CelebA dataset.}
    \label{fig:celeba_inp}
\end{figure}


\begin{figure}[t]
    \centering
    \includegraphics[width=\textwidth]{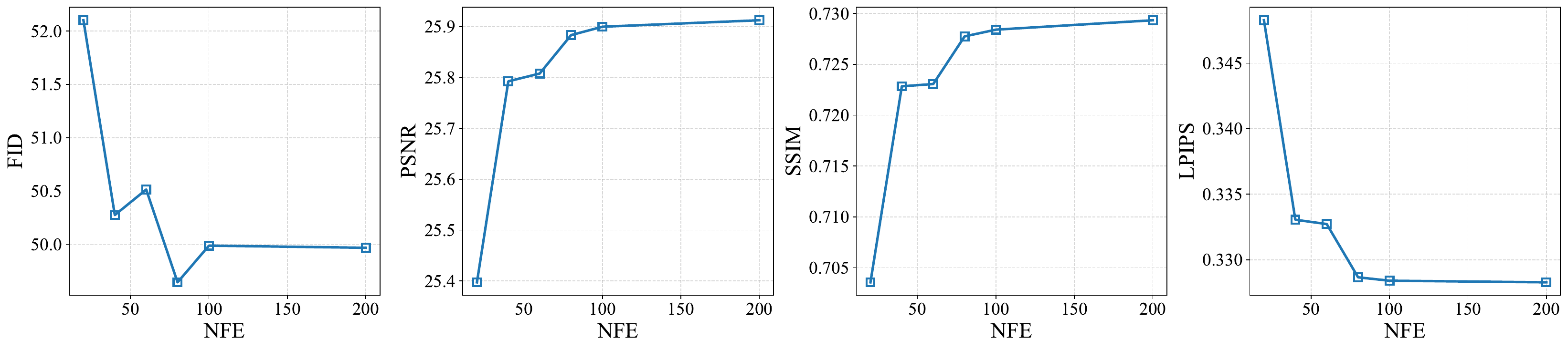}
    \caption{Effect of the NFE on the ImageNet super-resolution task.}
    \label{fig:imgnet_sr}
\end{figure}

\begin{table*}[t]
\centering
\caption{Empirical results for Gaussian deblurring on CelebA and ImageNet.\protect\footnotemark[2]}
\label{tab:deblurring}
\resizebox{\textwidth}{!}{
\begin{tabular}{c|c|cccc|cccc}
\hline
\multirow{2}{*}{} & \multirow{2}{*}{Method}
& \multicolumn{4}{c|}{CelebA}
& \multicolumn{4}{c}{ImageNet} \\
\cline{3-10}

&
& PSNR$\uparrow$ & SSIM$\uparrow$ & LPIPS$\downarrow$ & FID$\downarrow$
& PSNR$\uparrow$ & SSIM$\uparrow$ & LPIPS$\downarrow$ & FID$\downarrow$ \\
\hline

\multirow{8}{*}{\makecell{Single-\\sample}}
& $A^\dagger y$
& 23.64 & 0.49 & 68.72 & 147.89
& 17.78 & 0.29 & 60.25 & 100.05 \\

& \texttt{DPS}
& 27.98 & 0.78 & 23.10 & 39.91
& 24.26 & 0.64 & 37.03 & 50.17 \\

& \texttt{DDRM}
& 29.20 & 0.82 & 21.92 & 40.14
& 27.82 & 0.79 & 28.62 & 45.96 \\

& \texttt{DDNM+}
& 29.19 & 0.82 & 21.89 & 40.20
& 27.90 & 0.79 & 28.63 & 46.54 \\

& \texttt{RED-diff}
& 24.98 & 0.55 & 50.59 & 73.89
& 23.74 & 0.50 & 48.12 & 68.23 \\

& \texttt{ProjDiff}
& \underline{31.41} & \underline{\textbf{0.87}} & \textbf{\underline{18.12}} & 34.89
& \underline{27.91} & \underline{\textbf{0.79}} & 23.85 & 29.12 \\

& \textbf{\pddim (Ours)}
& 30.18 & 0.84 & 18.90 & \underline{27.70}
& 26.99 & 0.76 & \underline{23.71} & \underline{24.54} \\

\hline

\multirow{2}{*}{\makecell{Posterior\\ mean}}
& \texttt{ProjDiff}
& \textbf{31.44} & \textbf{0.87} & 18.22 & 35.96
& \textbf{27.96} & \textbf{0.79} & 23.93 & 29.41 \\

& \textbf{\pddim (Ours)}
& \textbf{31.44} & \textbf{0.87} & 18.18 & \textbf{27.15}
& 27.28 & 0.77 & \textbf{23.05} & \textbf{23.69} \\

\hline
\end{tabular}
}
\end{table*}

\footnotetext[2]{When reproducing \texttt{DDNM+} on the Gaussian deblurring task, we observe that the obtained results are substantially inferior to those reported in \citet{zhang2024unleashing}. Consequently, we adopt the results from \citet{zhang2024unleashing}; since only the published single-sample results are available, posterior-mean results for \texttt{DDNM+}.  
}

\begin{table*}[t]
\centering
\caption{Empirical results for compressed sensing on CelebA and ImageNet.\protect\footnotemark[3]}
\label{tab:cs}
\resizebox{\textwidth}{!}{
\begin{tabular}{c|c|cccc|cccc}
\hline
\multirow{2}{*}{} & \multirow{2}{*}{Method}
& \multicolumn{4}{c|}{CelebA}
& \multicolumn{4}{c}{ImageNet} \\
\cline{3-10}

&
& PSNR$\uparrow$ & SSIM$\uparrow$ & LPIPS$\downarrow$ & FID$\downarrow$
& PSNR$\uparrow$ & SSIM$\uparrow$ & LPIPS$\downarrow$ & FID$\downarrow$ \\
\hline

\multirow{3}{*}{\makecell{Single-\\sample}}

& $A^\dagger y$
& 14.96 & 0.34 & 66.48 & 397.50
& 15.08 & 0.37 & 60.45 & 257.70 \\

& \texttt{DDNM+}
& 29.05 & 0.88 & 18.46 & 36.43
& 23.64 & 0.81 & 21.10 & 25.38 \\


& \textbf{\makecell{\pddim (Ours)}}
& 28.94 & 0.88 & 16.91 & \textbf{26.69}
& 23.39 & 0.82 & 18.67 & \textbf{19.85} \\

\hline

\multirow{2}{*}{\makecell{Posterior\\ mean}}
& \texttt{DDNM+}
& 30.10 & \textbf{0.90} & 17.72 & 39.07
& 24.40 & 0.83 & 20.02 & 25.54 \\


& \textbf{\makecell{\pddim  (Ours)}}
& \textbf{30.35} & \textbf{0.90} & \textbf{15.20} & 28.57
& \textbf{24.70} & \textbf{0.85} & \textbf{16.68} & 18.73 \\

\hline
\end{tabular}
}
\end{table*}

\footnotetext[3]{   Results for the compressed sensing task are not implemented in \citet{zhang2024unleashing}; therefore, we compare only with \texttt{DDNM+}. We note that \citet{wang2022zero} reported that \texttt{DDNM+} outperforms \texttt{DDRM} with PSNR, SSIM, and FID metrics, and our algorithm outperforms \texttt{DDNM+}.}

\section{Discussion}

In this work, we have proposed a novel DDIM-type posterior sampler for solving linear inverse problems with diffusion priors.  In contrast to many existing methods that rely on heuristic projection steps or approximate posterior scores, our approach follows a different route by treating different singular directions of the measurement operator differently according to the relative strengths of the observation noise and the diffusion noise. This leads to a simple, easy-to-implement, DDIM-type sampling procedure that is both practically effective and theoretically principled. In particular, we have established that, under mild conditions,  the proposed \pddim sampler converges to the true posterior distribution of the signal given the measurement. 
Empirically, our experiments on a range of real-world image restoration tasks have shown that  the proposed method consistently performs favorably against existing diffusion-based posterior samplers, achieving the best performance on the majority of evaluation metrics across all tasks.

Several directions merit further investigation. For instance, the present work has focused on  linear inverse problems, and a fundamentally important next step is to extend the proposed framework to broader nonlinear inverse problems, which naturally arise in wide-ranging real-world applications. In addition, while our theoretical analysis has established asymptotic convergence to the true posterior distribution, it falls short of characterizing the rate of convergence, which would be of  interest to quantify. It would also be valuable to understand how the convergence behavior and sampling quality depend on additional factors such as the noise level, the properties of the measurement operator, and the choice of time discretization of the diffusion process.

\section*{Acknowledgments}
Y.~Chen is supported in part by the Alfred P.~Sloan Research Fellowship,  the ONR grant N00014-25-1-2344, 
the NSF grants 2221009 and 2218773, 
the Wharton AI \& Analytics Initiative's AI Research Fund, 
and the Amazon Research Award. 
G.~Li is supported in part by the Chinese University of Hong Kong Direct Grant for Research and the Hong Kong Research Grants Council ECS 24305724 and GRF 14307525.
C.~Cai is supported in part by the NSF grants CAREER award CCF-2541600 and DMS-2515333. 

\appendix

\section{Proof of Theorem \ref{thm:main}}
\label{sec:analysis}





For convenience, we introduce the following notations: for any vector $X_{t_i}\in\mathbb{R}^d$, denote
$$
X_{t_i}^{\mathsf{st}}\coloneqq V_{\Sti{t_i}}^{\top}{X}_{t_i},\qquad {X}_{t_i}^{\mathsf{sc}}\coloneqq V_{\Stci{t_i}}^{\top}X_{t_i},\qquad X_{t_i}^{\mathsf{c}}\coloneqq V_{\mathcal{S}^{\mathsf{c}}}^{\top}X_{t_i},\qquad X_{t_i}^{\mathsf{s}}\coloneqq V_{\mathcal{S}}^{\top}X_{t_i},
$$
the projection of observation
$\overline{Y}_{t_i}\coloneqq \Sigma_{\Sti{t_i}}^{-1}U_{\Sti{t_i}}^{\top} Y$, and the noise $\overline{Z}_{t_i}\coloneqq [Z_s]_{s\in {\Sti{t_i}}}$ used to update $V_{\Sti{t_i}}^{\top}\widehat{X}_{t_i}$. Additionally, throughout this paper, we denote by $\mathsf{KL}(p \,\Vert\, q)$ the Kullback-Leibler (KL) divergence between two probability measures $p$ and $q$.

We intend to establish the following stronger conclusion than Theorem \ref{thm:main}: for any $t_i$, conditioned on $V_{\Sti{t_i}}^{\top}\widehat{X}_{t_i}$ and $v_{s}^{\top}\widehat{X}_{\tau_s}$for all $\tau_s> t_i$, we have
\begin{align*}
    \mathsf{TV}\big(p_{\widehat{X}_{t_i}\mid V_{\Sti{t_i}}^{\top}\widehat{X}_{t_i}=V_{\Sti{t_i}}^{\top}x_{t_i}, v_{s}^{\top}\widehat{X}_{\tau_s}=v_s^{\top}x_{\tau_s}}, p_{X_{t_i} \mid V_{\Sti{t_i}}^{\top}X_{t_i} = V_{\Sti{t_i}}^{\top}x_{t_i}, v_{s}^{\top}X_{\tau_s} = v_{s}^{\top}x_{\tau_s}}) \to 0.
\end{align*}
When $t_M$ is sufficiently close to $0$, the set $\Sti{t_M}$ becomes empty, and we obtain
\begin{align*}
    \mathsf{TV}\big(p_{\widehat{X}_{t_M}\mid  v_{s}^{\top}\widehat{X}_{\tau_s}=v_s^{\top}x_{\tau_s}}, p_{X_{t_M} \mid  v_{s}^{\top}X_{\tau_s} = v_{s}^{\top}x_{\tau_s}}) \to 0.
\end{align*}
Moreover, noticing that $v_{s}^{\top}\widehat{X}_{\tau_s}=\Sigma_{ss}^{-1}u_s^{\top}Y$, and $(X_{t_M},v_s^{\top}X_{\tau_s})$ have the same joint distribution as $(X_{t_M},\Sigma_{ss}^{-1}u_s^{\top}Y)$, we can complete the proof of Theorem \ref{thm:main}.
\paragraph{Step 1. construct auxiliary sequences.}
We begin with introducing some auxiliary sequences.
Note that $\sum_{i=1}^{M-1}\delta_{i}=\Theta(1)$.
Denote
\begin{align}\label{eq:def-plambd}
\gamma_i \coloneqq \frac{\eta_1(1-{\rm e}^{-2\delta_{i+1}})}{2}\sigma_{t_{i+1}}^2,\qquad p_{\lambda_i}(\cdot,\cdot) \coloneqq p_{X_{t_i}^{\mathsf{c}},X_{t_i}^{\mathsf{sc}}|X_{t_i}^{\mathsf{st}}}(\cdot,\cdot|\alpha_{t_{i}}\overline{Y}_{t_{i}} + \widehat{\sigma}_{t_{i}}\overline{Z}_{t_{i}}),
\end{align}
where $\widehat{\sigma}_{t_i}$ is defined in \eqref{eq:def-xhat-t}.
We have for $\delta_{i+1}\le 1/2$,
\begin{align}\label{eq:proof-bound-gamma}
\frac12\eta_1\delta_{i+1}\sigma_{t_{i+1}}^2 = \frac{\eta_1 \delta_{i+1}}{2}\sigma_{t_{i+1}}^2 \le \gamma_i \le \frac{\eta_1 \cdot 2\delta_{i+1}}{2} \sigma_{t_{i+1}}^2 =\eta_1\delta_{i+1}\sigma_{t_{i+1}}^2.
\end{align}

An auxiliary sequence $\widetilde{X}_{k}$ with initialization $\widetilde{X}_0 = \widehat{X}_{t_0}$ is iteratively defined as 
\begin{subequations}\label{eq:def-auxiliary-sequence}
\begin{align}
    \widetilde{X}_{k+1}^{\mathsf{st}} &=\alpha_{t_{k+1}}\overline{Y}_{t_{k+1}} + \widehat{\sigma}_{t_{k+1}}\overline{Z}_{t_{k+1}},\notag\\
\widetilde{X}_{k+1}^{\mathsf{sc}} &= \frac{\alpha_{t_{k+1}}}{\alpha_{t_k}}\widetilde{X}_{t_{k}}^{\mathsf{sc}} + \frac{\alpha_{t_{k+1}}\sigma_{t_k}^2}{\alpha_{t_{k}}}\left(1-{\rm e}^{-2\delta_{k+1}}\right)\nabla_{\widetilde{X}_{k}^{\mathsf{sc}}} \log p_{\lambda_k}(\widetilde{X}_{k}^{\mathsf{c}},\widetilde{X}_{k}^{\mathsf{sc}}) + \sigma_{t_{k+1}}\sqrt{1 - e^{-2\delta_{k+1}}} Z_{k}^{\mathsf{sc}},\notag\\
\widetilde{X}_{k+1}^{\mathsf{c}} &= \widetilde{X}_{k}^{\mathsf{c}} +  \gamma_k \nabla_{\widetilde{X}_k^{\mathsf{c}}} \log p_{\lambda_{k+1}}(\widetilde{X}_{k}^{\mathsf{c}},\widetilde{X}_{k+1}^{\mathsf{sc}}) + \sqrt{2\gamma_k} Z_k^{\mathsf{c}},
\end{align}
\end{subequations}
where $Z_k^{\mathsf{sc}}$ and $Z_k^{\mathsf{c}}$ are independent Gaussian random processes.
Another auxiliary sequence is defined as
\begin{align}\label{eq:def-auxiliary-sequence-bar}
\overline{X}_{k+1}^{\mathsf{st}} &=\alpha_{t_{k+1}}\overline{Y}_{t_{k+1}} + \widehat{\sigma}_{t_{k+1}}\overline{Z}_{t_{k+1}},\quad
\overline{X}_{k+1}^{\mathsf{sc}} \sim p_{X_{k+1}^{\mathsf{sc}}\mid X_k}(\cdot\mid \overline{X}_{k}),\quad
\overline{X}_{k+1}^{\mathsf{c}} \sim p_{X_{k+1}^{\mathsf{c}}\mid X_{k+1}^{\mathsf{s}}}(\cdot\mid \overline{X}_{k+1}^{\mathsf{s}}).
\end{align}
Based on these sequences, we have the following lemma, whose proof is postponed to Section \ref{subsec:proof-lem-convergence-auxiliary-sequence}.
\begin{lemma}\label{lem:convergence-auxiliary-sequence}
The auxiliary sequence defined in \eqref{eq:def-auxiliary-sequence} satisfies
\begin{align}
\mathsf{KL}(p_{\widetilde{X}_{0:M}\mymid Y,Z} \,\Vert\, p_{\widehat{X}_{t_{0:M}}\mymid Y,Z})\to 0,\qquad \mathrm{as}\quad \eta_1^4\delta^3\to 0, \quad \mathrm{and}\quad \eta_1\to\infty.
\end{align}
This implies that conditioned on observation $Y$ and Gaussian noise $Z$, random variables $\widetilde{X}_{0:M}$ converges to the same distribution with $\widehat{X}_{t_{0:M}}$ if this limit exists.
Moreover, the sequence defined in \eqref{eq:def-auxiliary-sequence-bar} satisfies $p_{\overline{X}_k \mymid V_{\Sti{t_k}}^{\top}\overline{X}_{k}, v_{s}^{\top}\overline{X}_{k_{\tau_s}}} = p_{X_{t_k} \mymid V_{\Sti{t_k}}^{\top}X_{t_k}, v_{s}^{\top}X_{\tau_s}}$, where $k_{\tau_s}$ satisfies $t_{k_{\tau_s}} = \tau_s$.
\end{lemma}
This lemma tells us that the marginal distribution $\mathsf{TV}\big(p_{\widetilde{X}_{k}\mymid Y,Z}, p_{\widehat{X}_{k}\mymid Y,Z}\big) \to 0$. 
Denote $\overline{Z}_{t_i}^{\mathsf{c}} = [Z_s]_{s\in\Stci{t_i}}$. 
The convergence of marginal distribution leads to 
\begin{align}\label{eq:proof-thm-1}
    &\mathsf{TV}\big(p_{\widetilde{X}_{k}\mymid V_{\Sti{t_k}}^{\top}\widetilde{X}_{k}, v_{s}^{\top}\widetilde{X}_{\tau_s}}(\cdot,x^{\mathsf{st}},v_{s}^{\top}{x}_{\tau_s}), p_{\widehat{X}_{t_k}\mymid V_{\Sti{t_k}}^{\top}\widehat{X}_{t_k}, v_{s}^{\top}\widehat{X}_{\tau_s}}(\cdot,x^{\mathsf{st}},v_{s}^{\top}{x}_{\tau_s})\big)\notag\\
    &\quad \le \mathbb{E}_{\overline{Y}_{t_k},\overline{Z}_{t_k}^{\mathsf{c}}\mid V_{\Sti{t_k}}^{\top}\widetilde{X}_{k}}[\mathsf{TV}\big(p_{\widetilde{X}_{k}\mymid V_{\Sti{t_k}}^{\top}\widetilde{X}_{k}, v_{s}^{\top}\widetilde{X}_{\tau_s},\overline{Y}_{t_k},\overline{Z}_{t_k}^{\mathsf{c}}}(\cdot\mymid x^{\mathsf{st}},v_{s}^{\top}{x}_{\tau_s},\overline{y}_{t_k},\widehat{z}), \notag\\
    &\qquad \qquad \qquad \qquad p_{\widehat{X}_{t_k}\mymid V_{\Sti{t_k}}^{\top}\widehat{X}_{t_k}, v_{s}^{\top}\widehat{X}_{\tau_s},\overline{Y}_{t_k},\overline{Z}_{t_k}^{\mathsf{c}}}(\cdot\mymid x^{\mathsf{st}},v_{s}^{\top}{x}_{\tau_s},\overline{y}_{t_k},\widehat{z})\big)]\notag\\
    &\quad = \mathbb{E}_{\overline{Y}_{t_k},\overline{Z}_{t_k}^{\mathsf{c}}\mid V_{\Sti{t_k}}^{\top}\widetilde{X}_{k}}[\mathsf{TV}\big(p_{\widetilde{X}_{k}\mymid Y,Z}(\cdot\mymid y,z), p_{\widehat{X}_{t_k}\mymid Y,Z}(\cdot\mymid y,z)\big)]\to 0,
\end{align}
where $v_s^{\top}x_{\tau_s} = \Sigma_{ss}^{-1}u_s^{\top}y$, $\overline{z}_{t_k} = \widehat{\sigma}_{t_k}^{-1}(x^{\mathsf{st}} - \alpha_{t_k}\overline{y}_{t_k})$, and the last equation holds since $p_{\widetilde{X}_{k}\mymid V_{\Sti{t_k}}^{\top}\widetilde{X}_{k}, v_{s}^{\top}\widetilde{X}_{\tau_s},\overline{Y}_{t_k},\overline{Z}_{t_k}^{\mathsf{c}}} = p_{\widetilde{X}_{k}\mymid Y,Z}$ and $p_{\widehat{X}_{k}\mymid V_{\Sti{t_k}}^{\top}\widehat{X}_{t_k}, v_{s}^{\top}\widehat{X}_{\tau_s},\overline{Y}_{t_k},\overline{Z}_{t_k}^{\mathsf{c}}} = p_{\widehat{X}_{t_k}\mymid Y,Z}$.
Based on Lemma \ref{lem:convergence-auxiliary-sequence} and \eqref{eq:proof-thm-1}, it suffices to prove that for any $x^{\mathsf{st}}$ and $x_{\tau_s}$,
\begin{align*}
\mathsf{TV}\big(p_{\widetilde{X}_k\mymid \widetilde{X}_{k}^{\mathsf{st}}, v_{s}^{\top}\widetilde{X}_{\tau_s}}(\cdot\mymid x^{\mathsf{st}}, v_{s}^{\top}x_{\tau_s}),p_{\overline{X}_{t_k}\mymid \overline{X}_{t_k}^{\mathsf{st}}, v_{s}^{\top}\overline{X}_{\tau_s}}(\cdot\mymid x^{\mathsf{st}}, v_{s}^{\top}x_{\tau_s})\big)\to 0.
\end{align*}
By a similar argument as in \eqref{eq:proof-thm-1}, it suffices to prove that for any $y$, $z$,
\begin{align*}
\mathsf{TV}\big(p_{\widetilde{X}_k\mymid Y,Z}(\cdot\mymid y,z),p_{\overline{X}_{t_k}\mymid Y,Z}(\cdot\mymid y,z)\big)\to 0.
\end{align*}

\paragraph{Step 2. transform to the analysis of auxiliary sequences.}
We begin with the decomposition that
\begin{align*}
    p_{\widetilde{X}_k\mymid Y, Z}(x|y,z) 
    &= p_{\widetilde{X}_k^{\mathsf{c}}\mymid \widetilde{X}_k^{\mathsf{s}},Y,Z}(x^{\mathsf{c}}|x^{\mathsf{s}},y,z)  p_{\widetilde{X}_k^{\mathsf{sc}}\mymid Y,Z}(x^{\mathsf{sc}}|y,z),\notag\\
p_{\overline{X}_{k}\mymid Y, Z}(x\mymid y,z) 
&= p_{\overline{X}_{k}^{\mathsf{c}}\mymid \overline{X}_{k}^{\mathsf{s}},Y,Z}(x^{\mathsf{c}}\mymid x^{\mathsf{s}},y,z)p_{\overline{X}_{k}^{\mathsf{sc}}\mymid Y,Z}(x^{\mathsf{sc}}\mymid y,z)
\end{align*}
Then the total variation distance can be decomposed by
\begin{align}\label{eq:proof-24}
&\mathsf{TV}\Big(p_{\widetilde{X}_k\mymid Y,Z}, p_{\overline{X}_{k}\mymid Y,Z}\Big)  \le \mathsf{TV}\Big( p_{\widetilde{X}_{k}^{\mathsf{sc}}\mymid Y,Z}, p_{\overline{X}_{k}^{\mathsf{sc}}\mymid Y,Z}\Big)  \notag\\
&\qquad\qquad +\mathbb{E}_{x^{\mathsf{sc}}\sim p_{\widetilde{X}_{k}^{\mathsf{sc}}\mymid Y,Z}}\bigg[\mathsf{TV}\bigg(p_{\widetilde{X}_{k}^{\mathsf{c}}\mymid \widetilde{X}_{k}^{\mathsf{s}},Y,Z}(\cdot|x^{\mathsf{s}},y,z), p_{\overline{X}_{k}^{\mathsf{c}}\mymid \overline{X}_{k}^{\mathsf{s}},Y,Z}(\cdot|x^{\mathsf{s}},y,z)\bigg)\bigg].
\end{align}
The convergence of the second terms comes from the following lemma:
\begin{lemma}\label{lem:distri-xc}
Suppose that $\eta_1$ is large enough and $\delta$ small enough. As $\eta_1 \delta\log\frac{1}{\delta} \to 0$ and ${\eta_1^2 \delta}/{\log\frac{1}{\delta}}\to\infty$, for any $x$,
\begin{align*}
\mathbb{E}_{x\sim p_{\widehat{X}_t}}\big[\mathsf{TV}\left(p_{V_{\mathcal{S}^{\mathrm{c}}}^{\top}\wh X_{t} \mid V_{\mathcal{S}}^{\top}\wh X_{t} }(\cdot\mid V_{\mathcal{S}}^{\top} x), p_{V_{\mathcal{S}^{\mathrm{c}}}^{\top}X_{t} \mid V_{\mathcal{S}}^{\top}X_{t}}(\cdot\mid V_{\mathcal{S}}^{\top} x)\right)\big]\to 0.
\end{align*}
\end{lemma}
The proof is postponed to Appendix \ref{subsec:proof-lem-distri-xc}.
Here we use a more strong conclusion than Lemma \ref{lem:distri-xc}, i.e., for any $y$, $z$,
\begin{align}\label{eq:distri-xc}
\mathbb{E}_{x^{\mathsf{sc}}\sim\widetilde{X}_k^{\mathsf{sc}}\mymid Y,Z}\big[\mathsf{TV}\left(p_{\widetilde{X}_{k}^{\mathsf{c}} \mid \widetilde{X}_{k}^{\mathsf{s}}, Y,Z}(\cdot\mid x^{\mathsf{s}}, y,z), p_{X_{t_k}^{\mathsf{c}} \mid X_{t_k}^{\mathsf{s}}}(\cdot\mid  x^{\mathsf{s}})\right)\big]\to 0.
\end{align}
The proof is postponed to Section \ref{subsec:proof-lem-distri-xc} (cf. \eqref{eq:proof-lem-distri-sc-5}).
Then the second term goes to zero:
\begin{align}\label{eq:proof-25}
    &\mathbb{E}_{x^{\mathsf{sc}}\sim p_{\widetilde{X}_{k}^{\mathsf{sc}}\mymid Y,Z}}\bigg[\mathsf{TV}\bigg(p_{\widetilde{X}_{k}^{\mathsf{c}}\mymid \widetilde{X}_{k}^{\mathsf{s}},Y,Z}(\cdot|x^{\mathsf{s}},y,z), p_{\overline{X}_{k}^{\mathsf{c}}\mymid \overline{X}_{k}^{\mathsf{s}},Y,Z}(\cdot|x^{\mathsf{s}},y,z)\bigg)\bigg] \notag\\
&\quad \overset{\text{(i)}}{=} \mathbb{E}_{x^{\mathsf{sc}}\sim p_{\widetilde{X}_{k}^{\mathsf{sc}}\mymid Y,Z}}\bigg[\mathsf{TV}\bigg(p_{\widetilde{X}_{k}^{\mathsf{c}}\mymid \widetilde{X}_{k}^{\mathsf{s}},Y,Z}(\cdot|x^{\mathsf{s}},y,z), p_{{X}_{t_k}^{\mathsf{c}}\mymid {X}_{t_k}^{\mathsf{s}}}(\cdot|x^{\mathsf{s}})\bigg)\bigg]\to 0,
\end{align}
where (i) arises from the fact that $p_{\overline{X}_{k}^{\mathsf{c}}\mymid \overline{X}_{k}^{\mathsf{s}},Y,Z} = p_{{X}_{k}^{\mathsf{c}}\mymid {X}_{k}^{\mathsf{s}}}$.
It suffices to prove the first term goes to zero, i.e.,
\begin{align}\label{eq:proof-thm-4}
\mathsf{TV}\big(p_{\widetilde{X}_k^{\mathsf{sc}}\mymid Y,Z}, p_{\overline{X}_{k}^{\mathsf{sc}}\mymid Y,Z}\big)\to 0.
\end{align}

\paragraph{Step 3. decompose the error terms.}
By virtue of Pinsker's inequality, we further decompose the total variation distance as 
\begin{align}\label{eq:proof-20}
&\mathsf{TV}\big(p_{\widetilde{X}_k^{\mathsf{sc}}\mymid Y,Z}, p_{\overline{X}_{k}^{\mathsf{sc}}\mymid Y,Z}\big)^2 \notag\\
&\le \mathsf{KL}\bigg(p_{\widetilde{X}_k^{\mathsf{sc}}\mymid Y,Z}\,\Vert\, p_{\overline{X}_{k}^{\mathsf{sc}}\mymid Y,Z}\bigg)\notag\\
&\le \sum_{i=1}^k\mathbb{E}_{x_{i-1}^{\mathsf{sc}}\sim p_{\widetilde{X}_{i-1}^{\mathsf{sc}}}\mymid Y,Z}\bigg[\mathsf{KL}\bigg(p_{\widetilde{X}_{i}^{\mathsf{sc}}\mymid \widetilde{X}_{i-1}^{\mathsf{s}},Y,Z}(\cdot\mymid x_{i-1}^{\mathsf{s}},y,z) \,\Vert\, p_{\overline{X}_i^{\mathsf{sc}}\mymid \overline{X}_{i-1}^{\mathsf{s}},Y,Z}(\cdot\mymid x_{i-1}^{\mathsf{s}},y,z)\bigg)\bigg]+ \mathsf{KL}(p_{\widetilde{X}_{0}^{\mathsf{sc}}}\,\Vert\, p_{\overline{X}_{0}^{\mathsf{sc}}})\notag\\
&= \sum_{i=1}^k\mathbb{E}_{x_{i-1}^{\mathsf{sc}}\sim p_{\widetilde{X}_{i-1}^{\mathsf{sc}}}\mymid Y,Z}\bigg[\mathsf{KL}\bigg( p_{\widetilde{X}_{i}^{\mathsf{sc}}\mymid \widetilde{X}_{i-1}^{\mathsf{s}},Y,Z}(\cdot|x_{i-1}^{\mathsf{s}},y,z)\Vert p_{\overline{X}_i^{\mathsf{sc}}\mymid \overline{X}_{i-1}^{\mathsf{s}}}(\cdot|x_{i-1}^{\mathsf{s}})\bigg)\bigg] + \mathsf{KL}(p_{\widetilde{X}_{0}^{\mathsf{sc}}}\,\Vert\, p_{\overline{X}_{0}^{\mathsf{sc}}}),
\end{align}
where $x_{i-1}^{\mathsf{s}}$ satisfies $x_{i-1}^{\mathsf{st}} = \alpha_{t_{i-1}}\Sigma_{\mathcal{S}_{t_{i-1}}}^{-1}U_{\mathcal{S}_{t_{i-1}}}^{\top}y + \widehat{\sigma}_{t_{i-1}}[Z_s]_{s\in\mathcal{S}_{t_{i-1}}}$, and the last line holds since $p_{\overline{X}_{i}^{\mathsf{sc}}\mymid \overline{X}_{i-1}^{\mathsf{s}},Y,Z} = p_{\overline{X}_{i}^{\mathsf{sc}}\mymid \overline{X}_{i-1}^{\mathsf{s}}}$.
According to the definition of $\overline{X}_i$, we further have \begin{align}\label{eq:proof-thm-2}
&\mathsf{KL}\bigg( p_{\widetilde{X}_{i}^{\mathsf{sc}}\mymid \widetilde{X}_{i-1}^{\mathsf{s}},Y,Z}(\cdot|x_{i-1}^{\mathsf{s}},y,z)\,\Vert\, p_{\overline{X}_i^{\mathsf{sc}}\mymid \overline{X}_{i-1}^{\mathsf{s}}}(\cdot|x_{i-1}^{\mathsf{s}})\bigg)\notag\\
&\quad \overset{\text{(i)}}{=} \mathsf{KL}\bigg( p_{\widetilde{X}_{i}^{\mathsf{sc}}\mymid \widetilde{X}_{i-1}^{\mathsf{s}},Y,Z}(\cdot|x_{i-1}^{\mathsf{s}},y,z) \,\Vert\,p_{{X}_{t_i}^{\mathsf{sc}}\mymid {X}_{t_{i-1}}^{\mathsf{s}}}(\cdot|x_{i-1}^{\mathsf{s}})\bigg)\notag\\
&\quad =\int \log \frac{p_{\widetilde{X}_{i}^{\mathsf{sc}}\mymid \widetilde{X}_{i-1}^{\mathsf{s}},Y,Z}(x_i^{\mathsf{sc}}\mymid x_{i-1}^{\mathsf{s}},y,z)}{p_{{X}_{t_i}^{\mathsf{sc}}\mymid {X}_{t_{i-1}}^{\mathsf{s}}}(x_i^{\mathsf{sc}}\mymid x_{i-1}^{\mathsf{s}})} p_{\widetilde{X}_{i}^{\mathsf{sc}}\mymid \widetilde{X}_{i-1}^{\mathsf{s}},Y,Z}(x_i^{\mathsf{sc}}\mymid x_{i-1}^{\mathsf{s}},y,z)\mathrm{d} {x}_i^{\mathsf{sc}},
\end{align}
where (i) holds since
\begin{align*}
p_{\overline{X}_{i}^{\mathsf{sc}}\mymid \overline{X}_{i-1}^{\mathsf{s}}}(x_i^{\mathsf{sc}}\mymid x_{i-1}^{\mathsf{s}}) 
&= \int p_{\overline{X}_{i}^{\mathsf{sc}}\mymid \overline{X}_{i-1}}(x_i^{\mathsf{sc}}\mymid x_{i-1})p_{\overline{X}_{i-1}^{\mathsf{c}}\mymid \overline{X}_{i-1}^{\mathsf{s}}}(x_{i-1}^{\mathsf{c}}\mymid x_{i-1}^{\mathsf{s}}) \mathrm{d}{x}_{i-1}^{\mathsf{c}}\notag\\
&= \int p_{{X}_{t_i}^{\mathsf{sc}}\mymid {X}_{t_{i-1}}}(x_i^{\mathsf{sc}}\mymid x_{i-1})p_{{X}_{t_{i-1}}^{\mathsf{c}}\mymid {X}_{t_{i-1}}^{\mathsf{s}}}(x_{i-1}^{\mathsf{c}}\mymid x_{i-1}^{\mathsf{s}}) \mathrm{d}{x}_{i-1}^{\mathsf{c}} = p_{{X}_{t_i}^{\mathsf{sc}}\mymid {X}_{t_{i-1}}^{\mathsf{s}}}(x_i^{\mathsf{sc}}\mymid x_{i-1}^{\mathsf{s}}) .
\end{align*}
We further decompose the KL divergence in \eqref{eq:proof-thm-2} as
\begin{align}\label{eq:proof-21}
&\mathsf{KL}\bigg( p_{\widetilde{X}_{i}^{\mathsf{sc}}\mymid \widetilde{X}_{i-1}^{\mathsf{s}},Y,Z}(\cdot|x_{i-1}^{\mathsf{s}},y,z)\,\Vert\, p_{\overline{X}_i^{\mathsf{sc}}\mymid \overline{X}_{i-1}^{\mathsf{s}}}(\cdot|x_{i-1}^{\mathsf{s}})\bigg)\notag\\
&\quad = \int \log \frac{\int p_{\widetilde{X}_i^{\mathsf{sc}}\mymid \widetilde{X}_{i-1},Y,Z}({x}_{i}^{\mathsf{sc}}\mymid {x}_{i-1},y,z)p_{\widetilde{X}_{i-1}^{\mathsf{c}}\mymid \widetilde{X}_{i-1}^{\mathsf{s}},Y,Z}({x}_{i-1}^{\mathsf{c}}\mymid {x}_{i-1}^{\mathsf{s}},y,z) \mathrm{d}{x}_{i-1}^{\mathsf{c}}}{\int p_{{X}_{t_i}^{\mathsf{sc}}\mymid {X}_{t_{i-1}}}({x}_{i}^{\mathsf{sc}}\mymid {x}_{i-1})p_{{X}_{t_{i-1}}^{\mathsf{c}}\mymid {X}_{t_{i-1}}^{\mathsf{s}}}({x}_{i-1}^{\mathsf{c}}\mymid {x}_{i-1}^{\mathsf{s}}) \mathrm{d}{x}_{i-1}^{\mathsf{c}}} p_{\widetilde{X}_{i}^{\mathsf{sc}}\mymid \widetilde{X}_{{i-1}}^{\mathsf{s}},Y,Z}({x}_i^{\mathsf{sc}}\mymid {x}_{i-1}^{\mathsf{s}},y,z)\mathrm{d} {x}_i^{\mathsf{sc}}\notag\\
&\quad=\int \log \frac{\int p_{\widetilde{X}_{i}^{\mathsf{sc}}\mymid \widetilde{X}_{{i-1}},Y,Z}({x}_{i}^{\mathsf{sc}}\mymid {x}_{i-1},y,z)p_{{X}_{t_{i-1}}^{\mathsf{c}}\mymid {X}_{t_{i-1}}^{\mathsf{s}}}({x}_{i-1}^{\mathsf{c}}\mymid {x}_{i-1}^{\mathsf{s}}) \mathrm{d}{x}_{i-1}^{\mathsf{c}}}{\int p_{{X}_{t_i}^{\mathsf{sc}}\mymid {X}_{t_{i-1}}}({x}_{i}^{\mathsf{sc}}\mymid {x}_{i-1})p_{{X}_{t_{i-1}}^{\mathsf{c}}\mymid {X}_{t_{i-1}}^{\mathsf{s}}}({x}_{i-1}^{\mathsf{c}}\mymid {x}_{i-1}^{\mathsf{s}}) \mathrm{d}{x}_{i-1}^{\mathsf{c}}} p_{\widetilde{X}_{i}^{\mathsf{sc}}\mymid \widetilde{X}_{{i-1}}^{\mathsf{s}},Y,Z}({x}_i^{\mathsf{sc}}\mymid {x}_{i-1}^{\mathsf{s}},y,z)\mathrm{d} {x}_i^{\mathsf{sc}}\notag\\
&\qquad  + \int \log \frac{\int p_{\widetilde{X}_{i}^{\mathsf{sc}}\mymid \widetilde{X}_{i-1},Y,Z}({x}_{i}^{\mathsf{sc}}\mymid {x}_{i-1},y,z)p_{\widetilde{X}_{i-1}^{\mathsf{c}}\mymid \widetilde{X}_{i-1}^{\mathsf{s}},Y,Z}({x}_{i-1}^{\mathsf{c}}\mymid {x}_{i-1}^{\mathsf{s}},y,z) \mathrm{d}{x}_{i-1}^{\mathsf{c}}}{\int p_{\widetilde{X}_{i}^{\mathsf{sc}}\mymid \widetilde{X}_{i-1},Y,Z}({x}_{i}^{\mathsf{sc}}\mymid {x}_{i-1},y,z)p_{{X}_{t_{i-1}}^{\mathsf{c}}\mymid {X}_{t_{i-1}}^{\mathsf{s}}}({x}_{i-1}^{\mathsf{c}}\mymid {x}_{i-1}^{\mathsf{s}}) \mathrm{d}{x}_{i-1}^{\mathsf{c}}} p_{\widetilde{X}_{i}^{\mathsf{sc}}\mymid \widetilde{X}_{{i-1}}^{\mathsf{s}},Y,Z}({x}_i^{\mathsf{sc}}\mymid {x}_{i-1}^{\mathsf{s}},y,z)\mathrm{d} {x}_i^{\mathsf{sc}}.
\end{align}
The first term is further decomposed by
\begin{align}\label{eq:proof-26}
&\int \log \frac{\int p_{\widetilde{X}_{i}^{\mathsf{sc}}\mymid \widetilde{X}_{{i-1}},Y,Z}({x}_{i}^{\mathsf{sc}}\mymid {x}_{i-1},y,z)p_{{X}_{t_{i-1}}^{\mathsf{c}}\mymid {X}_{t_{i-1}}^{\mathsf{s}}}({x}_{i-1}^{\mathsf{c}}\mymid {x}_{i-1}^{\mathsf{s}}) \mathrm{d}{x}_{i-1}^{\mathsf{c}}}{\int p_{{X}_{t_i}^{\mathsf{sc}}\mymid {X}_{t_{i-1}}}({x}_{i}^{\mathsf{sc}}\mymid {x}_{i-1})p_{{X}_{t_{i-1}}^{\mathsf{c}}\mymid {X}_{t_{i-1}}^{\mathsf{s}}}({x}_{i-1}^{\mathsf{c}}\mymid {x}_{i-1}^{\mathsf{s}}) \mathrm{d}{x}_{i-1}^{\mathsf{c}}} p_{\widetilde{X}_{i}^{\mathsf{sc}}\mymid \widetilde{X}_{{i-1}}^{\mathsf{s}},Y,Z}({x}_i^{\mathsf{sc}}\mymid {x}_{i-1}^{\mathsf{s}},y,z)\mathrm{d} {x}_i^{\mathsf{sc}}\notag\\
&\quad =\int  f(x_i^{\mathsf{sc}},x_{i-1}^{\mathsf{s}})\int p_{\widetilde{X}_{i}^{\mathsf{sc}}\mymid \widetilde{X}_{{i-1}},Y,Z}({x}_{i}^{\mathsf{sc}}\mymid {x}_{i-1},y,z)p_{{X}_{t_{i-1}}^{\mathsf{c}}\mymid {X}_{t_{i-1}}^{\mathsf{s}}}({x}_{i-1}^{\mathsf{c}}\mymid {x}_{i-1}^{\mathsf{s}}) \mathrm{d}{x}_{i-1}^{\mathsf{c}}\mathrm{d} {x}_i^{\mathsf{sc}}\notag\\
&\qquad +  \int f(x_i^{\mathsf{sc}},x_{i-1}^{\mathsf{s}})\int p_{\widetilde{X}_{i}^{\mathsf{sc}}\mymid \widetilde{X}_{{i-1}},Y,Z}({x}_{i}^{\mathsf{sc}}\mymid {x}_{i-1},y,z)\notag\\
&\qquad \cdot \left(p_{\widetilde{X}_{{i-1}}^{\mathsf{c}}\mymid \widetilde{X}_{{i-1}}^{\mathsf{s}},Y,Z}({x}_{i-1}^{\mathsf{c}}\mymid {x}_{i-1}^{\mathsf{s}},y,z)-p_{{X}_{t_{i-1}}^{\mathsf{c}}\mymid {X}_{t_{i-1}}^{\mathsf{s}}}({x}_{i-1}^{\mathsf{c}}\mymid {x}_{i-1}^{\mathsf{s}})\right) \mathrm{d}{x}_{i-1}^{\mathsf{c}}\mathrm{d} {x}_i^{\mathsf{sc}},
\end{align}
where 
$$
f(x_i^{\mathsf{sc}},x_{i-1}^{\mathsf{s}})\coloneqq \log \frac{\int p_{\widetilde{X}_{i}^{\mathsf{sc}}\mymid \widetilde{X}_{{i-1}},Y,Z}({x}_{i}^{\mathsf{sc}}\mymid {x}_{i-1},y,z)p_{{X}_{t_{i-1}}^{\mathsf{c}}\mymid {X}_{t_{i-1}}^{\mathsf{s}}}({x}_{i-1}^{\mathsf{c}}\mymid {x}_{i-1}^{\mathsf{s}}) \mathrm{d}{x}_{i-1}^{\mathsf{c}}}{\int p_{{X}_{t_i}^{\mathsf{sc}}\mymid {X}_{t_{i-1}}}({x}_{i}^{\mathsf{sc}}\mymid {x}_{i-1})p_{{X}_{t_{i-1}}^{\mathsf{c}}\mymid {X}_{t_{i-1}}^{\mathsf{s}}}({x}_{i-1}^{\mathsf{c}}\mymid {x}_{i-1}^{\mathsf{s}}) \mathrm{d}{x}_{i-1}^{\mathsf{c}}}.
$$
By using \eqref{eq:distri-xc}, the second term in the right-hand-side of \eqref{eq:proof-26} satisfies
\begin{align}\label{eq:proof-27}
&\mathbb{E}_{x_{i-1}^{\mathsf{sc}}\sim p_{\widetilde{X}_{i-1}^{\mathsf{sc}}\mymid Y,Z}}\Big[\int f(x_i^{\mathsf{sc}},x_{i-1}^{\mathsf{s}})\int p_{\widetilde{X}_{i}^{\mathsf{sc}}\mymid \widetilde{X}_{{i-1}},Y,Z}({x}_{i}^{\mathsf{sc}}\mymid {x}_{i-1},y,z)\notag\\
&\qquad\cdot \left(p_{\widetilde{X}_{{i-1}}^{\mathsf{c}}\mymid \widetilde{X}_{{i-1}}^{\mathsf{s}},Y,Z}({x}_{i-1}^{\mathsf{c}}\mymid {x}_{i-1}^{\mathsf{s}},y,z)-p_{{X}_{t_{i-1}}^{\mathsf{c}}\mymid {X}_{t_{i-1}}^{\mathsf{s}}}({x}_{i-1}^{\mathsf{c}}\mymid {x}_{i-1}^{\mathsf{s}})\right) \mathrm{d}{x}_{i-1}^{\mathsf{c}}\mathrm{d} {x}_i^{\mathsf{sc}}\Big]\notag\\
&\quad=\mathbb{E}_{x_{i-1}^{\mathsf{sc}}\sim p_{\widetilde{X}_{i-1}^{\mathsf{sc}}\mymid Y,Z}}\Big[\int \left(\int f(x_i^{\mathsf{sc}},x_{i-1}^{\mathsf{s}}) p_{\widetilde{X}_{i}^{\mathsf{sc}}\mymid \widetilde{X}_{{i-1}},Y,Z}({x}_{i}^{\mathsf{sc}}\mymid {x}_{i-1},y,z)\mathrm{d} {x}_i^{\mathsf{sc}} \right)\notag\\
&\qquad\quad\cdot \left(p_{\widetilde{X}_{{i-1}}^{\mathsf{c}}\mymid \widetilde{X}_{{i-1}}^{\mathsf{s}},Y,Z}({x}_{i-1}^{\mathsf{c}}\mymid {x}_{i-1}^{\mathsf{s}},y,z)-p_{{X}_{t_{i-1}}^{\mathsf{c}}\mymid {X}_{t_{i-1}}^{\mathsf{s}}}({x}_{i-1}^{\mathsf{c}}\mymid {x}_{i-1}^{\mathsf{s}})\right) \mathrm{d}{x}_{i-1}^{\mathsf{c}}\Big] \notag\\
&\quad \le \mathbb{E}_{x_{i-1}^{\mathsf{sc}}\sim p_{\widetilde{X}_{i-1}^{\mathsf{sc}}\mymid Y,Z}}\Big[\int \left(\int f(x_i^{\mathsf{sc}},x_{i-1}^{\mathsf{s}}) p_{\widetilde{X}_{i}^{\mathsf{sc}}\mymid \widetilde{X}_{{i-1}},Y,Z}({x}_{i}^{\mathsf{sc}}\mymid {x}_{i-1},y,z)\mathrm{d} {x}_i^{\mathsf{sc}} \right)^2\notag\\
&\qquad\qquad\cdot \left|p_{\widetilde{X}_{{i-1}}^{\mathsf{c}}\mymid \widetilde{X}_{{i-1}}^{\mathsf{s}},Y,Z}({x}_{i-1}^{\mathsf{c}}\mymid {x}_{i-1}^{\mathsf{s}},y,z)-p_{{X}_{t_{i-1}}^{\mathsf{c}}\mymid {X}_{t_{i-1}}^{\mathsf{s}}}({x}_{i-1}^{\mathsf{c}}\mymid {x}_{i-1}^{\mathsf{s}})\right| \mathrm{d}{x}_{i-1}^{\mathsf{c}}\Big]^{1/2}\notag\\
&\qquad\qquad \cdot\mathbb{E}_{x_{i-1}^{\mathsf{sc}}\sim p_{\widetilde{X}_{i-1}^{\mathsf{sc}}\mymid Y,Z}}\Big[\int \left|p_{\widetilde{X}_{{i-1}}^{\mathsf{c}}\mymid \widetilde{X}_{{i-1}}^{\mathsf{s}},Y,Z}({x}_{i-1}^{\mathsf{c}}\mymid {x}_{i-1}^{\mathsf{s}},y,z)-p_{{X}_{t_{i-1}}^{\mathsf{c}}\mymid {X}_{t_{i-1}}^{\mathsf{s}}}({x}_{i-1}^{\mathsf{c}}\mymid {x}_{i-1}^{\mathsf{s}})\right| \mathrm{d}{x}_{i-1}^{\mathsf{c}}\Big]^{1/2} \to 0.
\end{align}
For the first term in the right-hand-side of \eqref{eq:proof-26}, note that it is the KL divergence between the marginal distribution of $x_i^{\mathsf{sc}}$, and it is smaller than the KL divergence of the joint distribution of $x_i^{\mathsf{sc}}$ and $x_{i-1}^{\mathsf{c}}$.
Therefore we have
\begin{align}\label{eq:proof-22}
&\int \log \frac{\int p_{\widetilde{X}_{i}^{\mathsf{sc}}\mymid \widetilde{X}_{i-1},Y,Z}({x}_{i}^{\mathsf{sc}}\mymid {x}_{i-1},y,z)p_{{X}_{t_{i-1}}^{\mathsf{c}}\mymid {X}_{t_{i-1}}^{\mathsf{s}}}({x}_{i-1}^{\mathsf{c}}\mymid {x}_{i-1}^{\mathsf{s}}) \mathrm{d}{x}_{i-1}^{\mathsf{c}}}{\int p_{{X}_{t_i}^{\mathsf{sc}}\mymid {X}_{t_{i-1}}}({x}_{i}^{\mathsf{sc}}\mymid {x}_{i-1})p_{{X}_{t_{i-1}}^{\mathsf{c}}\mymid {X}_{t_{i-1}}^{\mathsf{s}}}({x}_{i-1}^{\mathsf{c}}\mymid {x}_{i-1}^{\mathsf{s}}) \mathrm{d}{x}_{i-1}^{\mathsf{c}}}\notag\\
&\qquad \cdot\int p_{\widetilde{X}_{i}^{\mathsf{sc}}\mymid \widetilde{X}_{i-1},Y,Z}({x}_{i}^{\mathsf{sc}}\mymid {x}_{i-1},y,z)p_{{X}_{t_{i-1}}^{\mathsf{c}}\mymid {X}_{t_{i-1}}^{\mathsf{s}}}({x}_{i-1}^{\mathsf{c}}\mymid {x}_{i-1}^{\mathsf{s}}) \mathrm{d}{x}_{i-1}^{\mathsf{c}}\mathrm{d} {x}_i^{\mathsf{sc}}\notag\\
&\quad \le \iint \log \frac{p_{\widetilde{X}_{i}^{\mathsf{sc}}\mymid \widetilde{X}_{{i-1}},Y,Z}({x}_{i}^{\mathsf{sc}}\mymid {x}_{i-1},y,z)p_{{X}_{t_{i-1}}^{\mathsf{c}}\mymid {X}_{t_{i-1}}^{\mathsf{s}}}({x}_{i-1}^{\mathsf{c}}\mymid {x}_{i-1}^{\mathsf{s}})}{p_{{X}_{t_i}^{\mathsf{sc}}\mymid {X}_{t_{i-1}}}({x}_{i}^{\mathsf{sc}}\mymid {x}_{i-1})p_{{X}_{t_{i-1}}^{\mathsf{c}}\mymid {X}_{t_{i-1}}^{\mathsf{s}}}({x}_{i-1}^{\mathsf{c}}\mymid {x}_{i-1}^{\mathsf{s}}) }\notag\\
&\qquad\qquad\times p_{\widetilde{X}_{i}^{\mathsf{sc}}\mymid \widetilde{X}_{{i-1}},Y,Z}({x}_{i}^{\mathsf{sc}}\mymid {x}_{i-1},y,z)p_{{X}_{t_{i-1}}^{\mathsf{c}}\mymid {X}_{t_{i-1}}^{\mathsf{s}}}({x}_{i-1}^{\mathsf{c}}\mymid {x}_{i-1}^{\mathsf{s}}) \mathrm{d}{x}_{i-1}^{\mathsf{c}} \mathrm{d} {x}_i^{\mathsf{sc}} \notag\\
&\quad = \mathbb{E}_{{x}_{i-1}^{\mathsf{c}}\sim X_{t_{i-1}}^{\mathsf{c}}\mymid X_{t_{i-1}}^{\mathsf{s}}}\big[\mathsf{KL}(p_{\widetilde{X}_{i}^{\mathsf{sc}}\mymid \widetilde{X}_{{i-1}},Y,Z}(\cdot\mymid  {x}_{i-1},y,z)\,\Vert\, p_{{X}_{t_i}^{\mathsf{sc}}\mymid {X}_{t_{i-1}}}(\cdot\mymid {x}_{i-1}) )\big].
\end{align}
Substituting \eqref{eq:proof-22} and \eqref{eq:proof-27} into \eqref{eq:proof-26}, we have
\begin{align}\label{eq:proof-28}
    &\int \log \frac{\int p_{\widetilde{X}_{i}^{\mathsf{sc}}\mymid \widetilde{X}_{{i-1}},Y,Z}({x}_{i}^{\mathsf{sc}}\mymid {x}_{i-1},y,z)p_{{X}_{t_{i-1}}^{\mathsf{c}}\mymid {X}_{t_{i-1}}^{\mathsf{s}}}({x}_{i-1}^{\mathsf{c}}\mymid {x}_{i-1}^{\mathsf{s}}) \mathrm{d}{x}_{i-1}^{\mathsf{c}}}{\int p_{{X}_{t_i}^{\mathsf{sc}}\mymid {X}_{t_{i-1}},Y,Z}({x}_{i}^{\mathsf{sc}}\mymid {x}_{i-1})p_{{X}_{t_{i-1}}^{\mathsf{c}}\mymid {X}_{t_{i-1}}^{\mathsf{s}}}({x}_{i-1}^{\mathsf{c}}\mymid {x}_{i-1}^{\mathsf{s}}) \mathrm{d}{x}_{i-1}^{\mathsf{c}}} p_{\widetilde{X}_{i}^{\mathsf{sc}}\mymid \widetilde{X}_{{i-1}}^{\mathsf{s}},Y,Z}({x}_i^{\mathsf{sc}}\mymid {x}_{i-1}^{\mathsf{s}},y,z)\mathrm{d} {x}_i^{\mathsf{sc}}\notag\\
    &\qquad  \to \mathbb{E}_{{x}_{i-1}^{\mathsf{c}}\sim X_{t_{i-1}}^{\mathsf{c}}\mymid X_{t_{i-1}}^{\mathsf{s}}}\big[\mathsf{KL}(p_{\widetilde{X}_{i}^{\mathsf{sc}}\mymid \widetilde{X}_{{i-1}},Y,Z}(\cdot\mymid  {x}_{i-1},y,z)\,\Vert\, p_{{X}_{t_i}^{\mathsf{sc}}\mymid {X}_{t_{i-1}}}(\cdot\mymid {x}_{i-1}) )\big].
\end{align}
Combining  \eqref{eq:proof-20}, \eqref{eq:proof-21} and \eqref{eq:proof-28}, we have
\begin{align}\label{eq:proof-23}
&\lim \mathsf{TV}\bigg( p_{\widetilde{X}_k^{\mathsf{sc}}\mymid Y,Z}, p_{\overline{X}_{k}^{\mathsf{sc}}\mymid Y,Z}\bigg)^2  \notag\\
&\quad\le \lim \sum_{i=1}^k\mathbb{E}_{{x}_{i-1}^{\mathsf{s}}\sim \widetilde{X}_{i-1}^{\mathsf{s}}\mymid Y,Z}\bigg[\mathbb{E}_{{x}_{i-1}^{\mathsf{c}}\sim X_{t_{i-1}}^{\mathsf{c}}\mymid X_{t_{i-1}}^{\mathsf{s}}}\big[\mathsf{KL}(p_{\widetilde{X}_{i}^{\mathsf{sc}}\mymid \widetilde{X}_{{i-1}},Y,Z}(\cdot\mymid  {x}_{i-1},y,z)\,\Vert\, p_{{X}_{t_i}^{\mathsf{sc}}\mymid {X}_{t_{i-1}}}(\cdot\mymid {x}_{i-1}))\big]\bigg]\notag\\
&\qquad +\lim\sum_{i=1}^k\mathbb{E}_{{x}_{i-1}^{\mathsf{s}}\sim \widetilde{X}_{i-1}^{\mathsf{s}}\mymid Y,Z}\bigg[\int \log \frac{\int p_{\widetilde{X}_{i}^{\mathsf{sc}}\mymid \widetilde{X}_{{i-1}},Y,Z}({x}_{i}^{\mathsf{sc}}\mymid {x}_{i-1},y,z)p_{\widetilde{X}_{i-1}^{\mathsf{c}}\mymid \widetilde{X}_{{i-1}}^{\mathsf{s}},Y,Z}({x}_{i-1}^{\mathsf{c}}\mymid {x}_{i-1}^{\mathsf{s}},y,z) \mathrm{d}{x}_{i-1}^{\mathsf{c}}}{\int p_{\widetilde{X}_{i}^{\mathsf{sc}}\mymid \widetilde{X}_{{i-1}},Y,Z}({x}_{i}^{\mathsf{sc}}\mymid {x}_{i-1},y,z)p_{{X}_{t_{i-1}}^{\mathsf{c}}\mymid {X}_{t_{i-1}}^{\mathsf{s}},Y,Z}({x}_{i-1}^{\mathsf{c}}\mymid {x}_{i-1}^{\mathsf{s}},y,z) \mathrm{d}{x}_{i-1}^{\mathsf{c}}}\notag\\
&\qquad\qquad\qquad\qquad\qquad\qquad\qquad\qquad\times p_{\widetilde{X}_{i}^{\mathsf{sc}}\mymid \widetilde{X}_{{i-1}}^{\mathsf{s}},Y,Z}({x}_i^{\mathsf{sc}}\mymid {x}_{i-1}^{\mathsf{s}},y,z)\mathrm{d} {x}_i^{\mathsf{sc}}\bigg].
\end{align}

\paragraph{Step 4. control error terms.}
Particularly, with $\eta_0 = 1$ which corresponds to DDPM, we have the following result, whose proof is postponed to Appendix \ref{subsec:proof-lem-distri-xsc}.
\begin{lemma}\label{lem:distri-xsc}
Suppose that $\eta_0=1$ and $\delta$ small enough. For any $x$,
\begin{align}
\mathsf{KL}\Big(p_{V_{\Stci{t_i}}^{\top}\widehat{X}_{t_i}\mid \widehat{X}_{t_{i-1}} = x} \parallel p_{V_{\Stci{t_i}}^{\top}X_{t_i}\mid {X}_{t_{i-1}} = x}\Big) = o(\delta_{i}).
\end{align}
\end{lemma}

By using Lemma \ref{lem:distri-xsc}, and notice that $\widetilde{X}_{i}^{\mathsf{sc}}\mymid \widetilde{X}_{i-1},Y,Z$ shares the same transition probability as $\widehat{X}_{t_i}^{\mathsf{sc}}\mymid \widehat{X}_{t_{i-1}}$, we have
\begin{align}\label{eq:proof-19}
    &\sum_{i=1}^k\mathbb{E}_{{x}_{i-1}^{\mathsf{s}}\sim \widetilde{X}_{i-1}^{\mathsf{s}}\mymid Y,Z}\bigg[\mathbb{E}_{{x}_{i-1}^{\mathsf{c}}\sim X_{t_{i-1}}^{\mathsf{c}}\mymid X_{t_{i-1}}^{\mathsf{s}}}\big[\mathsf{KL}(p_{\widetilde{X}_{i}^{\mathsf{sc}}\mymid \widetilde{X}_{{i-1}},Y,Z}(\cdot\mymid  {x}_{i-1},y,z)\,\Vert\, p_{{X}_{t_i}^{\mathsf{sc}}\mymid {X}_{t_{i-1}}}(\cdot\mymid {x}_{i-1}))\big]\bigg]\notag\\
        &\quad = \sum_{i=1}^k\mathbb{E}_{{x}_{i-1}^{\mathsf{s}}\sim \widetilde{X}_{i-1}^{\mathsf{s}}\mymid Y,Z; ~{x}_{i-1}^{\mathsf{c}}\sim X_{t_{i-1}}^{\mathsf{c}}\mymid X_{t_{i-1}}^{\mathsf{s}}={x}_{i-1}^{\mathsf{s}}}\bigg[\mathsf{KL}(p_{\widehat{X}_{t_i}^{\mathsf{sc}}\mymid \widehat{X}_{t_{i-1}}}(\cdot\mymid  {x}_{i-1})\,\Vert\, p_{{X}_{t_i}^{\mathsf{sc}}\mymid {X}_{t_{i-1}}}(\cdot\mymid {x}_{i-1}))\bigg] \notag\\
        &\quad = \sum_{i=1}^k o(\delta_i)\to 0,~\mathrm{as}~ \delta \to 0.
\end{align}
The second term describes the influence of the sampling error of $x_{i-1}^{\mathsf{c}}$ on the distribution of $x_i^{\mathsf{sc}}$.
It is given by the following lemma, whose proof is postponed to Section \ref{subsec:proof-lem-distri-sc-conds}.
\begin{lemma}\label{lem:distri-sc-conds}
If $p_{\widetilde{X}_{i-1}^{\mathsf{c}}\mymid \widetilde{X}_{i-1}^{\mathsf{s}}}({x}_{i-1}^{\mathsf{c}}\mymid {x}_{i-1}^{\mathsf{s}})\to p_{{X}_{t_{i-1}}^{\mathsf{c}}\mymid {X}_{t_{i-1}}^{\mathsf{s}}}({x}_{i-1}^{\mathsf{c}}\mymid {x}_{i-1}^{\mathsf{s}})$ in distribution, then we have
\begin{align}\label{eq:lem-distri-sc-conds}
    &\sum_{i=1}^k\mathbb{E}_{{x}_{i-1}^{\mathsf{s}}\sim \widetilde{X}_{i-1}^{\mathsf{s}}\mymid Y,Z}\bigg[\int \log \frac{\int p_{\widetilde{X}_{i}^{\mathsf{sc}}\mymid \widetilde{X}_{{i-1}}}({x}_{i}^{\mathsf{sc}}\mymid {x}_{i-1})p_{\widetilde{X}_{i-1}^{\mathsf{c}}\mymid \widetilde{X}_{{i-1}}^{\mathsf{s}},Y,Z}({x}_{i-1}^{\mathsf{c}}\mymid {x}_{i-1}^{\mathsf{s}},y,z) \mathrm{d}{x}_{i-1}^{\mathsf{c}}}{\int p_{\widetilde{X}_{i}^{\mathsf{sc}}\mymid \widetilde{X}_{{i-1}}}({x}_{i}^{\mathsf{sc}}\mymid {x}_{i-1})p_{{X}_{t_{i-1}}^{\mathsf{c}}\mymid {X}_{t_{i-1}}^{\mathsf{s}}}({x}_{i-1}^{\mathsf{c}}\mymid {x}_{i-1}^{\mathsf{s}})  \mathrm{d}{x}_{i-1}^{\mathsf{c}}}\notag\\
&\qquad\qquad\qquad\qquad\qquad\qquad\qquad\qquad\times p_{\widetilde{X}_{i}^{\mathsf{sc}}\mymid \widetilde{X}_{{i-1}}^{\mathsf{s}},Y,Z}({x}_i^{\mathsf{sc}}\mymid {x}_{i-1}^{\mathsf{s}},y,z)\mathrm{d} {x}_i^{\mathsf{sc}}\bigg] \to 0.
\end{align}
\end{lemma}
Combining \eqref{eq:proof-19} and Lemma \ref{lem:distri-sc-conds} together with \eqref{eq:proof-23}, we have
\begin{align}
\mathsf{TV}\big( p_{\widetilde{X}_k^{\mathsf{sc}}\mymid Y,Z}, p_{\overline{X}_{k}^{\mathsf{sc}}\mymid Y,Z}\big) \to 0.
\end{align}
Thus we complete the proof.

\section{Proof of auxiliary lemmas and facts}
\label{sec:proof-thm-main}

\subsection{Proof of \eqref{eq:DDIM-St}}
\label{app:proof-eq-ddim-st}

We first present the conclusion as the following lemma, and then prove it.
\begin{lemma}\label{lem:eq-ddim-st}
    The update rule \eqref{eq:DDIM-St} is equivalent to applying DDIM with $\eta=0$ by taking \eqref{eq:def-xhat-t} as forward process.
    That is,
    \begin{align}\label{eq:lem-ddim-st}
    \frac{v_s^{\top}\widehat{X}_{t_{i+1}}}{\widehat{\sigma}_{t_{i+1}}} = \frac{v_s^{\top}\widehat{X}_{t_{i}}}{\widehat{\sigma}_{t_i}}  +  {\rm e}^{\widehat{\lambda}_{t_i}}(1-{\rm e}^{-\widehat{\delta}_{{i+1}}})v_s^{\top}\mu_{t_i}(\widehat{X}_{t_i}),
    \end{align}
    where $\widehat{\alpha}_{t_i}=\alpha_{t_i}$, $\widehat{\sigma}_{t_i}^2=\sigma_{t_i}^2 - \alpha_{t_i}^2\sigma^2\Sigma_{ss}^{-1}$, $\widehat{\lambda}_{t_i} = \log(\widehat{\alpha}_{t_i}/\widehat{\sigma}_{t_i})$ and $\widehat{\delta}_{i+1} = \widehat{\lambda}_{t_{i+1}}-\widehat{\lambda}_{t_i}$, and $v_s^{\top}\mu_{t_i}(\widehat{X}_{t_i})=\Sigma_s^{-1}u_s^{\top}Y$.
\end{lemma}
\begin{proof}
By simplification, the update rule \eqref{eq:lem-ddim-st} is equivalent to
\begin{align}\label{eq:proof-lem-ddim-st-temp-1}
v_s^{\top}\widehat{X}_{t_{i+1}} = \frac{\widehat{\sigma}_{t_{i+1}}}{\widehat{\sigma}_{t_i}} v_s^{\top}\widehat{X}_{t_{i}} + \left(\widehat{\alpha}_{t_{i+1}} - \frac{\widehat{\sigma}_{t_{i+1}}\widehat{\alpha}_{t_i}}{\widehat{\sigma}_{t_i}}\right) v_s^{\top}\mu_{t_i}(\widehat{X}_{t_i}).
\end{align}
Assume that $v_s^{\top}\widehat{X}_{t_i} =\alpha_{t_i} \Sigma_{s}^{-1}U^{\top}Y + (\sigma_{t_i}^2 - \alpha_{t_i}^2\sigma^2\Sigma_{\mathcal{S}_{t_i}}^{-2})^{1/2} v_{s}^{\top}Z$.
Instituting into \eqref{eq:proof-lem-ddim-st-temp-1} leads to
\begin{align*}
v_s^{\top}\widehat{X}_{t_{i+1}} &= \frac{\widehat{\sigma}_{t_{i+1}}}{\widehat{\sigma}_{t_i}} \left(\alpha_{t_i} \Sigma_{s}^{-1}U^{\top}Y + (\sigma_{t_i}^2 - \alpha_{t_i}^2\sigma^2\Sigma_{\mathcal{S}_{t_i}}^{-2})^{1/2} v_{s}^{\top}Z\right) + \left(\widehat{\alpha}_{t_{i+1}} - \frac{\widehat{\sigma}_{t_{i+1}}\widehat{\alpha}_{t_i}}{\widehat{\sigma}_{t_i}}\right) \Sigma_{s}^{-1}u_s^{\top}Y\notag\\
&=\widehat{\alpha}_{t_{i+1}}\Sigma_{s}^{-1}u_s^{\top}Y + \frac{\widehat{\sigma}_{t_{i+1}}}{\widehat{\sigma}_{t_i}}(\sigma_{t_{i}}^2-\alpha_{t_{i}}^2\sigma^2\Sigma_{s}^{-2})^{\frac12}v_{s}^{\top}Z = {\alpha}_{t_{i+1}}\Sigma_{s}^{-1}u_s^{\top}Y + (\sigma_{t_{i+1}}^2-\alpha_{t_{i+1}}^2\sigma^2\Sigma_{s}^{-2})^{\frac12}v_{s}^{\top}Z,
\end{align*}
where the last equation arises from the definition of $\widehat{\sigma}_{t_i}$ and $\widehat{\alpha}_{t_{i}}$.
\end{proof}

\subsection{Preliminaries}

Throughout the proof, we redefine $\eta_1$ as $\eta$ for ease of notations.
We begin with approximating the update rule \eqref{eq:ddim-prior-dominated} for $\widehat{X}_{t_i}^{\mathsf{c}}$.
We have
\begin{align}\label{eq:proof-update-0}
\widehat{X}_{t_{i+1}}^{\mathsf{c}} &= \frac{\sigma_{t_{i+1}}}{\sigma_{t_{i}}}\sqrt{1-\eta(1 - e^{-2\delta_{i+1}})}\widehat{X}_{t_{i}}^{\mathsf{c}} + \alpha_{t_{i+1}}\big(1 - e^{-\delta_{i+1}}\sqrt{1-\eta(1 - e^{-2\delta_{i+1}})}\big)\mu_{t_{i}}^{\mathsf{c}}(\widehat{X}_{t_i})\notag\\
&\qquad + \sigma_{t_{i+1}}\sqrt{\eta(1 - e^{-2\delta_{i+1}})} Z_{t_i}^{\mathsf{c}}.
\end{align}
For the first term, we have
\begin{align}\label{eq:proof-update-1}
\frac{\sigma_{t_{i+1}}}{\sigma_{t_{i}}}\sqrt{1-\eta(1 - e^{-2\delta_{i+1}})}\widehat{X}_{t_{i}}^{\mathsf{c}} 
&= 
\sqrt{1-\eta(1 - e^{-2\delta_{i+1}})}\widehat{X}_{t_{i}}^{\mathsf{c}} + \left(\frac{\sigma_{t_{i+1}}}{\sigma_{t_{i}}}-1\right)\sqrt{1-\eta(1 - e^{-2\delta_{i+1}})}\widehat{X}_{t_{i}}^{\mathsf{c}}\notag\\
& \eqqcolon \sqrt{1-\eta(1 - e^{-2\delta_{i+1}})}\widehat{X}_{t_{i}}^{\mathsf{c}} + e_{t_i,1}(\widehat{X}_{t_{i}}^{\mathsf{c}} ).
\end{align}
For the second term, we have
\begin{align}\label{eq:proof-update-2}
&\alpha_{t_{i+1}}\big(1 - e^{-\delta_{i+1}}\sqrt{1-\eta(1 - e^{-2\delta_{i+1}})}\big)\mu_{t_{i}}^{\mathsf{c}}(\widehat{X}_{t_i}) \notag\\
&\quad = \alpha_{t_{i}}\big(1 - \sqrt{1-\eta(1 - e^{-2\delta_{i+1}})}\big)\mu_{t_{i}}^{\mathsf{c}}(\widehat{X}_{t_i})\notag\\
&\qquad +  \underbrace{\big(\left(\alpha_{t_{i+1}}-\alpha_{t_i}\right) + \left(\alpha_{t_{i+1}}e^{-\delta_{i+1}}-\alpha_{t_i}\right)\sqrt{1-\eta(1 - e^{-2\delta_{i+1}})}\big)\mu_{t_{i}}^{\mathsf{c}}(\widehat{X}_{t_i})}_{e_{t_i,2}(\widehat{X}_{t_{i}} )}\notag\\
&\quad = \big(1 - e^{-\delta_{i+1}}\sqrt{1-\eta(1 - e^{-2\delta_{i+1}})}\big)\left(\widehat{X}_{t_{i}}^{\mathsf{c}} - \sigma_{t_{i}}\epsilon_{t_{i}}^{\mathsf{c}}(\widehat{X}_{t_{i}} )\right) + e_{t_i,2}(\widehat{X}_{t_{i}})\notag\\
&\quad =\big(1 - \sqrt{1-\eta(1 - e^{-2\delta_{i+1}})}\big)\left(\widehat{X}_{t_{i}}^{\mathsf{c}} - \sigma_{t_{i+1}}\epsilon_{t_{i+1}}^{\mathsf{c}}(\widehat{X}_{t_{i+1}}^{\mathsf{est}} )\right) + e_{t_i,2}(\widehat{X}_{t_{i}} ) \notag\\
&\qquad + \underbrace{\big(1 - \sqrt{1-\eta(1 - e^{-2\delta_{i+1}})}\big)\big(\sigma_{t_{i}}\epsilon_{t_{i}}^{\mathsf{c}}(\widehat{X}_{t_{i}} )-\sigma_{t_{i+1}}\epsilon_{t_{i+1}}^{\mathsf{c}}(\widehat{X}_{t_{i+1}}^{\mathsf{est}} )\big)}_{e_{t_i,3}(\widehat{X}_{t_{i}},\widehat{X}_{t_{i+1}}^{\mathsf{s}} )},
\end{align}
where $\widehat{X}_{t_{i+1}}^{\mathsf{est}}$ is defined by $V_{\mathcal{S}^{\mathsf{c}}}^{\top}\widehat{X}_{t_{i+1}}^{\mathsf{est}} = \widehat{X}_{t_{i}}^{\mathsf{c}}$ and $V_{\mathcal{S}}^{\top}\widehat{X}_{t_{i+1}}^{\mathsf{est}} = \widehat{X}_{t_{i+1}}^{\mathsf{s}}$.
Note that
\begin{align}\label{eq:def-epsc}
\epsilon_{t_{i}}^{\mathsf{c}}(x)
&\coloneqq \frac{1}{\sigma_{t_i}}\mathbb{E}[x^{\mathsf{c}}-\alpha_{t_i}X_{0}^{\mathsf{c}}\mymid X_{t_i}=x] = -\sigma_{t_i}\nabla_{x^{\mathsf{c}}} \log p_{X_{t_i}^{\mathsf{c}}|X_{t_i}^{\mathsf{s}}}(x^{\mathsf{c}}|x^{\mathsf{s}}),
\end{align}
where $X_{t_i}\overset{\text{d}}{=}\alpha_{t_i} X_0 + \sigma_{t_i}\mathcal{N}(0,I_d)$, and the second equation holds because
\begin{align}\label{eq:nablap-exp}
    \nabla_{x^{\mathsf{c}}} \log p_{X_{t_i}^{\mathsf{c}}|X_{t_i}^{\mathsf{s}}}(x^{\mathsf{c}}|x^{\mathsf{s}}) = \nabla_{x^{\mathsf{c}}} \log p_{X_{t_i}}(x) = -\frac{1}{\sigma_{t_i}^2}\mathbb{E}[x^{\mathsf{c}}-\alpha_{t_i}X_{0}^{\mathsf{c}}|X_{t_i}=x].
\end{align}
Combining \eqref{eq:proof-update-1} and \eqref{eq:proof-update-2} together with \eqref{eq:proof-update-0}, we have
\begin{align}\label{eq:proof-update-3}
    \widehat{X}_{t_{i+1}}^{\mathsf{c}} &= \widehat{X}_{t_{i}}^{\mathsf{c}} - \sigma_{t_{i+1}}\big(1 - \sqrt{1-\eta(1 - e^{-2\delta_{i+1}})}\big)\epsilon_{t_{i+1}}^{\mathsf{c}}(\widehat{X}_{t_{i+1}}^{\mathsf{est}} ) + \sigma_{t_{i+1}}\sqrt{\eta(1 - e^{-2\delta_{i+1}})} Z_{t_i}^{\mathsf{c}}\notag\\
    &\quad  + e_{t_i,1}(\widehat{X}_{t_{i}}^{\mathsf{c}} )+ e_{t_i,2}(\widehat{X}_{t_{i}} )+ e_{t_i,3}(\widehat{X}_{t_{i}},\widehat{X}_{t_{i+1}}^{\mathsf{s}})\notag\\
    &\eqqcolon \widehat{X}_{t_{i}}^{\mathsf{c}} - \frac{\eta(1 - e^{-2\delta_{i+1}})}{2}\sigma_{t_{i+1}}\epsilon_{t_{i+1}}^{\mathsf{c}}(\widehat{X}_{t_{i+1}}^{\mathsf{est}})+ \sigma_{t_{i+1}}\sqrt{\eta(1 - e^{-2\delta_{i+1}})} Z_{t_i}^{\mathsf{c}} + e_{t_i}(\widehat{X}_{t_{i}},\widehat{X}_{t_{i+1}}^{\mathsf{s}})
\end{align}
where $ e_{t_i}(\widehat{X}_{t_{i}},\widehat{X}_{t_{i+1}}^{\mathsf{s}} )$ is defined as
\begin{align}\label{eq:def-et}
e_{t_i}(\widehat{X}_{t_{i}},\widehat{X}_{t_{i+1}}^{\mathsf{s}} )
&\coloneqq e_{t_i,1}(\widehat{X}_{t_{i}}^{\mathsf{c}} ) + e_{t_i,2}(\widehat{X}_{t_{i}} )+ e_{t_i,3}(\widehat{X}_{t_{i}},\widehat{X}_{t_{i+1}}^{\mathsf{s}} )\notag\\
&\qquad  + \Bigg(1 - \sqrt{1-\eta(1 - e^{-2\delta_{i+1}})} - \frac{\eta(1 - e^{-2\delta_{i+1}})}{2}\Bigg)\sigma_{t_{i+1}}\epsilon_{t_{i+1}}^{\mathsf{c}}(\widehat{X}_{t_{i+1}}^{\mathsf{est}}).
\end{align}
We have the following lemma, whose proof is postponed to Appendix \ref{app:proof-lem-boundet}.
\begin{lemma}\label{lem:proof-boundet}
In \eqref{eq:proof-update-3}, we have
\begin{align*}
\mathbb{E}_{\widehat{X}_{t_{i+1}}^{\mathsf{sc}}|\widehat{X}_{t_{i}},Y,Z}[\|e_{t_i}(\widehat{X}_{t_{i}},\widehat{X}_{t_{i+1}}^{\mathsf{sc}} )\|_2^2] \lesssim \bigg(\|\widehat{X}_{t_i}^{\mathsf{c}}\|_2 +\|\widehat{X}_{t_i}^{\mathsf{sc}}\|_2 + \|Z\|_2+ 1\bigg)^2(\delta_{{i+1}}^2+\gamma_i^4)+\delta_{{i+1}}\gamma_i^2.
\end{align*}
\end{lemma}

Next, we rewrite the update rule \eqref{eq:ddim-prior-dominated} for $\widehat{X}_{t_i}^{\mathsf{sc}}$.
When $\eta_0=1$, it is equivalent to
\begin{align}\label{eq:proof-alg-xsc}
\widehat{X}_{t_{i+1}}^{\mathsf{sc}} &= \frac{\sigma_{t_{i+1}}}{\sigma_{t_{i}}} e^{-\delta_{i+1}} \widehat{X}_{t_{i}}^{\mathsf{sc}}+ \sigma_{t_{i+1}}e^{\lambda_{t_{i+1}}}\big(1 - e^{-2\delta_{i+1}}\big)\mu_{t_{i}}^{\mathsf{sc}}(\widehat{X}_{t_i}) + \sigma_{t_{i+1}}\sqrt{1 - e^{-2\delta_{i+1}}} Z_{t_i}^{\mathsf{sc}}\notag\\
&=\frac{\sigma_{t_{i+1}}}{\sigma_{t_{i}}} e^{-\delta_{i+1}} \widehat{X}_{t_{i}}^{\mathsf{sc}}+ \frac{\alpha_{t_{i+1}}}{\alpha_{t_i}}\big(1 - e^{-2\delta_{i+1}}\big)(\widehat{X}_{t_{i}}^{\mathsf{sc}}-\sigma_{t_i}\epsilon_{t_{i}}^{\mathsf{sc}}) + \sigma_{t_{i+1}}\sqrt{1 - e^{-2\delta_{i+1}}} Z_{t_i}^{\mathsf{sc}}\notag\\
&=\left(\frac{\sigma_{t_{i+1}}}{\sigma_{t_{i}}} e^{-\delta_{i+1}} + \frac{\alpha_{t_{i+1}}}{\alpha_{t_i}} - \frac{\alpha_{t_{i+1}}}{\alpha_{t_i}}e^{-2\delta_{i+1}}\right)\widehat{X}_{t_{i}}^{\mathsf{sc}} - \frac{\alpha_{t_{i+1}}}{\alpha_{t_i}}\big(1 - e^{-2\delta_{i+1}}\big)\sigma_{t_i}\epsilon_{t_{i}}^{\mathsf{sc}} + \sigma_{t_{i+1}}\sqrt{1 - e^{-2\delta_{i+1}}} Z_{t_i}^{\mathsf{sc}}\notag\\
&=\frac{\alpha_{t_{i+1}}}{\alpha_{t_i}}\widehat{X}_{t_{i}}^{\mathsf{sc}} - \alpha_{t_{i+1}}\left({\rm e}^{-\lambda_{t_i}}-{\rm e}^{-2\lambda_{t_{i+1}}+\lambda_{t_{i}}}\right)\epsilon_{t_{i}}^{\mathsf{sc}} + \sigma_{t_{i+1}}\sqrt{1 - e^{-2\delta_{i+1}}} Z_{t_i}^{\mathsf{sc}}\notag\\
&=\frac{\alpha_{t_{i+1}}}{\alpha_{t_i}}\widehat{X}_{t_{i}}^{\mathsf{sc}} - \frac{\alpha_{t_{i+1}}\sigma_{t_i}}{\alpha_{t_i}}\left(1-{\rm e}^{-2\delta_{i+1}}\right)\epsilon_{t_{i}}^{\mathsf{sc}} + \sigma_{t_{i+1}}\sqrt{1 - e^{-2\delta_{i+1}}} Z_{t_i}^{\mathsf{sc}}.
\end{align}


\subsection{Proof of Lemma \ref{lem:convergence-auxiliary-sequence}}
\label{subsec:proof-lem-convergence-auxiliary-sequence}

Combining \eqref{eq:DDIM-St}, \eqref{eq:proof-update-3} and \eqref{eq:proof-alg-xsc}, the update rule of $\widehat{X}_{t_k}$ can be rewritten as 
\begin{align}
\widehat{X}_{t_{k+1}}^{\mathsf{st}} &=\alpha_{t_{k+1}}\overline{Y}_{t_{k+1}} + \widehat{\sigma}_{t_{k+1}}\overline{Z}_{t_{k+1}},\notag\\
\widehat{X}_{t_{k+1}}^{\mathsf{sc}}&=\frac{\alpha_{t_{k+1}}}{\alpha_{t_k}}\widehat{X}_{t_{k}}^{\mathsf{sc}} + \alpha_{t_{k+1}}\left({\rm e}^{-\lambda_{t_k}}-{\rm e}^{-2\lambda_{t_{k+1}}+\lambda_{t_{k}}}\right)\frac{\nabla_{\widehat{X}_{t_k}^{\mathsf{sc}}} \log p_{\lambda_k}(\widehat{X}_{t_k}^{\mathsf{c}},\widehat{X}_{t_k}^{\mathsf{sc}})}{\sigma_{t_k}} + \sigma_{t_{k+1}}\sqrt{1 - e^{-2\delta_{k+1}}} Z_{t_k}^{\mathsf{sc}},\notag\\
\widehat{X}_{t_{k+1}}^{\mathsf{c}} &=
    \widehat{X}_{t_{k}}^{\mathsf{c}} +\gamma_k\nabla \log p_{\lambda_{k+1}}(\widehat{X}_{t_k}^{\mathsf{c}},\widehat{X}_{t_{k+1}}^{\mathsf{sc}}) + \sigma_{t_{k+1}}\sqrt{\eta(1 - e^{-2\delta_{k+1}})} Z_{t_k}^{\mathsf{c}} + e_{t_k}(\widehat{X}_{t_{k}},\widehat{X}_{t_{k+1}}^{\mathsf{sc}}).
\end{align}

The KL divergence between $\widehat{X}_{t_{0:M}}$ and $\widetilde{X}_{0:M}$ conditoned on $Y$ and $Z$ can be decomposed as
\begin{align*}
\mathsf{KL}(p_{\widetilde{X}_{0:M}\mymid Y,Z}\Vert p_{\widehat{X}_{t_{0:M}}\mymid Y,Z})
&= \mathsf{KL}(p_{\widetilde{X}_{0}\mymid Y,Z}\Vert p_{\widehat{X}_{t_{0}}\mymid Y,Z}) + 
\sum_{i=0}^{M-1} \mathbb{E}_{x\sim p_{\widetilde{X}_{t_{i}}\mymid Y,Z}} \mathsf{KL}\big(p_{\widetilde{X}_{i+1}^{\mathsf{sc}}|\widetilde{X}_{i}}(\cdot|x) \Vert p_{\widehat{X}_{t_{i+1}}^{\mathsf{sc}}|\widehat{X}_{t_i}}(\cdot|x)\big)\notag\\
&\quad +\sum_{i=0}^{M-1} \mathbb{E}_{x,x'\sim p_{\widetilde{X}_{i+1}^{\mathsf{sc}},\widetilde{X}_i\mymid Y,Z}} \mathsf{KL}\big(p_{\widetilde{X}_{i+1}^{\mathsf{c}}|\widetilde{X}_{i+1}^{\mathsf{sc}},\widetilde{X}_{i}}(\cdot|x,x') \Vert p_{\widehat{X}_{t_{i+1}}^{\mathsf{c}}|\widehat{X}_{t_{i+1}}^{\mathsf{sc}},\widehat{X}_{t_i}}(\cdot|x,x')\big)\notag\\
&=\sum_{i=0}^{M-1} \mathbb{E}_{x,x'\sim p_{\widetilde{X}_{i+1}^{\mathsf{sc}},\widetilde{X}_i\mymid Y,Z}} \mathsf{KL}\big(p_{\widetilde{X}_{i+1}^{\mathsf{c}}|\widetilde{X}_{i+1}^{\mathsf{sc}},\widetilde{X}_{i}}(\cdot|x,x') \Vert p_{\widehat{X}_{t_{i+1}}^{\mathsf{c}}|\widehat{X}_{t_{i+1}}^{\mathsf{sc}},\widehat{X}_{t_i}}(\cdot|x,x')\big).
\end{align*}
Note that the both of the above conditional distributions are Gaussian ditributions with variance $2\gamma_i$ and mean vector $(x')^{\mathsf{c}}+\gamma_i\nabla \log p_{\lambda_{i+1}}((x')^{\mathsf{c}},x^{\mathsf{sc}})+e_{t_i}(x',x^{\mathsf{sc}})$ and $(x')^{\mathsf{c}}+\gamma_i\nabla \log p_{\lambda_{i+1}}((x')^{\mathsf{c}},x^{\mathsf{sc}})$, respectively.
Combining with Lemma \ref{lem:proof-boundet}, we finally get
\begin{align*}
&\mathbb{E}_{x,x'\sim p_{\widetilde{X}_{i+1}^{\mathsf{sc}},\widetilde{X}_i\mymid Y,Z}} \big[\mathsf{KL}\big(p_{\widetilde{X}_{i+1}^{\mathsf{c}}|\widetilde{X}_{i+1}^{\mathsf{sc}},\widetilde{X}_{i}}(\cdot|x,x') \Vert p_{\widehat{X}_{t_{i+1}}^{\mathsf{c}}|\widehat{X}_{t_{i+1}}^{\mathsf{sc}},\widehat{X}_{t_i}}(\cdot|x,x')\big)\big]\notag\\
&\quad = \frac{\mathbb{E}_{x,x'\sim p_{\widetilde{X}_{i+1}^{\mathsf{sc}},\widetilde{X}_i}\mymid Y,Z} \big[\|e_{t_i}(x', x)\|_2^2\big]}{4\gamma_i}\notag\\
&\quad \lesssim  \frac{1}{\gamma_i}\left(\mathbb{E}\bigg[\bigg(\|\widetilde{X}_{i}^{\mathsf{c}}\|_2 +\|\widetilde{X}_{i}^{\mathsf{sc}}\|_2 + \|Z\|_2+ 1\bigg)^2\bigg]\right)(\delta_{{i+1}}^2+\gamma_i^4) + \delta_{{i+1}}\gamma_i.
\end{align*}
Summarizing from $i=0$ to $M-1$ yields
\begin{align*}
&\sum_{i=0}^{M-1}\mathbb{E}_{x,x'\sim p_{\widetilde{X}_{i+1}^{\mathsf{sc}},\widetilde{X}_i\mymid Y,Z}} \big[\mathsf{KL}\big(p_{\widetilde{X}_{i+1}^{\mathsf{c}}|\widetilde{X}_{i+1}^{\mathsf{sc}},\widetilde{X}_{i}}(\cdot|x,x') \Vert p_{\widehat{X}_{t_{i+1}}^{\mathsf{c}}|\widehat{X}_{t_{i+1}}^{\mathsf{sc}},\widehat{X}_{t_i}}(\cdot|x,x')\big)\big]\notag\\
& \quad\overset{\text{(i)}}{\lesssim}  \max\frac{\delta_{{i+1}}}{\gamma_i} + \eta \max\gamma_i^2 + \max\gamma_i \overset{\text{(ii)}}{\lesssim}  \frac{1}{\eta} + \eta \max\gamma_i^2 + \max\gamma_i\notag\\
&\quad =O\big(\eta^{-1}\big) + O(\eta^3\delta^2) + O(\eta\delta),
\end{align*}
where (i) uses $\sum_{i=0}^{M-1}\delta_{{i+1}} = O(1)$, $\sum_{i=0}^{M-1}\gamma_{i} = O(\eta)$, and $\mathbb{E}[\|\widetilde{X}_{i}^{\mathsf{c}}\|_2^2+\|\widetilde{X}_{i}^{\mathsf{sc}}\|_2^2+\|Z\|_2^2]=O(1)$, and (ii) uses \eqref{eq:proof-bound-gamma}.
As $\eta^3\delta^2\to 0$ and $\eta\to\infty$, which is satisfied by $\eta=\Theta(\delta^{-7/12})$, we have $\mathsf{KL}\big(p_{\widetilde{X}_{0:M}\mymid Y,Z}\Vert p_{\widehat{X}_{t_{0:M}}\mymid Y,Z}\big)\to 0$ and complete the proof.

Next, we intend to prove that $p_{\overline{X}_{k} \mymid V_{\Sti{t_k}}^{\top}\overline{X}_{k}, v_{s}^{\top}\overline{X}_{k_{\tau_s}}} = p_{X_{t_k} \mymid V_{\Sti{t_k}}^{\top}X_{t_k}, v_{s}^{\top}X_{\tau_s}}$.
When $\sigma_{t_k}>\alpha_{t_k} \Sigma_{ss}^{-1}\sigma$, we have $\Sti{t_k}=\mathcal{S}$ and $\Stci{t_k}$ is empty.
Thus we have
\begin{align*}
p_{\overline{X}_{k} \mymid V_{\Sti{t_k}}^{\top}\overline{X}_{k}, v_{s}^{\top}\overline{X}_{k_{\tau_s}}} =p_{\overline{X}_{k} \mymid V_{\Sti{t_k}}^{\top}\overline{X}_{k}} = p_{X_{t_k} \mymid V_{\Sti{t_k}}^{\top}X_{t_k}}=p_{{X}_{t_k} \mymid V_{\Sti{t_k}}^{\top}{X}_{t_k}, v_{s}^{\top}{X}_{{\tau_s}}} ,
\end{align*}
which makes Lemma \ref{lem:convergence-auxiliary-sequence} hold obviously.
In particular, this leads to the following equation when $t_i=\tau_s$:
\begin{align*}
p_{\overline{X}_{k_{\tau_s}} \mymid V_{\mathcal{S}}^{\top}\overline{X}_{k_{\tau_s}}}= p_{X_{\tau_s} \mymid V_{\mathcal{S}}^{\top}X_{\tau_s}}.
\end{align*}
When $\sigma_{t_k}\le\alpha_{t_k} \Sigma_{ss}^{-1}\sigma$, we have $\Stci{t_k}=\mathcal{S}$ and $\Sti{t_k}$ is empty.
Thus we have
\begin{align*}
p_{\overline{X}_{k} \mymid \overline{X}_{{k-1}}} =p_{\overline{X}_{k}^{\mathsf{c}} \mymid \overline{X}_{k}^{\mathsf{s}}}p_{\overline{X}_{k}^{\mathsf{s}} \mymid \overline{X}_{{k-1}}} = p_{{X}_{t_k}^{\mathsf{c}} \mymid {X}_{t_k}^{\mathsf{s}}}p_{{X}_{t_k}^{\mathsf{s}} \mymid {X}_{t_{k-1}}} =p_{{X}_{t_k} \mymid {X}_{t_{k-1}}}.
\end{align*}
Therefore we have
\begin{align*}
p_{\overline{X}_{k} \mymid V_{\Sti{t_k}}^{\top}\overline{X}_{k}, v_{s}^{\top}\overline{X}_{k_{\tau_s}}}(x_k\mymid \cdot, x_{\tau_s}^{\mathsf{s}})  
&= p_{\overline{X}_{k} \mymid V_{\mathcal{S}}^{\top}\overline{X}_{k_{\tau_s}}}(x_k\mymid x_{\tau_s}^{\mathsf{s}}) \notag\\
&= \int \prod_{j=k_{\tau_s}+1}^{k} p_{\overline{X}_{j} \mymid \overline{X}_{{j-1}}}(x_j\mymid x_{j-1}) p_{\overline{X}_{k_{\tau_s}} \mymid V_{\mathcal{S}}^{\top}\overline{X}_{k_{\tau_s}}}(x_{\tau_s}\mymid x_{\tau_s}^{\mathsf{s}}) \prod_{j=k_{\tau_s}+1}^{k-1}\mathrm{d} x_j\mathrm{d} x_{\tau_s}^{\mathsf{c}}\notag\\
&= \int \prod_{j=k_{\tau_s}+1}^{k} p_{{X}_{t_j} \mymid {X}_{t_{j-1}}}(x_j\mymid x_{j-1}) p_{{X}_{\tau_s} \mymid V_{\mathcal{S}}^{\top}{X}_{\tau_s}}(x_{\tau_s}\mymid x_{\tau_s}^{\mathsf{s}}) \prod_{j=k_{\tau_s}+1}^{k-1}\mathrm{d} x_j\mathrm{d} x_{\tau_s}^{\mathsf{c}}\notag\\
&=p_{{X}_{t_k} \mymid V_{\mathcal{S}}^{\top}{X}_{\tau_s}}(x_k\mymid x_{\tau_s}^{\mathsf{s}})=p_{{X}_{t_k} \mymid V_{\Sti{t_k}}^{\top}{X}_{t_k}, v_{s}^{\top}{X}_{\tau_s}}(x_k\mymid \cdot, x_{\tau_s}^{\mathsf{s}}).
\end{align*}

\subsection{Proof of Lemma \ref{lem:distri-xc}}
\label{subsec:proof-lem-distri-xc}

According to Lemma \ref{lem:convergence-auxiliary-sequence}, it sufficies to prove that
\begin{align}\label{eq:proof-lem-distri-sc-5}
\mathbb{E}_{V_{\mathcal{S}}^{\top}\widetilde{X}_k\mymid Y,Z}\big[\mathsf{TV}(p_{V_{\mathcal{S}^{\mathsf{c}}}^{\top}\widetilde{X}_k|V_{\mathcal{S}}^{\top}\widetilde{X}_k,Y,Z},\Vert p_{V_{\mathcal{S}^{\mathsf{c}}}^{\top}{X}_{t_k}|V_{\mathcal{S}}^{\top}{X}_{t_k}} )\big] \to 0\qquad \mathsf{as}\qquad \delta\eta\log\frac{1}{\delta} \to 0\quad \mathrm{and}\quad \frac{\delta\eta^2}{\log\frac{1}{\delta}}\to\infty.
\end{align}
Then by using Lemma \ref{lem:convergence-auxiliary-sequence}, we have
\begin{align}
&\mathbb{E}_{x_k^{\mathsf{s}}\sim \widehat{X}_{t_k}^{\mathsf{s}}}\big[\mathsf{TV}(p_{V_{\mathcal{S}^{\mathsf{c}}}^{\top}\widehat{X}_{t_k}|V_{\mathcal{S}}^{\top}\widehat{X}_{t_k}}(\cdot\mymid x_k^{\mathsf{s}}), p_{V_{\mathcal{S}^{\mathsf{c}}}^{\top}X_{t_k}|V_{\mathcal{S}}^{\top}X_{t_k}}(\cdot\mymid x_k^{\mathsf{s}}) ) \big]\notag\\
&\quad= \mathbb{E}_{x_k^{\mathsf{s}}\sim \widehat{X}_{t_k}^{\mathsf{s}}}\big[\mathsf{TV}(p_{\widehat{X}_{t_k}^{\mathsf{c}}|\widehat{X}_{t_k}^{\mathsf{s}}}(\cdot\mymid x_k^{\mathsf{s}}), p_{X_{t_k}^{\mathsf{c}}|X_{t_k}^{\mathsf{s}}}(\cdot\mymid x_k^{\mathsf{s}}) ) \big]\notag\\
&\quad= \mathbb{E}_{Y,Z}\Big[\mathbb{E}_{x_k^{\mathsf{s}}\sim \widehat{X}_{t_k}^{\mathsf{s}}\mymid Y,Z}\big[\mathsf{TV}(p_{\widehat{X}_{t_k}^{\mathsf{c}}|\widehat{X}_{t_k}^{\mathsf{s}},Y,Z}(\cdot\mymid x_k^{\mathsf{s}},y,z), p_{X_{t_k}^{\mathsf{c}}|X_{t_k}^{\mathsf{s}}}(\cdot\mymid x_k^{\mathsf{s}}) ) \big]\Big]\notag\\
&\quad\le\mathbb{E}_{Y,Z}\Big[\mathbb{E}_{x_k^{\mathsf{s}}\sim \widehat{X}_{t_k}^{\mathsf{s}}\mymid Y,Z}\big[\mathsf{TV}(p_{\widehat{X}_{t_k}^{\mathsf{c}}|\widehat{X}_{t_k}^{\mathsf{s}},Y,Z}(\cdot\mymid x_k^{\mathsf{s}},y,z), p_{\widetilde{X}_{k}^{\mathsf{c}}|\widetilde{X}_{k}^{\mathsf{s}},Y,Z}(\cdot\mymid x_k^{\mathsf{s}},y,z) ) \big]\Big]\notag\\
&\qquad  + \mathbb{E}_{Y,Z}\Big[\mathbb{E}_{x_k^{\mathsf{s}}\sim \widehat{X}_{t_k}^{\mathsf{s}}\mymid Y,Z}\big[\mathsf{TV}(p_{\widetilde{X}_{k}^{\mathsf{c}}|\widetilde{X}_{k}^{\mathsf{s}},Y,Z}(\cdot\mymid x_k^{\mathsf{s}},y,z), p_{X_{t_k}^{\mathsf{c}}|X_{t_k}^{\mathsf{s}}}(\cdot\mymid x_k^{\mathsf{s}}) ) \big]\Big]\notag\\
&\quad\le \mathbb{E}_{Y,Z}\Big[\mathsf{TV}(p_{\widehat{X}_{t_k}\mymid Y,Z}, p_{\widetilde{X}_{k}\mymid Y,Z} ) + \mathsf{TV}(p_{\widehat{X}_{t_k}^{\mathsf{s}}\mymid Y,Z}, p_{\widetilde{X}_{k}^{\mathsf{s}}\mymid Y,Z} )\Big] \notag\\
&\qquad + \mathbb{E}_{Y,Z}\Big[\mathbb{E}_{x_k^{\mathsf{s}}\sim \widehat{X}_{t_k}^{\mathsf{s}}\mymid Y,Z}\big[\mathsf{TV}(p_{\widetilde{X}_{k}^{\mathsf{c}}|\widetilde{X}_{k}^{\mathsf{s}},Y,Z}(\cdot\mymid x_k^{\mathsf{s}},y,z), p_{X_{t_k}^{\mathsf{c}}|X_{t_k}^{\mathsf{s}}}(\cdot\mymid x_k^{\mathsf{s}}) )\big]\Big]\notag\\
&\quad\le \mathbb{E}_{Y,Z}\Big[\mathsf{TV}(p_{\widehat{X}_{t_k}\mymid Y,Z}, p_{\widetilde{X}_{k}\mymid Y,Z} ) + 2\mathsf{TV}(p_{\widehat{X}_{t_k}^{\mathsf{s}}\mymid Y,Z}, p_{\widetilde{X}_{k}^{\mathsf{s}}\mymid Y,Z} )\Big] \notag\\
&\qquad + \mathbb{E}_{Y,Z}\Big[\mathbb{E}_{x_k^{\mathsf{s}}\sim \widetilde{X}_{k}^{\mathsf{s}}\mymid Y,Z}\big[\mathsf{TV}(p_{\widetilde{X}_{k}^{\mathsf{c}}|\widetilde{X}_{k}^{\mathsf{s}},Y,Z}(\cdot\mymid x_k^{\mathsf{s}},y,z), p_{X_{t_k}^{\mathsf{c}}|X_{t_k}^{\mathsf{s}}}(\cdot\mymid x_k^{\mathsf{s}}) )\big]\Big]\to 0.
\end{align}

We introduce an auxiliary sequence as below: for $k-N\le i < k$, with initilization $\breve{X}_{k-N}\sim p_{\widetilde{X}_{k-N}}$, and iterates as
\begin{subequations}\label{eq:def-breveX}
    \begin{align}
    \breve{X}_{i+1}^{\mathsf{st}} &=\widetilde{X}_{i+1}^{\mathsf{st}},\\
\breve{X}_{i+1}^{\mathsf{sc}} &= \frac{\alpha_{t_{i+1}}}{\alpha_{t_i}}\breve{X}_{{i}}^{\mathsf{sc}} + \frac{\alpha_{t_{i+1}}\sigma_{t_k}^2}{\alpha_{t_{i}}}\left(1-{\rm e}^{-2\delta_{i+1}}\right)\nabla_{\breve{X}_{i}^{\mathsf{sc}}} \log p_{\lambda_i}(\breve{X}_{i-N}^{\mathsf{c}},\breve{X}_{i}^{\mathsf{sc}}) + \sigma_{t_{i+1}}\sqrt{1 - e^{-2\delta_{i+1}}} Z_{i}^{\mathsf{sc}},\label{eq:def-breveX-sc}\\
\breve{X}_{i+1}^{\mathsf{c}} &= \breve{X}_{i}^{\mathsf{c}} +  \gamma_i \nabla_{\breve{X}_i^{\mathsf{c}}} \log p_{\lambda_{i+1}}(\breve{X}_{i}^{\mathsf{c}},\breve{X}_{i+1}^{\mathsf{sc}}) + \sqrt{2\gamma_i} Z_i^{\mathsf{c}}.\label{eq:def-breveX-c}
\end{align}
\end{subequations}
We choose $N$ satisfying $\sum_{i=k-N}^k \delta_{i+1} = \Theta(\delta^{\frac12})$. 
For any $y,z$, we have
\begin{align}\label{eq:proof-breveX-tildeX}
\mathsf{TV}(p_{\breve{X}_k\mymid Y,Z},p_{\widetilde{X}_k\mymid Y,Z}) \to 0.
\end{align}
The proof is postponed to the end of this subsection.
Armed with this auxiliary sequence, we prove \eqref{eq:proof-lem-distri-sc-5} via establishing
\begin{align}\label{eq:proof-lem-distri-sc-6}
\mathbb{E}_{\widetilde{X}_k^{\mathsf{s}}\mymid Y,Z}\big[\mathsf{TV}(p_{\breve{X}_k^{\mathsf{c}}|\breve{X}_k^{\mathsf{s}},Y,Z},\Vert p_{{X}_{t_k}^{\mathsf{c}}|{X}_{t_k}^{\mathsf{s}}} )\big] & \to 0\qquad \mathsf{as}\quad \delta\eta\log\frac{1}{\delta} \to 0\quad \mathrm{and}\quad \frac{\delta\eta^2}{\log\frac{1}{\delta}}\to\infty.
\end{align}
which implies that 
\begin{align*}
    &\mathbb{E}_{x_k^{\mathsf{sc}}\sim \widetilde{X}_k^{\mathsf{sc}}\mymid Y,Z}[\mathsf{TV}(p_{\widetilde{X}_k^{\mathsf{c}}\mymid \widetilde{X}_k^{\mathsf{sc}},Y,Z}(\cdot\mymid x_k^{\mathsf{sc}}, y, z),p_{{X}_{t_k}^{\mathsf{c}}\mymid {X}_{t_k}^{\mathsf{sc}}}(\cdot\mymid x_k^{\mathsf{sc}}))]\notag\\
&\quad    \le \mathbb{E}_{x_k^{\mathsf{sc}}\sim \widetilde{X}_k^{\mathsf{sc}}\mymid Y,Z}[\mathsf{TV}(p_{\breve{X}_k^{\mathsf{c}}\mymid \breve{X}_k^{\mathsf{sc}},Y,Z}(\cdot\mymid x_k^{\mathsf{sc}}, y, z),p_{{X}_{t_k}^{\mathsf{c}}\mymid {X}_{t_k}^{\mathsf{sc}}}(\cdot\mymid x_k^{\mathsf{sc}}))]\notag\\
&\qquad    + \mathbb{E}_{x_k^{\mathsf{sc}}\sim \widetilde{X}_k^{\mathsf{sc}}\mymid Y,Z}[\mathsf{TV}(p_{\breve{X}_k^{\mathsf{c}}\mymid \breve{X}_k^{\mathsf{sc}},Y,Z}(\cdot\mymid x_k^{\mathsf{sc}}, y, z),p_{\widetilde{X}_k^{\mathsf{c}}\mymid \widetilde{X}_k^{\mathsf{sc}},Y,Z}(\cdot\mymid x_k^{\mathsf{sc}}, y, z))] \notag\\
&\quad    \le \mathbb{E}_{x_k^{\mathsf{sc}}\sim \widetilde{X}_k^{\mathsf{sc}}\mymid Y,Z}[\mathsf{TV}(p_{\breve{X}_k^{\mathsf{c}}\mymid \breve{X}_k^{\mathsf{sc}},Y,Z}(\cdot\mymid x_k^{\mathsf{sc}}, y, z),p_{{X}_{t_k}^{\mathsf{c}}\mymid {X}_{t_k}^{\mathsf{sc}}}(\cdot\mymid x_k^{\mathsf{sc}}))]\notag\\
&\qquad    +\mathsf{TV}(p_{\breve{X}_k\mymid Y,Z}(\cdot\mymid y, z),p_{\widetilde{X}_k\mymid Y,Z}(\cdot\mymid y, z)) + \mathsf{TV}(p_{\breve{X}_k^{\mathsf{sc}}\mymid Y,Z}(\cdot\mymid y, z),p_{\widetilde{X}_k^{\mathsf{sc}}\mymid Y,Z}(\cdot\mymid y, z))\to 0.
\end{align*}

To establish \eqref{eq:proof-lem-distri-sc-6},
we begin with
considering given $\breve{X}_{1:k}^{\mathsf{sc}}=\breve{x}_{1:k}^{\mathsf{sc}}$, which satisfies
\begin{align}\label{eq:def-set-E}
\mathcal{E}_k\coloneqq\big\{\|\breve{x}_{i+1}^{\mathsf{sc}}-\breve{x}_{i}^{\mathsf{sc}}\|_2\lesssim \delta_{i+1}\log\frac{1}{\delta},\forall i\le k-1\big\}.
\end{align} 
For ease of notation, denote the distribution (pdf) of $\breve{X}_{k}^{\mathsf{c}}$ under given $\breve{X}_{1:k}^{\mathsf{s}}=\breve{x}_{1:k}^{\mathsf{sc}}$ as $\psi_k$.
Denote the target conditional distribution
$$
p_k^{\mathsf{c}}(x)\coloneqq p_{X_{t_k}^{\mathsf{c}}|X_{t_k}^{\mathsf{s}}}(x|\breve{x}_k^{\mathsf{s}}).
$$
By deifnition of $p_{\lambda_k}$, we have $\nabla_x \log p_{\lambda_k}(x,\breve{x}_k^{\mathsf{sc}}) = \nabla_x \log p_k^{\mathsf{c}}(x)$.
Moreover, denote $P_{\gamma_k}$ as the transition probability of the following SDE from $t$ to $t+\gamma_k$:
\begin{align}
\mathrm{d} X(t) = \nabla \log p_k^{\mathsf{c}}(X(t)) \mathrm{d} t + \sqrt{2} \mathrm{d} W_t,
\end{align}
where $W_t$ denotes the Brownian motion.
According to the analysis of Langevin, we have
\begin{align*}
\mathsf{KL}(P_{\gamma_k}\psi_{k}\Vert p_{k+1}^{\mathsf{c}}) \le {\rm e}^{-2\kappa\gamma_k}\mathsf{KL}(\psi_{k}\Vert p_{k+1}^{\mathsf{c}}),
\end{align*}
where $\kappa<\infty$ is a constant.
Therefore, we have
\begin{align}\label{eq:proof-13}
\mathsf{KL}(\psi_{k+1}\Vert p_{k+1}^{\mathsf{c}}) &= \mathsf{KL}(P_{\gamma_k}\psi_{k}\Vert p_{k+1}^{\mathsf{c}}) +\mathsf{KL}(\psi_{k+1}\Vert p_{k+1}^{\mathsf{c}}) - \mathsf{KL}(P_{\gamma_k}\psi_{k}\Vert p_{k+1}^{\mathsf{c}})\notag\\
&\le {\rm e}^{-2\kappa\gamma_k}\mathsf{KL}(\psi_{k}\Vert p_{k+1}^{\mathsf{c}}) +\mathsf{KL}(\psi_{k+1}\Vert p_{k+1}^{\mathsf{c}}) - \mathsf{KL}(P_{\gamma_k}\psi_{k}\Vert p_{k+1}^{\mathsf{c}})\notag\\
&= {\rm e}^{-2\kappa\gamma_k}\mathsf{KL}(\psi_{k}\Vert p_{k}^{\mathsf{c}}) + {\rm e}^{-2\kappa\gamma_k}\left(\mathsf{KL}(\psi_{k}\Vert p_{k+1}^{\mathsf{c}})-\mathsf{KL}(\psi_{k}\Vert p_{k}^{\mathsf{c}}) \right)\notag\\
&\quad +\mathsf{KL}(\psi_{k+1}\Vert p_{k+1}^{\mathsf{c}}) - \mathsf{KL}(P_{\gamma_k}\psi_{k}\Vert p_{k+1}^{\mathsf{c}}).
\end{align}

Next, we control the two error terms in the right-hand-side of \eqref{eq:proof-13} separately.
\begin{itemize}
\item
\emph{Control the term $\mathsf{KL}(\psi_{k}\Vert p_{k+1}^{\mathsf{c}})-\mathsf{KL}(\psi_{k}\Vert p_{k}^{\mathsf{c}})$.}
According to the definition of KL divergence, we have
\begin{align}\label{eq:proof-10}
\mathsf{KL}(\psi_{k}\Vert p_{k+1}^{\mathsf{c}}) - \mathsf{KL}(\psi_{k}\Vert p_{k}^{\mathsf{c}}) &= \int \log\frac{p_{k}^{\mathsf{c}}(x)}{p_{k+1}^{\mathsf{c}}(x)}\psi_{k}(x)\mathrm{d}x.
\end{align}
According to the definition of $p_{k+1}^{\mathsf{c}}(x)$, we have
\begin{align*}
p_{k}^{\mathsf{c}}(x)=\frac{p_{X_{t_k}}(x)}{p_{X_{t_k}^{\mathsf{s}}}(\breve{x}_k^{\mathsf{s}})}
\end{align*}
Then we have
\begin{align*}
\frac{\partial \log p_{k}^{\mathsf{c}}(x)}{\partial \alpha_{t_i}} & = 
\frac{\partial \log p_{X_{t_k}}(x)}{\partial \alpha_{t_k}} - \frac{\partial \log p_{X_{t_k}^{\mathsf{s}}}(\breve{x}_k^{\mathsf{s}})}{\partial \alpha_{t_k}}\notag\\
 &= \frac{1}{\sigma_{t_k}^2}\mathbb{E}[(x-\alpha_{t_k}X_0^{\mathsf{c}})^{\top}X_0^{\mathsf{c}}+(\breve{x}_k-\alpha_{t_k}X_0^{\mathsf{s}})^{\top}X_0^{\mathsf{s}} |X_{t_k}^{\mathsf{c}}=x,X_{t_k}^{\mathsf{s}}=\breve{x}_k^{\mathsf{s}}] \notag\\
&\quad - \frac{1}{\sigma_{t_k}^2}\mathbb{E}[(\breve{x}_k^{\mathsf{s}}-\alpha_{t_k}X_0^{\mathsf{s}})^{\top}X_0^{\mathsf{s}}|X_{t_k}^{\mathsf{s}}=\breve{x}_k^{\mathsf{s}}].
\end{align*}
Recalling that $X_0$ is bounded, we further have
\begin{align*}
\left|\frac{\partial \log p_{k}^{\mathsf{c}}(x)}{\partial \alpha_{t_k}}\right|\lesssim (\|x\|_2+\|\breve{x}_k^{\mathsf{s}}\|_2+1).
\end{align*}
By similar calculation, we have
\begin{align*}
\left| \frac{\partial \log p_{k}^{\mathsf{c}}(x)}{\partial \sigma_{t_i}} \right| \lesssim (\|x\|_2^2+\|\breve{x}_k^{\mathsf{s}}\|_2+ 1),\qquad \left\|\nabla_{x^{\mathsf{s}}} \log p_{k}^{\mathsf{c}}(x)\right\|_2\lesssim (\|x\|_2+\|\breve{x}_k^{\mathsf{s}}\|_2+ 1).
\end{align*}
Thus we have
\begin{align*}
\log \frac{p_{k+1}^{\mathsf{c}}(x)}{p_{k}^{\mathsf{c}}} &= O(|\alpha_{t_{k+1}}-\alpha_{t_{k}}|\notag\\
&\qquad + |\sigma_{t_{k+1}}-\sigma_{t_{k}}|)(\|x\|_2^2+\|\breve{x}_k^{\mathsf{s}}\|_2+1) + O(\|\breve{x}_{k+}^{\mathsf{s}}-\breve{x}_k^{\mathsf{s}}\|_2)(\|x\|_2+\|\breve{x}_k^{\mathsf{s}}\|_2+1)\notag\\
& = (\|x\|_2^2+\|\breve{x}_k^{\mathsf{s}}\|_2+1)O\bigg(\delta_{k+1}+\sqrt{\delta_{k+1}\log\frac{1}{\delta}}\bigg),
\end{align*}
where the last line arises from \eqref{eq:proof-diff-alpha}, \eqref{eq:proof-diff-sigma}, and \eqref{eq:def-set-E}.
Based on this observation, we have
\begin{align*}
\int \log\frac{p_{k}^{\mathsf{c}}(x)}{p_{k+1}^{\mathsf{c}}(x)}\psi_{k}^{\mathsf{c}}(x)\mathrm{d}x &= -\int \log\frac{p_{k+1}^{\mathsf{c}}(x)}{p_{k}^{\mathsf{c}}(x)}\psi_{k}^{\mathsf{c}}(x)\mathrm{d}x \notag\\
&=O(\sqrt{\delta_{k+1}\log\delta^{-1}})\big(\mathbb{E}[\|\breve{X}_k^{\mathsf{c}}\|_2^2\mymid \breve{X}_k^{\mathsf{s}}=\breve{x}_{1:k}^{\mathsf{sc}}]+\|\breve{x}_{k}^{\mathsf{sc}}\|_2+1\big)\notag\\
&=O(\sqrt{\delta_{k+1}\log\delta^{-1}}),
\end{align*}
where the last line uses $\mathbb{E}[\|\breve{X}_k^{\mathsf{c}}\|_2^2\mymid \breve{X}_k^{\mathsf{s}}=\breve{x}_{1:k}^{\mathsf{sc}}]<\infty$.
Combining these with \eqref{eq:proof-10} yields
\begin{align}\label{eq:proof-14}
\left|\mathsf{KL}(\psi_{k}\Vert p_{k+1}^{\mathsf{c}}) - \mathsf{KL}(\psi_{k}\Vert p_{k+1}^{\mathsf{c}})\right| 
\lesssim \sqrt{\delta_{k+1}\log\delta^{-1}}.
\end{align}

\item
\emph{Control the term $\mathsf{KL}(\psi_{k+1}\Vert p_{k+1}^{\mathsf{c}}) - \mathsf{KL}(P_{\gamma_k}\psi_{k}\Vert p_{k+1}^{\mathsf{c}})$.} For any smooth function $f$, we have
\begin{align}\label{eq:proof-11}
\mathbb{E}_{\breve{X}_{k+1}^{\mathsf{c}}}[f(\breve{X}_{k+1}^{\mathsf{c}})] 
&= \mathbb{E}_{\breve{X}_{k}^{\mathsf{c}},Z_k^{\mathsf{c}}}\big[f\big(\breve{X}_{k}^{\mathsf{c}}+\gamma_k\nabla\log p_{k+1}^{\mathsf{c}}(\breve{X}_{k}^{\mathsf{c}})+\sqrt{2\gamma_k}Z_k^{\mathsf{c}}\big)\big] \notag\\
&= \mathbb{E}\big[f\big(\breve{X}_{k}^{\mathsf{c}}\big)+\gamma_k\nabla f(\breve{X}_{k}^{\mathsf{c}})^{\top}\nabla\log p_{k+1}^{\mathsf{c}}(\breve{X}_{k}^{\mathsf{c}}) + \gamma_k\Delta f(\breve{X}_{k}^{\mathsf{c}}) + O(\gamma_k^2)\big],
\end{align}
where $\Delta f(x)\coloneqq \mathsf{Tr}(\nabla_x^2 f(x))$.
In contrast, by using the geneator operator of Langevin proccess $\mathcal{L} = \nabla\log p_{k+1}^{\top} \nabla + \Delta$, we have
\begin{align}\label{eq:proof-12}
\mathbb{E}_{X\sim P_{\gamma_k}\psi_k}[f(X)] &= \mathbb{E}_{X\sim\psi_k}[f(X)+\gamma_k (\mathcal{L}f)(X)+O(\gamma_k^2)]\notag\\
& =\mathbb{E}_{X\sim\psi_k}[f(X)+\gamma_k\nabla\log p_{k+1}^{\top}(X) \nabla f(X) + \gamma_k\Delta f(X)+O(\gamma_k^2)] 
\end{align}
Combining \eqref{eq:proof-11} and \eqref{eq:proof-12}, we have
\begin{align}\label{eq:proof-29}
\mathbb{E}_{X\sim \psi_{k+1}}[f(X)] -\mathbb{E}_{X\sim P_{\gamma_k}\psi_{k}}[f(X)] =O(\gamma_k^2).
\end{align}
Armed with this, we have
\begin{align*}
    &\mathsf{KL}(\psi_{k+1}\Vert p_{k+1}^{\mathsf{c}}) - \mathsf{KL}(P_{\gamma_k}\psi_{k}\Vert p_{k+1}^{\mathsf{c}}) \notag\\
    &\quad = \mathbb{E}_{X\sim \psi_{k+1}}\Big[\log\frac{\psi_{k+1}(X)}{p_{k+1}^{\mathsf{c}}(X)}\Big] - \mathbb{E}_{X\sim P_{\gamma_k}\psi_{k}}\Big[\log\frac{P_{\gamma_k}\psi_{k}(X)}{p_{k+1}^{\mathsf{c}}(X)}\Big]\notag\\
    &\quad =\mathbb{E}_{X\sim \psi_{k+1}}\Big[\log\frac{\psi_{k+1}(X)}{p_{k+1}^{\mathsf{c}}(X)}\Big] - \mathbb{E}_{X\sim P_{\gamma_k}\psi_{k}}\Big[\log\frac{\psi_{k+1}(X)}{p_{k+1}^{\mathsf{c}}(X)}\Big] - \mathbb{E}_{X\sim P_{\gamma_k}\psi_{k}}\Big[\log\frac{P_{\gamma_k}\psi_k(X)}{\psi_{k+1}(X)}\Big].
\end{align*}
Notice that as the KL divergence between $P_{\gamma_k}\psi_{k}$ and $\psi_{k+1}$ is non-negative, i.e.,
\begin{align*}
\mathbb{E}_{X\sim P_{\gamma_k}\psi_{k}}\Big[\log\frac{P_{\gamma_k}\psi_k(X)}{\psi_{k+1}(X)}\Big]\ge 0.
\end{align*}
By using \eqref{eq:proof-29}, we have
\begin{align}\label{eq:proof-15}
    \mathsf{KL}(\psi_{k+1}\Vert p_{k+1}^{\mathsf{c}}) - \mathsf{KL}(P_{\gamma_k}\psi_{k}\Vert p_{k+1}^{\mathsf{c}}) &\le \mathbb{E}_{X\sim \psi_{k+1}}\Big[\log\frac{\psi_{k+1}(X)}{p_{k+1}^{\mathsf{c}}(X)}\Big] - \mathbb{E}_{X\sim P_{\gamma_k}\psi_{k}}\Big[\log\frac{\psi_{k+1}(X)}{p_{k+1}^{\mathsf{c}}(X)}\Big]\notag\\
    & = O(\gamma_k^2).
\end{align}
\end{itemize}

Substituting \eqref{eq:proof-14} and \eqref{eq:proof-15} into \eqref{eq:proof-13}, we have
\begin{align}\label{eq:proof-lem-distri-xsc-1}
\mathsf{KL}(\psi_{k+1}\Vert p_{k+1}^{\mathsf{c}})&\le {\rm e}^{-2\kappa\gamma_k}\mathsf{KL}(\psi_{k}\Vert p_k^{\mathsf{c}})   + O\big(\sqrt{\delta_{k+1}\log\delta^{-1}}\big) + O(\gamma_k^2)\notag\\
&\le {\rm e}^{-2\kappa\sum_{i=k-N}^k{\gamma_i}}\mathsf{KL}(\psi_{k-N}\Vert p_{k-N}^{\mathsf{c}}) + O\Big(\frac{\sqrt{\delta}\log\delta^{-1}}{\delta\eta}\Big) + O(\max_k \gamma_k)\notag\\
&\lesssim  {\rm e}^{-2\Theta( \sqrt{\delta})\eta} + \frac{\log\delta^{-1}}{\sqrt{\delta}\eta} + \max_k \gamma_k\to0,\qquad\mathrm{as}\quad \frac{\delta\eta^2}{\log\delta^{-1}}\to\infty,\quad \delta\eta\to 0,
\end{align}
where we have used the assumption that $\sum_{i=k-N}^k\gamma_i\asymp \sqrt{\delta}\eta \to \infty$.
%
Futhermore, notice that
\begin{align}
\mathbb{P}(\mathcal{E}_k^{\mathsf{c}}) \le\sum_{i=1}^{k+1}\mathbb{P}\left(O(\delta_{i+1})+\sqrt{\delta_{i+1}}\|Z_{t_i}^{\mathsf{sc}}\|_2\le O(\sqrt{\delta_{i+1}}\log\delta^{-1})\right) \le \delta^{10}\to 0,\qquad \mathrm{as} \quad \delta \to 0.
\end{align}
We have
\begin{align*}
&\mathsf{TV}(p_{\breve{X}_{k}^{\mathsf{c}}|\breve{X}_{k}^{\mathsf{s}}}(\cdot\mymid \breve{x}_{k}^{\mathsf{s}}), p_{X_{t_k}^{\mathsf{c}}|X_{t_k}^{\mathsf{s}}}(\cdot\mymid \breve{x}_k^{\mathsf{s}}) ) \notag\\
&\quad \le \mathbb{E}_{\breve{x}_{1:k-1}^{\mathsf{s}}\sim \breve{X}_{1:k-1}^{\mathsf{s}}\mymid \breve{X}_{k}^{\mathsf{s}}}\big[\mathsf{TV}(p_{\breve{X}_{k}^{\mathsf{c}}|\breve{X}_{1:k}^{\mathsf{s}}}(\cdot\mymid \breve{x}_{1:k}^{\mathsf{s}}), p_{X_{t_k}^{\mathsf{c}}|X_{t_k}^{\mathsf{s}}}(\cdot\mymid x_k^{\mathsf{s}}) ) \big]\notag\\
&\quad \le \mathbb{E}_{\breve{x}_{1:k-1}^{\mathsf{s}}\sim \breve{X}_{1:k-1}^{\mathsf{s}}\mymid \breve{X}_{k}^{\mathsf{s}},\breve{x}_{1:k-1}^{\mathsf{s}}\in\mathcal{E}_k}\big[\mathsf{TV}(p_{\breve{X}_{k}^{\mathsf{c}}|\breve{X}_{1:k}^{\mathsf{s}}}(\cdot\mymid \breve{x}_{1:k}^{\mathsf{s}}), p_{X_{t_k}^{\mathsf{c}}|X_{t_k}^{\mathsf{s}}}(\cdot\mymid x_k^{\mathsf{s}}) ) \big]
+ \mathbb{P}(\mathcal{E}_k^{\mathsf{c}}) \to 0.
\end{align*}
We complete the proof.

\paragraph{Proof of \eqref{eq:proof-breveX-tildeX}.}

By Pinsker's inequality, we have
\begin{align*}
&\quad \mathsf{TV}(p_{\breve{X}_k\mymid Y,Z},p_{\widetilde{X}_k\mymid Y,Z})^2 \le \mathsf{KL}(p_{\widetilde{X}_k\mymid Y,Z} \Vert p_{\breve{X}_k\mymid Y,Z}) \notag\\
&\le \sum_{i=k-N}^{k-1} \mathbb{E}_{x_i^{\mathsf{sc}}\sim\widetilde{X}_i^{\mathsf{sc}},x_i^{\mathsf{c}}\sim\widetilde{X}_i^{\mathsf{c}},x_{i-N}^{\mathsf{c}}\sim\widetilde{X}_{i-N}^{\mathsf{c}}}\left[\mathsf{KL}(p_{\widetilde{X}_{i+1}^{\mathsf{sc}}\mymid \widetilde{X}_{i}^{\mathsf{sc}},\widetilde{X}_{i}^{\mathsf{c}}}(\cdot\mymid x_i^{\mathsf{sc}},x_i^{\mathsf{c}}) \Vert p_{\breve{X}_{i+1}^{\mathsf{sc}}\mymid \breve{X}_{i}^{\mathsf{sc}},\breve{X}_{i-N}^{\mathsf{c}}}(\cdot\mymid x_i^{\mathsf{sc}},x_{i-N}^{\mathsf{c}})) \right]\notag\\
&\qquad  + \sum_{i=k-N}^{k-1} \mathbb{E}_{x_{i}^{\mathsf{c}}\sim\widetilde{X}_{i}^{\mathsf{c}},x_{i+1}^{\mathsf{sc}}\sim\widetilde{X}_{i+1}^{\mathsf{sc}}}\left[\mathsf{KL}(p_{\widetilde{X}_{i+1}^{\mathsf{c}}\mymid \widetilde{X}_{i}^{\mathsf{c}},\widetilde{X}_{i+1}^{\mathsf{sc}}}(\cdot\mymid x_i^{\mathsf{c}},x_{i+1}^{\mathsf{sc}}) \Vert p_{\breve{X}_{i+1}^{\mathsf{c}}\mymid \breve{X}_{i}^{\mathsf{c}},\breve{X}_{i+1}^{\mathsf{sc}}}(\cdot\mymid x_i^{\mathsf{c}},x_{i+1}^{\mathsf{sc}})) \right]\notag\\
&= \sum_{i=k-N}^{k-1} \mathbb{E}_{x_i^{\mathsf{sc}}\sim\widetilde{X}_i^{\mathsf{sc}},x_i^{\mathsf{c}}\sim\widetilde{X}_i^{\mathsf{c}},x_{i-N}^{\mathsf{c}}\sim\widetilde{X}_{i-N}^{\mathsf{c}}}\left[\mathsf{KL}(p_{\widetilde{X}_{i+1}^{\mathsf{sc}}\mymid \widetilde{X}_{i}^{\mathsf{sc}},\widetilde{X}_{i}^{\mathsf{c}}}(\cdot\mymid x_i^{\mathsf{sc}},x_i^{\mathsf{c}}) \Vert p_{\breve{X}_{i+1}^{\mathsf{sc}}\mymid \breve{X}_{i}^{\mathsf{sc}},\breve{X}_{i-N}^{\mathsf{c}}}(\cdot\mymid x_i^{\mathsf{sc}},x_{i-N}^{\mathsf{c}})) \right],
\end{align*}
where the last equation holds since $p_{\widetilde{X}_{i+1}^{\mathsf{c}}\mymid \widetilde{X}_{i}^{\mathsf{c}},\widetilde{X}_{i+1}^{\mathsf{sc}}}(\cdot\mymid x_i^{\mathsf{c}},x_{i+1}^{\mathsf{sc}}) =p_{\breve{X}_{i+1}^{\mathsf{c}}\mymid \breve{X}_{i}^{\mathsf{c}},\breve{X}_{i+1}^{\mathsf{sc}}}(\cdot\mymid x_i^{\mathsf{c}},x_{i+1}^{\mathsf{sc}})$ by \eqref{eq:def-breveX-c}.
According to \eqref{eq:def-breveX-sc}, both $p_{\widetilde{X}_{i+1}^{\mathsf{s}}\mymid \widetilde{X}_{i}^{\mathsf{s}},\widetilde{X}_{i}^{\mathsf{c}}}(\cdot\mymid x_i^{\mathsf{s}},x_i^{\mathsf{c}})$ and $p_{\breve{X}_{i+1}^{\mathsf{sc}}\mymid \breve{X}_{i}^{\mathsf{sc}},\breve{X}_{i-N}^{\mathsf{c}}}(\cdot\mymid x_i^{\mathsf{s}},x_{i-N}^{\mathsf{c}})$ are Gaussian distributions with the same variance $\sigma_{t_{i+1}}^2(1-{\rm e}^{-2\delta_{i+1}})I_d$ and different mean vectors.
We have
\begin{align*}
    &\mathsf{KL}\left( p_{\widetilde{X}_{i+1}^{\mathsf{s}}\mymid \widetilde{X}_{i}^{\mathsf{s}},\widetilde{X}_{i}^{\mathsf{c}}}(\cdot\mymid x_i^{\mathsf{s}},x_i^{\mathsf{c}})\Vert p_{\breve{X}_{i+1}^{\mathsf{sc}}\mymid \breve{X}_{i}^{\mathsf{sc}},\breve{X}_{i-N}^{\mathsf{c}}}(\cdot\mymid x_i^{\mathsf{s}},x_{i-N}^{\mathsf{c}}) \right)\notag\\
    &\quad = \frac{\alpha_{t_{i+1}}^2\sigma_{t_i}^4(1-{\rm e}^{-2\delta_{i+1}})^2}{2\alpha_{t_i}^2\sigma_{t_{i+1}}^2(1-{\rm e}^{-2\delta_{i+1}})}\left\|\nabla_{x_{i}^{\mathsf{sc}}} \log p_{\lambda_i}(x_{i-N}^{\mathsf{c}},x_{i}^{\mathsf{sc}})-\nabla_{x_{i}^{\mathsf{sc}}} \log p_{\lambda_i}(x_{i}^{\mathsf{c}},x_{i}^{\mathsf{sc}})\right\|_2^2\notag\\
    &\quad \lesssim \delta_{i+1}\left\|\nabla_{x_{i}^{\mathsf{sc}}} \log p_{\lambda_i}(x_{i-N}^{\mathsf{c}},x_{i}^{\mathsf{sc}})-\nabla_{x_{i}^{\mathsf{sc}}} \log p_{\lambda_i}(x_{i}^{\mathsf{c}},x_{i}^{\mathsf{sc}})\right\|_2^2.
\end{align*}
Therefore, we have
\begin{align*}
\mathsf{TV}(p_{\breve{X}_k\mymid Y,Z},p_{\widetilde{X}_k\mymid Y,Z})^2 &\lesssim \sum_{i=k-N}^{k-1} \delta_{i+1}\mathbb{E}_{x_i^{\mathsf{sc}}\sim\widetilde{X}_i^{\mathsf{sc}},x_i^{\mathsf{c}}\sim\widetilde{X}_i^{\mathsf{c}},x_{i-N}^{\mathsf{c}}\sim\widetilde{X}_{i-N}^{\mathsf{c}}}\left[\left\|\nabla_{x_{i}^{\mathsf{sc}}} \log p_{\lambda_i}(x_{i-N}^{\mathsf{c}},x_{i}^{\mathsf{sc}})-\nabla_{x_{i}^{\mathsf{sc}}} \log p_{\lambda_i}(x_{i}^{\mathsf{c}},x_{i}^{\mathsf{sc}})\right\|_2^2 \right]\notag\\
&\lesssim \sum_{i=k-N}^{k-1}\delta_{i+1} \asymp \delta^{1/2} \to 0,\qquad \mathrm{as}\quad \delta\to 0.
\end{align*}
Thus we complete the proof.

\subsection{Proof of Lemma \ref{lem:distri-xsc}}
\label{subsec:proof-lem-distri-xsc}
It suffices to prove that
\begin{align}
\mathsf{KL}(p_{\widehat{X}_{t_i}^{\mathsf{sc}}\mymid \widehat{X}_{t_{i-1}}}\Vert p_{X_{t_i}^{\mathsf{sc}}\mymid {X}_{t_{i-1}}}) = o(\delta_{i}).
\end{align}
According to the update rule, $p_{\widehat{X}_{t_i}^{\mathsf{sc}}|\widehat{X}_{t_{i-1}}}$ is a Gaussian distribution.
We consider the following density function:
\begin{align}\label{eq:proof-lem-dsitri-xsc-4}
\widehat{p}_i(x_i|x_{i-1}) \coloneqq (2\pi\varrho^2)^{-d/2}\exp\left(-\frac{1}{2\varrho^2}\|x_i-\mu_0-\Delta\|_2^2\right),
\end{align} 
where variance $\varrho^2=\sigma_{t_{i}}^2(1-{\rm e}^{-2\delta_{i}})$ and the mean vecotr $\mu_0+\Delta$ is given by
\begin{align*}
    \mu_0 = \frac{\alpha_{t_{i}}}{\alpha_{t_{i-1}}}\widehat{X}_{t_{i-1}},\qquad \Delta = \alpha_{t_{i}}\sigma_{t_{i-1}}\left({\rm e}^{-\lambda_{t_{i-1}}}-{\rm e}^{-2\lambda_{t_{i}}+\lambda_{t_{i-1}}}\right)\nabla_{{x}_{i-1}} \log p_{X_{t_{i-1}}}({x}_{i-1}).
\end{align*} 
By noticing that $V_{\Stci{t_i}}^{\top}\nabla_{{x}_{i-1}} \log p_{X_{t_{i-1}}}({x}_{i-1}) = \nabla_{{x}_{i-1}^{\mathsf{sc}}} \log p_{\lambda_{i-1}}({x}_{i-1}^{\mathsf{sc}},{x}_{i-1}^{\mathsf{c}})$, we conclude that $\widehat{p}_i(x_i|x_{i-1})$ has the same marginal distribution on directions $\Stci{t_i}$ as $p_{\widehat{X}_{t_i}^{\mathsf{sc}}|\widehat{X}_{t_{i-1}}}$.
It suffices to prove that
\begin{align}\label{eq:proof-lem-distri-xsc-9}
\mathsf{KL}(p_{\widehat{X}_{t_i}^{\mathsf{sc}}\mymid \widehat{X}_{t_{i-1}}}\Vert p_{X_{t_i}^{\mathsf{sc}}\mymid {X}_{t_{i-1}}}) \le \mathsf{KL}( \widehat{p}_{i}\Vert p_{X_{t_i}\mymid {X}_{t_{i-1}}})  = o(\delta_{i}).
\end{align}

Towards this, we decompose $\widehat{p}_i$ and  $p_{X_{t_i}\mymid {X}_{t_{i-1}}}$ seperately.
\begin{itemize}
\item \emph{Decomposition of $\widehat{p}_i(x_i|x_{i-1})$.} Expanding the exponent expression in \eqref{eq:proof-lem-dsitri-xsc-4} yields
\begin{align}\label{eq:proof-lem-dsitri-xsc-5}
\widehat{p}_i(x_i|x_{i-1}) =(2\pi\varrho^2)^{-d/2}\exp\left(-\frac{1}{2\varrho^2}\|x_i-\mu_0\|_2^2 + \frac{1}{\varrho^2}\Delta^{\top}(x_i-\mu_0) + \frac{1}{2\varrho^2}\|\Delta\|_2^2\right).
\end{align}
Note that
\begin{align*}
\frac{1}{\varrho^2}\Delta^{\top}(x_i-\mu_0) &= \frac{1}{\varrho^2}\alpha_{t_{i}}\sigma_{t_{i-1}}\left({\rm e}^{-\lambda_{t_{i-1}}}-{\rm e}^{-2\lambda_{t_{i}}+\lambda_{t_{i-1}}}\right)\nabla_{{x}_{i-1}} \log p_{X_{t_{i-1}}}({x}_{i-1})^{\top}(x_i-\mu_0)\notag\\
&= \frac{1}{\varrho^2}\alpha_{t_{i}}\sigma_{t_{i-1}}{\rm e}^{-\lambda_{t_{i-1}}}\left(1-{\rm e}^{-2\delta_{i}}\right)\nabla_{{x}_{i-1}} \log p_{X_{t_{i-1}}}({x}_{i-1})^{\top}(x_i-\mu_0)\notag\\
&= \frac{1}{\varrho^2}\frac{\alpha_{t_{i}}^2\sigma_{t_{i-1}}^2}{\alpha_{t_{i-1}}^2}\left(1-{\rm e}^{-2\delta_{i}}\right)\nabla_{{x}_{i-1}} \log p_{X_{t_{i-1}}}({x}_{i-1})^{\top}\big(\frac{\alpha_{t_{i-1}}}{\alpha_{t_{i}}}x_i-\widehat{X}_{t_{i-1}}\big)\notag\\
&={\rm e}^{2\delta_i}\nabla_{{x}_{i-1}} \log p_{X_{t_{i-1}}}({x}_{i-1})^{\top}\big(\frac{\alpha_{t_{i-1}}}{\alpha_{t_{i}}}x_i-\widehat{X}_{t_{i-1}}\big).
\end{align*}
By denoting $\widehat{x}_{i}=\alpha_{t_{i-1}}x_i/\alpha_{t_i}$, and noticing that $\|\Delta\|_2 = O(\delta_i)$, we have
\begin{align}\label{eq:proof-lem-dsitri-xsc-6}
\widehat{p}_i(x_i|x_{i-1}) =(2\pi\varrho^2)^{-d/2}\exp\left(-\frac{1}{2\varrho^2}\|x_i-\mu_0\|_2^2 + {\rm e}^{2\delta_i}\nabla_{{x}_{i-1}} \log p_{X_{t_{i-1}}}({x}_{i-1})^{\top}\big(\widehat{x}_{i}-\widehat{X}_{t_{i-1}} \big)+ O(\delta_i)\right)
\end{align}

\item \emph{Decomposition of $p_{X_{t_i}\mymid {X}_{t_{i-1}}}(x_i|x_{i-1})$.}
\begin{align}\label{eq:proof-lem-distri-sc-3}
p_{X_{t_i}\mymid {X}_{t_{i-1}}}(x_i|x_{i-1}) = p_{X_{t_{i-1}}\mymid {X}_{t_{i}}}(x_{i-1}|x_{i})\frac{p_{X_{t_{i}}}(x_{i})}{p_{X_{t_{i-1}}}(x_{i-1})}.
\end{align}
Furthermore, according to the forward process 
\begin{align*}
X_{t_{i-1}} &= \frac{\alpha_{t_{i-1}}}{\alpha_{t_{i}}} X_{t_{i}} + \left(\sigma_{t_{i-1}}^2-\frac{\alpha_{t_{i-1}}}{\alpha_{t_{i}}}\sigma_{t_{i}}^2\right)\mathcal{N}(0,I_d) = \frac{\alpha_{t_{i-1}}}{\alpha_{t_{i}}} X_{t_{i}} + \sigma_{t_{i-1}}\left(1-{\rm e}^{-2\delta_i}\right)^{1/2}\mathcal{N}(0,I_d)\notag\\
&= \frac{\alpha_{t_{i-1}}}{\alpha_{t_{i}}} X_{t_{i}} + \frac{\sigma_{t_{i-1}}}{\sigma_{t_i}}\varrho\mathcal{N}(0,I_d),
\end{align*}
we have
\begin{align}\label{eq:proof-lem-distri-sc-2}
p_{X_{t_{i-1}}\mymid {X}_{t_{i}}}(x_{i-1}|x_{i}) 
&= (2\pi\varrho^2)^{-d/2}\left(\frac{\sigma_{t_{i}}}{\sigma_{t_{i-1}}}\right)^{d}\exp\left(-\frac{\sigma_{t_{i}}^2}{2\varrho^2\sigma_{t_{i-1}}^2}\big\|\widehat{X}_{t_{i-1}}-\frac{\alpha_{t_{i-1}}}{\alpha_{t_{i}}} \widehat{X}_{t_{i}}\big\|_2^2\right)\notag\\
&=(2\pi\varrho^2)^{-d/2}\left(\frac{\sigma_{t_{i}}}{\sigma_{t_{i-1}}}\right)^{d}\exp\left(-\frac{\sigma_{t_{i}}^2\alpha_{t_{i-1}}^2}{2\varrho^2\sigma_{t_{i-1}}^2\alpha_{t_{i}}^2}
\big\|\frac{\alpha_{t_{i}}}{\alpha_{t_{i-1}}} \widehat{X}_{t_{i-1}}-\widehat{X}_{t_{i}}\big\|_2^2\right)\notag\\
&=(2\pi\varrho^2)^{-d/2}\left(\frac{\sigma_{t_{i}}}{\sigma_{t_{i-1}}}\right)^{d}\exp\left(-\frac{{\rm e}^{-2\delta_{i}}}{2\varrho^2}
\big\|\widehat{X}_{t_{i}}-\mu_0\big\|_2^2\right).
\end{align}
For the ratio $\frac{p_{X_{t_{i}}}(x_{i})}{p_{X_{t_{i-1}}}(x_{i-1})}$, we have
\begin{align}\label{eq:proof-lem-distri-sc-1}
\frac{p_{X_{t_{i}}}(x_{i})}{p_{X_{t_{i-1}}}(x_{i-1})} &= \frac{p_{X_{t_{i}}}(x_{i})}{p_{X_{t_{i-1}}}(\widehat{x}_{i})} \frac{p_{X_{t_{i-1}}}(\widehat{x}_{i})}{p_{X_{t_{i-1}}}(x_{i-1})} = \frac{p_{X_{t_{i}}}(x_{i})}{p_{X_{t_{i-1}}}(\widehat{x}_{i})}\exp\left(\log p_{X_{t_{i-1}}}(\widehat{x}_{i}) - \log p_{X_{t_{i-1}}}(x_{i-1}) \right) \notag\\
&=\frac{p_{X_{t_{i}}}(x_{i})}{p_{X_{t_{i-1}}}(\widehat{x}_{i})}\exp\left(\nabla \log p_{X_{t_{i-1}}}(x_{i-1})^{\top}(\widehat{x}_{i}-x_{i-1}) \right). 
\end{align}
Furthermore, we have
\begin{align*}
&\quad\frac{p_{X_{t_{i}}}(x_{i})}{p_{X_{t_{i-1}}}(\widehat{x}_{i})} \notag\\
&= \left(\frac{\sigma_{t_{i-1}}}{\sigma_{t_i}}\right)^{d}\frac{\int p_{X_0}(x_0)\exp\left(-\frac{1}{2\sigma_{t_i}^2}\|x_i-\alpha_{t_i}x_0\|_2^2\right)\mathrm{d} x_0}{\int p_{X_0}(x_0)\exp\left(-\frac{{\rm e}^{-2\delta_i}}{2\sigma_{t_i}^2}\|x_i-\alpha_{t_i}x_0\|_2^2\right)\mathrm{d} x_0}\notag\\
&=\left(\frac{\sigma_{t_{i-1}}}{\sigma_{t_i}}\right)^{d}\frac{\int p_{X_0}(x_0)\exp\left(-\frac{{\rm e}^{-2\delta_i}}{2\sigma_{t_i}^2}\|x_i-\alpha_{t_i}x_0\|_2^2\right)\left(1-\frac{(1-{\rm e}^{-2\delta_i})}{2\sigma_{t_i}^2}\|x_i-\alpha_{t_i}x_0\|_2^2 + O(\delta_i^2)\right)\mathrm{d} x_0}{\int p_{X_0}(x_0)\exp\left(-\frac{{\rm e}^{-2\delta_i}}{2\sigma_{t_i}^2}\|x_i-\alpha_{t_i}x_0\|_2^2\right)\mathrm{d} x_0},
\end{align*}
where the first equation comes from
\begin{align*}
p_{X_{t_{i}}}(x_{i}) &= (2\pi\sigma_{t_i}^2)^{-d/2}\int p_{X_0}(x_0)\exp\left(-\frac{1}{2\sigma_{t_i}^2}\|x_i-\alpha_{t_i}x_0\|_2^2\right)\mathrm{d} x_0\notag\\
p_{X_{t_{i-1}}}(\widehat{x}_{i})&=(2\pi\sigma_{t_{i-1}}^2)^{-d/2}\int p_{X_0}(x_0)\exp\left(-\frac{1}{2\sigma_{t_{i-1}}^2}\big\|\frac{\alpha_{t_{i-1}}}{\alpha_{t_{i}}}x_{i}-\alpha_{t_{i-1}}x_0\big\|_2^2\right)\mathrm{d} x_0\notag\\
&=(2\pi\sigma_{t_{i-1}}^2)^{-d/2}\int p_{X_0}(x_0)\exp\left(-\frac{\alpha_{t_{i-1}}^2}{2\sigma_{t_{i-1}}^2\alpha_{t_{i}}^2}\big\|x_{i-1}-\alpha_{t_{i}}x_0\big\|_2^2\right)\mathrm{d} x_0\notag\\
&=(2\pi\sigma_{t_{i-1}}^2)^{-d/2}\int p_{X_0}(x_0)\exp\left(-\frac{{\rm e}^{-2\delta_i}}{2\sigma_{t_{i}}^2}\big\|x_{i}-\alpha_{t_{i}}x_0\big\|_2^2\right)\mathrm{d} x_0.
\end{align*}
Thus we have
\begin{align*}
&\quad\left|\left(\frac{\sigma_{t_{i-1}}}{\sigma_{t_i}}\right)^{-d}\frac{p_{X_{t_{i}}}(x_{i})}{p_{X_{t_{i-1}}}(\widehat{x}_{i})} - 1\right|\notag\\
 &\le \frac{\int p_{X_0}(x_0)\exp\left(-\frac{{\rm e}^{-2\delta_i}}{2\sigma_{t_i}^2}\|x_i-\alpha_{t_i}x_0\|_2^2\right)\left(\frac{2\delta_i}{2\sigma_{t_i}^2}\|x_i-\alpha_{t_i}x_0\|_2^2 + O(\delta_i^2)\right)\mathrm{d} x_0}{\int p_{X_0}(x_0)\exp\left(-\frac{{\rm e}^{-2\delta_i}}{2\sigma_{t_i}^2}\|x_i-\alpha_{t_i}x_0\|_2^2\right)\mathrm{d} x_0} \notag\\
&= O(\delta_i)(\|x_i\|_2^2+1).
\end{align*}
Instituting into \eqref{eq:proof-lem-distri-sc-1}, we have
\begin{align}\label{eq:proof-lem-distri-sc-4}
    \frac{p_{X_{t_{i}}}(x_{i})}{p_{X_{t_{i-1}}}(x_{i-1})} &= \left(\frac{\sigma_{t_{i-1}}}{\sigma_{t_i}}\right)^{-d}\exp\left(\nabla \log p_{X_{t_{i-1}}}(x_{i-1})^{\top}(\widehat{x}_{i}-x_{i-1}) \right)\left(1+O(\delta_i)(1+\|x_i\|_2^2)\right)
\end{align}
Combining this with \eqref{eq:proof-lem-distri-sc-2} and \eqref{eq:proof-lem-distri-sc-3} yields
\begin{align}\label{eq:proof-lem-dsitri-xsc-7}
    p_{X_{t_i}\mymid {X}_{t_{i-1}}}(x_i|x_{i-1})
    &=(2\pi\varrho^2)^{-d/2}\exp\left(-\frac{{\rm e}^{-2\delta_{i}}}{2\varrho^2}
\big\|\widehat{X}_{t_{i}}-\mu_0\big\|_2^2\right)\notag\\
&\quad \cdot\exp\left(\nabla \log p_{X_{t_{i-1}}}(x_{i-1})^{\top}(\widehat{x}_{i}-x_{i-1}) \right)\left(1+O(\delta_i)(1+\|x_i\|_2^2)\right)\notag\\
&=(2\pi\varrho^2)^{-d/2}\exp\left(-\frac{{\rm e}^{-2\delta_{i}}}{2\varrho^2}
\big\|\widehat{X}_{t_{i}}-\mu_0\big\|_2^2+\nabla \log p_{X_{t_{i-1}}}(x_{i-1})^{\top}(\widehat{x}_{i}-x_{i-1}) \right)\notag\\
&\quad \cdot\left(1+O(\delta_i)(1+\|x_i\|_2^2)\right)
\end{align}
\end{itemize} 

Combining \eqref{eq:proof-lem-dsitri-xsc-6} and \eqref{eq:proof-lem-dsitri-xsc-7} leads to
\begin{align*}
    \frac{p_{X_{t_i}\mymid {X}_{t_{i-1}}}(x_i|x_{i-1})}{\widehat{p}_i(x_i|x_{i-1})} 
    &= \exp\left(\frac{(1-{\rm e}^{-2\delta_{i}})}{2\varrho^2}
\big\|x_{i}-\mu_0\big\|_2^2-({\rm e}^{2\delta_{i}}-1)\nabla \log p_{X_{t_{i-1}}}(x_{i-1})^{\top}(\widehat{x}_{i}-x_{i-1})+O(\delta_i)\right)\notag\\
&\qquad \cdot(1+O(\delta_i)(1+\|x_i\|_2^2))
\end{align*}
and 
\begin{align*}
    \frac{p_{X_{t_i}\mymid {X}_{t_{i-1}}}(x_i|x_{i-1})}{\widehat{p}_i(x_i|x_{i-1})}  - 1 = O(\delta_i)\frac{\|x_i-\mu_0\|_2^2}{\varrho^2} + O(\delta_i)\|\widehat{x}_i-x_{i-1}\|_2 + O(\delta_i)\|x_i\|^2 + O(\delta_i).
\end{align*}
By using the definition of KL divergence, we have
\begin{align*}
\mathsf{KL}(\widehat{p}_{i}\Vert p_{X_{t_i}\mymid {X}_{t_{i-1}}}) &= \int \log \frac{\widehat{p}_{i}(x_i|x_{i-1})}{p_{X_{t_i}\mymid {X}_{t_{i-1}}}(x_i|x_{i-1})}  \widehat{p}_{i}(x_i|x_{i-1}) \mathrm{d} x_i\notag\\
&=-\int \log \frac{p_{X_{t_i}\mymid {X}_{t_{i-1}}}(x_i|x_{i-1})}{\widehat{p}_{i}(x_i|x_{i-1})}  \widehat{p}_{i}(x_i|x_{i-1}) \mathrm{d} x_i\notag\\
&=-\int \left(\frac{p_{X_{t_i}\mymid {X}_{t_{i-1}}}(x_i|x_{i-1})}{\widehat{p}_{i}(x_i|x_{i-1})}  -1\right)\widehat{p}_{i}(x_i|x_{i-1}) \mathrm{d} x_i \notag\\
&\quad + \frac{1}{2}\int \left(\frac{p_{X_{t_i}\mymid {X}_{t_{i-1}}}(x_i|x_{i-1})}{\widehat{p}_{i}(x_i|x_{i-1})}  -1\right)^2\widehat{p}_{i}(x_i|x_{i-1}) \mathrm{d} x_i\notag\\
&=\frac{1}{2}\int \left(\frac{p_{X_{t_i}\mymid {X}_{t_{i-1}}}(x_i|x_{i-1})}{\widehat{p}_{i}(x_i|x_{i-1})}  -1\right)^2\widehat{p}_{i}(x_i|x_{i-1}) \mathrm{d} x_i\notag\\
&=O(\delta_i^2)\mathbb{E}_{x_i\sim \widehat{p}_i(\cdot\mymid x_{i-1})}\Bigg[\frac{\|x_i-\mu_0\|_2^2}{\varrho^2} + \|\widehat{x}_i-x_{i-1}\|_2 + \|x_i\|^2+1\Bigg] = O(\delta_i^2),
\end{align*}
where the last equation holds since by \eqref{eq:proof-lem-dsitri-xsc-4},
$$
\frac{\|x_i-\mu_0\|_2^2}{\varrho^2} = \frac{\|\Delta\|_2^2 + \varrho^2}{\varrho^2} = O(1).
$$
Combining with \eqref{eq:proof-lem-distri-xsc-9}, we complete the proof.

\subsection{Proof of Lemma \ref{lem:distri-sc-conds}}
\label{subsec:proof-lem-distri-sc-conds}
We start from the observation that $p_{\widetilde{X}_i^{\mathsf{sc}}\mymid \widetilde{X}_{i-1}}({x}_{i}^{\mathsf{sc}}\mymid {x}_{i-1})$ is a Gaussian distribution and $x_{i-1}^{\mathsf{c}}$ influences the mean vector:
\begin{align*}
p_{\widetilde{X}_i^{\mathsf{sc}}\mymid \widetilde{X}_{i-1}}({x}_{i}^{\mathsf{sc}}\mymid {x}_{i-1}) = \phi({x}_{i}^{\mathsf{sc}}; \mu_0 + \Delta(x_{i-1}^{\mathsf{c}}),\sigma_{t_{i}}^2(1-{\rm e}^{-2\delta_{i}})),
\end{align*}
where $\phi(\cdot;\mu,\sigma^2)$ denotes the density function of Gaussian distribution with mean vector $\mu$ and variance $\sigma^2$, and
\begin{align}
\mu_0&\coloneqq \frac{\alpha_{t_i}}{\alpha_{t_{i-1}}} x_{i-1}^{\mathsf{sc}},\notag\\
\Delta(x_{i-1}^{\mathsf{c}}) &\coloneqq \alpha_{t_{i}}\sigma_{t_{i-1}}\left({\rm e}^{-\lambda_{t_{i-1}}}-{\rm e}^{-2\lambda_{t_{i}}+\lambda_{t_{i-1}}}\right)\nabla_{{x}_{i-1}^{\mathsf{sc}}} \log p_{\lambda_{i-1}}({x}_{i-1}^{\mathsf{c}},{x}_{i-1}^{\mathsf{sc}}).
\end{align}
It can be immediately check that
\begin{align}
\|\Delta(x_{i-1}^{\mathsf{c}})\|_2 \lesssim O(\delta).
\end{align}
By Taylor expansion and denoting the variance as $\varrho^2=\sigma_{t_{i}}^2(1-{\rm e}^{-2\delta_{i}})$, we have
\begin{align*}
   p_{\widetilde{X}_i^{\mathsf{sc}}\mymid \widetilde{X}_{i-1}}({x}_{i}^{\mathsf{sc}}\mymid {x}_{i-1}) &= \phi({x}_{i}^{\mathsf{sc}}; \mu_0,\varrho^2)\left(1 + \frac{({x}_{i}^{\mathsf{sc}}-\mu_0)^{\top}\Delta(x_{i-1}^{\mathsf{c}})}{\varrho^2} - \frac{\|\Delta(x_{i-1}^{\mathsf{c}})\|_2^2}{\varrho^2} + \frac{(({x}_{i}^{\mathsf{sc}}-\mu_0)^{\top}\Delta(x_{i-1}^{\mathsf{c}}))^2}{\varrho^2}\right)\notag\\
   &\quad  + O(\|\Delta(x_{i-1}^{\mathsf{c}})\|_2^3).
\end{align*}
Thus we have
\begin{align*}
&\int p_{\widetilde{X}_i^{\mathsf{sc}}\mymid \widetilde{X}_{i-1}}({x}_{i}^{\mathsf{sc}}\mymid {x}_{i-1})p_{{X}_{t_i}^{\mathsf{sc}}\mymid {X}_{t_{i-1}}}({x}_{i-1}^{\mathsf{c}}\mymid {x}_{i-1}^{\mathsf{s}}) \mathrm{d}{x}_{i-1}^{\mathsf{c}} \notag\\
&\quad= \int \phi({x}_{i}^{\mathsf{sc}}; \mu_0,\varrho^2)\left(1 + \frac{({x}_{i}^{\mathsf{sc}}-\mu_0)^{\top}\Delta(x_{i-1}^{\mathsf{c}})}{\varrho^2} - \frac{\|\Delta(x_{i-1}^{\mathsf{c}})\|_2^2}{\varrho^2} + \frac{(({x}_{i}^{\mathsf{sc}}-\mu_0)^{\top}\Delta(x_{i-1}^{\mathsf{c}}))^2}{\varrho^2}\right)\notag\\
&\qquad \cdot p_{{X}_{t_{i-1}}^{\mathsf{c}}\mymid {X}_{t_{i-1}}^{\mathsf{s}}}({x}_{i-1}^{\mathsf{c}}\mymid {x}_{i-1}^{\mathsf{s}}) \mathrm{d}{x}_{i-1}^{\mathsf{c}}\notag\\
&\qquad + O(\mathbb{E}_{{X}_{t_{i-1}}^{\mathsf{c}}\mymid {X}_{t_{i-1}}^{\mathsf{s}}}[\|\Delta({X}_{t_{i-1}}^{\mathsf{c}})\|_2^3])\notag\\
&\quad= \phi({x}_{i}^{\mathsf{sc}}; \mu_0,\varrho^2)\bigg(1 + \varrho^{-2}({x}_{i}^{\mathsf{sc}}-\mu_0)^{\top}\mathbb{E}_{{X}_{t_{i-1}}^{\mathsf{c}}\mymid {X}_{t_{i-1}}^{\mathsf{s}}}[\Delta({X}_{t_{i-1}}^{\mathsf{c}})] - \varrho^{-2}\mathbb{E}_{{X}_{t_{i-1}}^{\mathsf{c}}\mymid {X}_{t_{i-1}}^{\mathsf{s}}}[\|\Delta({X}_{t_{i-1}}^{\mathsf{c}})\|_2^2]\notag\\
&\qquad\qquad\qquad + \varrho^{-2}\mathbb{E}_{{X}_{t_{i-1}}^{\mathsf{c}}\mymid {X}_{t_{i-1}}^{\mathsf{s}}}[(({x}_{i}^{\mathsf{sc}}-\mu_0)^{\top}\Delta({X}_{t_{i-1}}^{\mathsf{c}}))^2]\bigg)  + O(\mathbb{E}_{{X}_{t_{i-1}}^{\mathsf{c}}\mymid {X}_{t_{i-1}}^{\mathsf{s}}}[\|\Delta({X}_{t_{i-1}}^{\mathsf{c}})\|_2^3]).
\end{align*}
Similarly, we have
\begin{align*}
&\int p_{\widetilde{X}_i^{\mathsf{sc}}\mymid \widetilde{X}_{i-1}}({x}_{i}^{\mathsf{sc}}\mymid {x}_{i-1})p_{\widetilde{X}_{i-1}^{\mathsf{c}}\mymid \widetilde{X}_{i-1}^{\mathsf{s}}}({x}_{i-1}^{\mathsf{c}}\mymid {x}_{i-1}^{\mathsf{s}}) \mathrm{d}{x}_{i-1}^{\mathsf{c}} \notag\\
&\quad= \phi({x}_{i}^{\mathsf{sc}}; \mu_0,\varrho^2)\bigg(1 + \varrho^{-2}({x}_{i}^{\mathsf{sc}}-\mu_0)^{\top}\mathbb{E}_{\widetilde{X}_{i-1}^{\mathsf{c}}\mymid \widetilde{X}_{i-1}^{\mathsf{s}}}[\Delta(\widetilde{X}_{i-1}^{\mathsf{c}})] - \varrho^{-2}\mathbb{E}_{\widetilde{X}_{i-1}^{\mathsf{c}}\mymid \widetilde{X}_{i-1}^{\mathsf{s}}}[\|\Delta(\widetilde{X}_{i-1}^{\mathsf{c}})\|_2^2]\notag\\
&\qquad\qquad\qquad\qquad\qquad + \varrho^{-2}\mathbb{E}_{\widetilde{X}_{i-1}^{\mathsf{c}}\mymid \widetilde{X}_{i-1}^{\mathsf{s}}}[(({x}_{i}^{\mathsf{sc}}-\mu_0)^{\top}\Delta(\widetilde{X}_{i-1}^{\mathsf{c}}))^2]\bigg)  + O(\mathbb{E}_{\widetilde{X}_{i-1}^{\mathsf{c}}\mymid \widetilde{X}_{i-1}^{\mathsf{s}}}[\|\Delta(\widetilde{X}_{i-1}^{\mathsf{c}})\|_2^3]).
\end{align*}
These lead to
\begin{align*}
&\log \frac{\int p_{\widetilde{X}_i^{\mathsf{sc}}\mymid \widetilde{X}_{i-1}}({x}_{i}^{\mathsf{sc}}\mymid {x}_{i-1}) p_{\widetilde{X}_{i-1}^{\mathsf{c}}\mymid \widetilde{X}_{i-1}^{\mathsf{s}},Y,Z}({x}_{i-1}^{\mathsf{c}}\mymid {x}_{i-1}^{\mathsf{s}},y,z)\mathrm{d}{x}_{i-1}^{\mathsf{c}}}{\int p_{\widetilde{X}_i^{\mathsf{sc}}\mymid \widetilde{X}_{i-1}}({x}_{i}^{\mathsf{sc}}\mymid {x}_{i-1})p_{{X}_{t_{i-1}}^{\mathsf{c}}\mymid {X}_{t_{i-1}}^{\mathsf{s}}}({x}_{i-1}^{\mathsf{c}}\mymid {x}_{i-1}^{\mathsf{s}}) \mathrm{d}{x}_{i-1}^{\mathsf{c}}} \notag\\
&\quad = \log \frac{1 + \varrho^{-2}({x}_{i}^{\mathsf{sc}}-\mu_0)^{\top}\mathbb{E}[\Delta(\widetilde{X}_{i-1}^{\mathsf{c}})] - \varrho^{-2}\mathbb{E}[\|\Delta(\widetilde{X}_{i-1}^{\mathsf{c}})\|_2^2]+ \varrho^{-2}\mathbb{E}[(({x}_{i}^{\mathsf{sc}}-\mu_0)^{\top}\Delta(\widetilde{X}_{i-1}^{\mathsf{c}}))^2]}{1 + \varrho^{-2}({x}_{i}^{\mathsf{sc}}-\mu_0)^{\top}\mathbb{E}[\Delta(X_{t_{i-1}}^{\mathsf{c}})] - \varrho^{-2}\mathbb{E}[\|\Delta(X_{t_{i-1}}^{\mathsf{c}})\|_2^2]+ \varrho^{-2}\mathbb{E}[(({x}_{i}^{\mathsf{sc}}-\mu_0)^{\top}\Delta(X_{t_{i-1}}^{\mathsf{c}}))^2]} + O(\delta^3)\notag\\
&\quad \le \frac{\varrho^{-2}({x}_{i}^{\mathsf{sc}}-\mu_0)^{\top}\big(\mathbb{E}[\Delta(\widetilde{X}_{{i-1}}^{\mathsf{c}})-\Delta(X_{t_{i-1}}^{\mathsf{c}})]\big) - \varrho^{-2}\big(\mathbb{E}[\|\Delta(\widetilde{X}_{{i-1}}^{\mathsf{c}})\|_2^2-\|\Delta(X_{t_{i-1}}^{\mathsf{c}})\|_2^2]\big)}{1 + \varrho^{-2}({x}_{i}^{\mathsf{sc}}-\mu_0)^{\top}\mathbb{E}[\Delta(X_{t_{i-1}}^{\mathsf{c}})] - \varrho^{-2}\mathbb{E}[\|\Delta(X_{t_{i-1}}^{\mathsf{c}})\|_2^2]+ \varrho^{-2}\mathbb{E}[(({x}_{i}^{\mathsf{sc}}-\mu_0)^{\top}\Delta(X_{t_{i-1}}^{\mathsf{c}}))^2]}\notag\\
&\qquad + \frac{\varrho^{-2}\mathbb{E}[(({x}_{i}^{\mathsf{sc}}-\mu_0)^{\top}\Delta(\widetilde{X}_{{i-1}}^{\mathsf{c}}))^2-(({x}_{i}^{\mathsf{sc}}-\mu_0)^{\top}\Delta(X_{t_{i-1}}^{\mathsf{c}}))^2]}{1 + \varrho^{-2}({x}_{i}^{\mathsf{sc}}-\mu_0)^{\top}\mathbb{E}[\Delta(X_{t_{i-1}}^{\mathsf{c}})] - \varrho^{-2}\mathbb{E}[\|\Delta(X_{t_{i-1}}^{\mathsf{c}})\|_2^2]+ \varrho^{-2}\mathbb{E}[(({x}_{i}^{\mathsf{sc}}-\mu_0)^{\top}\Delta(X_{t_{i-1}}^{\mathsf{c}}))^2]}.
\end{align*}
Therefore, the term considered in \eqref{eq:lem-distri-sc-conds} is bounded by
\begin{align}\label{eq:proof-17}
&\int \log \frac{\int p_{\widetilde{X}_i^{\mathsf{sc}}\mymid \widetilde{X}_{i-1}}({x}_{i}^{\mathsf{sc}}\mymid {x}_{i-1}) p_{\widetilde{X}_{i-1}^{\mathsf{c}}\mymid \widetilde{X}_{i-1}^{\mathsf{s}},Y,Z}({x}_{i-1}^{\mathsf{c}}\mymid {x}_{i-1}^{\mathsf{s}},y,z) \mathrm{d}{x}_{i-1}^{\mathsf{c}}}{\int p_{\widetilde{X}_i^{\mathsf{sc}}\mymid \widetilde{X}_{i-1}}({x}_{i}^{\mathsf{sc}}\mymid {x}_{i-1})p_{{X}_{t_{i-1}}^{\mathsf{c}}\mymid {X}_{t_{i-1}}^{\mathsf{s}}}({x}_{i-1}^{\mathsf{c}}\mymid {x}_{i-1}^{\mathsf{s}})\mathrm{d}{x}_{i-1}^{\mathsf{c}}} p_{\widetilde{X}_{i}^{\mathsf{sc}}\mymid \widetilde{X}_{{i-1}}^{\mathsf{s}},Y,Z}({x}_i^{\mathsf{sc}}\mymid {x}_{i-1}^{\mathsf{s}},y,z)\mathrm{d} {x}_i^{\mathsf{sc}} \notag\\
&\lesssim \int \Big(\varrho^{-2}({x}_{i}^{\mathsf{sc}}-\mu_0)^{\top}\big(\mathbb{E}[\Delta(\widetilde{X}_{{i-1}}^{\mathsf{c}})-\Delta(X_{t_{i-1}}^{\mathsf{c}})]\big) - \varrho^{-2}\big(\mathbb{E}[\|\Delta(\widetilde{X}_{{i-1}}^{\mathsf{c}})\|_2^2-\|\Delta(X_{t_{i-1}}^{\mathsf{c}})\|_2^2]\big)
    \notag\\
&\qquad +\varrho^{-2}\mathbb{E}[(({x}_{i}^{\mathsf{sc}}-\mu_0)^{\top}\Delta(\widetilde{X}_{{i-1}}^{\mathsf{c}}))^2-(({x}_{i}^{\mathsf{sc}}-\mu_0)^{\top}\Delta(X_{t_{i-1}}^{\mathsf{c}}))^2]\Big)p_{\widetilde{X}_{i}^{\mathsf{sc}}\mymid \widetilde{X}_{{i-1}}^{\mathsf{s}}}({x}_i^{\mathsf{sc}}\mymid {x}_{i-1}^{\mathsf{s}})\mathrm{d} {x}_i^{\mathsf{sc}}\notag\\
&\le \varrho^{-2}\|\mathbb{E}[{x}_{i}^{\mathsf{sc}}-\mu_0]\|_2^{\top}\|\mathbb{E}[\Delta(X_{t_{i-1}}^{\mathsf{c}})-\Delta(\widetilde{X}_{{i-1}}^{\mathsf{c}})]\|_2 + \varrho^{-2}\big|\mathbb{E}[\|\Delta(X_{t_{i-1}}^{\mathsf{c}})\|_2^2-\|\Delta(\widetilde{X}_{{i-1}}^{\mathsf{c}})\|_2^2]\big|\notag\\
&\quad + \varrho^{-2}\mathbb{E}[\|{x}_{i}^{\mathsf{sc}}-\mu_0\|_2^2]\big\|\mathbb{E}[\Delta(X_{t_{i-1}}^{\mathsf{c}})\Delta(X_{t_{i-1}}^{\mathsf{c}})^{\top}-\Delta(\widetilde{X}_{i-1}^{\mathsf{c}})\Delta(\widetilde{X}_{i-1}^{\mathsf{c}})^{\top}]\big\|_{\mathsf{op}}.
\end{align}
According to the update rule of $\widetilde{X}_i^{\mathsf{sc}}$, the expectation of $\widetilde{X}_{i}^{\mathsf{sc}}-\mu_0$ conditioned on $\widetilde{X}_{i-1}^{\mathsf{s}} = {x}_{i-1}^{\mathsf{s}}$ is bounded by
\begin{align*}
\varrho^{-2}\|\mathbb{E}[\widetilde{x}_{i}^{\mathsf{sc}}-\mu_0]\|_2 &= \varrho^{-2}\|\mathbb{E}[\Delta(\widetilde{X}_{i-1}^{\mathsf{c}})\mymid \widetilde{X}_{i-1}^{\mathsf{s}} = {x}_{i-1}^{\mathsf{s}}] \|_2\notag\\
&=\frac{\alpha_{t_i}\sigma_{t_{i-1}}{\rm e}^{-\lambda_{t_{i-1}}}(1-{\rm e}^{-2\delta_{i}})}{\sigma_{t_i}^2(1-{\rm e}^{-2\delta_{i}})}
\|\mathbb{E}[\nabla_{{x}_{i-1}^{\mathsf{sc}}} \log p_{\lambda_{i-1}}(\widetilde{X}_{i-1}^{\mathsf{c}},{x}_{i-1}^{\mathsf{sc}})\mymid \widetilde{X}_{i-1}^{\mathsf{s}} = {x}_{i-1}^{\mathsf{s}}] \|_2
= O(1).
\end{align*}
Similarly, we have
\begin{align*}
\varrho^{-2}\mathbb{E}[\|{x}_{i}^{\mathsf{sc}}-\mu_0\|_2^2]
=\frac{\alpha_{t_i}^2\sigma_{t_{i-1}}^2{\rm e}^{-2\lambda_{t_{i-1}}}(1-{\rm e}^{-2\delta_{i}})^2}{\sigma_{t_i}^2(1-{\rm e}^{-2\delta_{i}})}
\mathbb{E}[\|\nabla_{{x}_{i-1}^{\mathsf{sc}}} \log p_{\lambda_{i-1}}(\widetilde{X}_{i-1}^{\mathsf{c}},{x}_{i-1}^{\mathsf{sc}})\|_2^2\mymid \widetilde{X}_{i-1}^{\mathsf{s}} = {x}_{i-1}^{\mathsf{s}}]
+d
=O(1).
\end{align*}
Morever, define 
$$
\overline{\Delta}({x}_{i-1}^{\mathsf{c}}) \coloneqq \delta_i^{-1}\Delta({x}_{i-1}^{\mathsf{c}}).
$$
We have $\mathbb{E}[\|\overline{\Delta}({X}_{t_{i-1}}^{\mathsf{c}})\|_2]<\infty$ and $\mathbb{E}[\|\overline{\Delta}(\widetilde{X}_{i-1}^{\mathsf{c}})\|_2]<\infty$.
Instituting into \eqref{eq:proof-17}, we have
\begin{align}\label{eq:proof-18}
&\int \log \frac{\int p_{\widetilde{X}_i^{\mathsf{sc}}\mymid \widetilde{X}_{i-1}}({x}_{i}^{\mathsf{sc}}\mymid {x}_{i-1}) p_{\widetilde{X}_{i-1}^{\mathsf{c}}\mymid \widetilde{X}_{i-1}^{\mathsf{s}},Y,Z}({x}_{i-1}^{\mathsf{c}}\mymid {x}_{i-1}^{\mathsf{s}},y,z)\mathrm{d}{x}_{i-1}^{\mathsf{c}}}{\int p_{\widetilde{X}_i^{\mathsf{sc}}\mymid \widetilde{X}_{i-1}}({x}_{i}^{\mathsf{sc}}\mymid {x}_{i-1})p_{{X}_{t_{i-1}}^{\mathsf{c}}\mymid {X}_{t_{i-1}}^{\mathsf{s}}}({x}_{i-1}^{\mathsf{c}}\mymid {x}_{i-1}^{\mathsf{s}}) \mathrm{d}{x}_{i-1}^{\mathsf{c}}} p_{\widetilde{X}_{i}^{\mathsf{sc}}\mymid \widetilde{X}_{{i-1}}^{\mathsf{s}},Y,Z}({x}_i^{\mathsf{sc}}\mymid {x}_{i-1}^{\mathsf{s}},y,z)\mathrm{d} {x}_i^{\mathsf{sc}} \notag\\
&\quad\lesssim \delta_i\|\mathbb{E}[\overline{\Delta}({X}_{t_{i-1}}^{\mathsf{c}})-\overline{\Delta}(\widetilde{X}_{{i-1}}^{\mathsf{c}})]\|_2 + \delta_i|\mathbb{E}[\|\overline{\Delta}(X_{t_{i-1}}^{\mathsf{c}})\|_2^2-\|\overline{\Delta}(\widetilde{X}_{i-1}^{\mathsf{c}})\|_2^2]|\notag\\
&\qquad  + \delta_i^2 \|\mathbb{E}[\overline{\Delta}({X}_{t_{i-1}}^{\mathsf{c}})\overline{\Delta}({X}_{t_{i-1}}^{\mathsf{c}})^{\top}-\overline{\Delta}(\widetilde{X}_{{i-1}}^{\mathsf{c}})\overline{\Delta}(\widetilde{X}_{{i-1}}^{\mathsf{c}})^{\top}]\|_{\mathsf{op}}.
\end{align}
Thus we have
\begin{align*}
&\quad \sum_{i=1}^k\mathbb{E}_{{x}_{i-1}^{\mathsf{s}}\sim \widetilde{X}_{i-1}^{\mathsf{s}}\mymid Y,Z}\bigg[\int \log \frac{\int p_{\widetilde{X}_i^{\mathsf{sc}}\mymid \widetilde{X}_{i-1}}({x}_{i}^{\mathsf{sc}}\mymid {x}_{i-1}) p_{\widetilde{X}_{i-1}^{\mathsf{c}}\mymid \widetilde{X}_{i-1}^{\mathsf{s}},Y,Z}({x}_{i-1}^{\mathsf{c}}\mymid {x}_{i-1}^{\mathsf{s}},y,z)\mathrm{d}{x}_{i-1}^{\mathsf{c}}}{\int p_{\widetilde{X}_i^{\mathsf{sc}}\mymid \widetilde{X}_{i-1}}({x}_{i}^{\mathsf{sc}}\mymid {x}_{i-1})p_{{X}_{t_{i-1}}^{\mathsf{c}}\mymid {X}_{t_{i-1}}^{\mathsf{s}}}({x}_{i-1}^{\mathsf{c}}\mymid {x}_{i-1}^{\mathsf{s}}) \mathrm{d}{x}_{i-1}^{\mathsf{c}}}\notag\\
&\qquad\qquad\qquad\qquad\qquad\cdot p_{\widetilde{X}_{i}^{\mathsf{sc}}\mymid \widetilde{X}_{{i-1}}^{\mathsf{s}}}({x}_i^{\mathsf{sc}}\mymid {x}_{i-1}^{\mathsf{s}},y,z)\mathrm{d} {x}_i^{\mathsf{sc}}\bigg]\notag\\
&\lesssim \max_i\mathbb{E}_{x_{i-1}^{\mathsf{sc}}\sim p_{\widetilde{X}_{i-1}^{\mathsf{sc}}\mymid Y,Z}}\big[\|\mathbb{E}[\overline{\Delta}({X}_{t_{i-1}}^{\mathsf{c}})-\overline{\Delta}(\widetilde{X}_{{i-1}}^{\mathsf{c}})\mymid \widetilde{X}_{i-1}^{\mathsf{sc}}=X_{t_{i-1}}^{\mathsf{sc}}=x_{i-1}^{\mathsf{sc}}]\|_2\big]\notag\\
&\quad + \max_i\mathbb{E}_{x_{i-1}^{\mathsf{sc}}\sim p_{\widetilde{X}_{i-1}^{\mathsf{sc}}\mymid Y,Z}}\big[|\mathbb{E}[\|\overline{\Delta}(X_{t_{i-1}}^{\mathsf{c}})\|_2^2-\|\overline{\Delta}(\widetilde{X}_{i-1}^{\mathsf{c}})\|_2^2\mymid \widetilde{X}_{i-1}^{\mathsf{sc}}=X_{t_{i-1}}^{\mathsf{sc}}=x_{i-1}^{\mathsf{sc}}]|\big]\notag\\
&\qquad  + \delta\max_i \mathbb{E}_{x_{i-1}^{\mathsf{sc}}\sim p_{\widetilde{X}_{i-1}^{\mathsf{sc}}\mymid Y,Z}}\big[\|\mathbb{E}[\overline{\Delta}({X}_{t_{i-1}}^{\mathsf{c}})\overline{\Delta}({X}_{t_{i-1}}^{\mathsf{c}})^{\top}-\overline{\Delta}(\widetilde{X}_{{i-1}}^{\mathsf{c}})\overline{\Delta}(\widetilde{X}_{{i-1}}^{\mathsf{c}})^{\top}\mymid \widetilde{X}_{i-1}^{\mathsf{sc}}=X_{t_{i-1}}^{\mathsf{sc}}=x_{i-1}^{\mathsf{sc}}]\|_{\mathsf{op}}\big] \to 0,
\end{align*}
as Lemma \ref{lem:distri-xc} implies that the following limit holds for any  function $f$ with finite expectation:
$$
\mathbb{E}_{x_{i-1}^{\mathsf{sc}}\sim p_{\widetilde{X}_{i-1}^{\mathsf{sc}}\mymid Y,Z}}\big[\mathbb{E}[f(X_{t_{i-1}}^{\mathsf{c}})-f(\widetilde{X}_{{i-1}}^{\mathsf{c}})\mymid \widetilde{X}_{i-1}^{\mathsf{sc}}=X_{t_{i-1}}^{\mathsf{sc}}=x_{i-1}^{\mathsf{sc}}]\big]\to 0.
$$ 

\subsection{Proof of Lemma \ref{lem:proof-boundet}}
\label{app:proof-lem-boundet}

We will bound the norm of $\|e_{t_i,j}\|_2$ for $j=1,2,3$ separately.
We start from $\|e_{t_i,1}\|_2$, which is controlled by $|\frac{\sigma_{t_{i+1}}}{\sigma_{t_i}}-1|\|\widehat{X}_{t_i}^{\mathsf{c}}\|_2$ according to its definition.
Note that $\sigma_{t_{i+1}}\le \sigma_{t_{i}}$ and $\alpha_{t_{i+1}}\ge \alpha_{t_i}$. Thus we have
\begin{align}\label{eq:proof-diff-sigma}
\big|\frac{\sigma_{t_{i+1}}}{\sigma_{t_i}}-1\big| = 1-\frac{\sigma_{t_{i+1}}}{\sigma_{t_i}} \le 1-\frac{\alpha_{t_i}}{\alpha_{t_{i+1}}}\frac{\sigma_{t_{i+1}}}{\sigma_{t_i}} = 1 - \frac{{\rm e}^{-\lambda_{t_{i+1}}}}{{\rm e}^{-\lambda_{t_i}}} = 1 - {\rm e}^{-\delta_{{i+1}}}\le \delta_{{i+1}}.
\end{align}
Thus we have
\begin{align}\label{eq:proof-lem-boundet-1}
\|e_{t_i,1}({x}_{t_i}^{\mathsf{c}})\|_2^2=\Big\|\left(\frac{\sigma_{t_{i+1}}}{\sigma_{t_{i}}}-1\right)\sqrt{1-\eta(1 - e^{-2\delta_{i+1}})}{x}_{t_{i}}^{\mathsf{c}}\Big\|_2^2 \le \delta_{{i+1}}^2 \|{x}_{t_i}^{\mathsf{c}}\|_2^2.
\end{align}

For the error term $\|e_{t_i,2}\|_2$, recalling that $x_0$ is bounded, which implies that $\mu_{t_i}^{\mathsf{c}}$ is also bounded by some constant $B$ as it is a conditional expectation of $x_0$.
Thus we have
\begin{align}\label{eq:proof-1}
\|e_{t_i,2}({x}_{t_i}^{\mathsf{c}})\|_2
&=\Big\|\big(\left(\alpha_{t_{i+1}}-\alpha_{t_i}\right) + \left(\alpha_{t_{i+1}}e^{-\delta_{i+1}}-\alpha_{t_i}\right)\sqrt{1-\eta(1 - e^{-2\delta_{i+1}})}\big)\mu_{t_{i}}^{\mathsf{c}}({x}_{t_i})\Big\|_2\notag\\
& \le B\left(\alpha_{t_{i+1}} - \alpha_{t_i} + |\alpha_{t_{i+1}}{\rm e}^{-\delta_{{i+1}}} - \alpha_{t_i}|\right) \le B\left(2\alpha_{t_{i+1}} - 2\alpha_{t_i} + \alpha_{t_{i+1}}(1-{\rm e}^{-\delta_{{i+1}}})\right).
\end{align}
Note that
\begin{align}\label{eq:proof-diff-alpha}
\alpha_{t_{i+1}} - \alpha_{t_i} = \alpha_{t_{i+1}}\left(1 - \frac{\alpha_{t_i}}{\alpha_{t_{i+1}}}\right) \le \alpha_{t_{i+1}}\left(1 - \frac{\sigma_{t_{i+1}}}{\sigma_{t_i}}\frac{\alpha_{t_i}}{\alpha_{t_{i+1}}}\right) = \alpha_{t_{i+1}}\left(1-{\rm e}^{-\delta_{{i+1}}}\right) \le \alpha_{t_{i+1}} \delta_{{i+1}}.
\end{align}
Inserting into \eqref{eq:proof-1}, we have
\begin{align}\label{eq:proof-lem-boundet-2}
\|e_{t_i,2}({x}_{t_i}^{\mathsf{c}})\|_2^2\le B^2\alpha_{t_{i+1}}^2\left(2\delta_{{i+1}} + \delta_{{i+1}}\right)^2 = O(\delta_{{i+1}}^2).
\end{align}

For the error term $\|e_{t_i,3}\|_2$, we have
\begin{align*}
\|e_{t_i,3}({x}_{t_i},{x}_{t_{i+1}}^{\mathsf{s}})\|_2 
&=\Big\|\big(1 - \sqrt{1-\eta(1 - e^{-2\delta_{i+1}})}\big)\big(\sigma_{t_{i}}\epsilon_{t_{i}}^{\mathsf{c}}({X}_{t_{i}} )-\sigma_{t_{i+1}}\epsilon_{t_{i+1}}^{\mathsf{c}}({x}_{t_{i+1}}^{\mathsf{est}} )\big)\Big\|_2\notag\\
&\overset{\text{(i)}}{\le} \sigma_{t_{i+1}}^{-2}\gamma_i\left\|\sigma_{t_{i}}\epsilon_{t_{i}}^{\mathsf{c}}-\sigma_{t_{i+1}}\epsilon_{t_{i+1}}^{\mathsf{c}}\right\|_2 \notag\\
&=\sigma_{t_{i+1}}^{-2}\gamma_i \Big\|\sigma_{t_i}^2\nabla \log p_{X_{t_i}^{\mathsf{c}},X_{t_i}^{\mathsf{sc}}|X_{t_i}^{\mathsf{s}}}({x}_{t_i}^{\mathsf{c}},{x}_{t_i}^{\mathsf{sc}}|\alpha_{t_i}\overline{Y}_{t_i} + \hat{\sigma}_{t_i}\overline{Z}_{t_i}) \notag\\
&\qquad \qquad - \sigma_{t_{i+1}}^2\nabla \log p_{X_{t_{i+1}}^{\mathsf{c}},X_{t_{i+1}}^{\mathsf{sc}}|x_{t_{i+1}}^{\mathsf{s}}}({x}_{t_i}^{\mathsf{c}},{x}_{t_{i+1}}^{\mathsf{sc}}|\alpha_{t_{i+1}}\overline{Y}_{t_i} + \hat{\sigma}_{t_{i+1}}\overline{Z}_{t_i})\Big\|,
\end{align*}
where ${x}_{t_{i+1}}^{\mathsf{est}}$ is defined by $V_{\mathcal{S}^{\mathsf{c}}}^{\top}{x}_{t_{i+1}}^{\mathsf{est}} = {x}_{t_{i}}^{\mathsf{c}}$ and $V_{\mathcal{S}}^{\top}{x}_{t_{i+1}}^{\mathsf{est}} = {x}_{t_{i+1}}^{\mathsf{s}}$, (i) arises from
\begin{align*}
1-\sqrt{1-\eta(1-{\rm e}^{-2\delta_{{i+1}}})} \le \eta(1-{\rm e}^{-2\delta_{{i+1}}}) = \sigma_{t_{i+1}}^{-2}\gamma_{i}.
\end{align*}
and the last equation comes from \eqref{eq:def-epsc}.
Combining with \eqref{eq:nablap-exp}, we have
\begin{align}\label{eq:proof-6}
&\quad\|e_{t_i,3}({x}_{t_i},{x}_{t_{i+1}}^{\mathsf{s}})\|_2 \notag\\
&\le\sigma_{t_{i+1}}^{-2}\gamma_i\Big\|\mathbb{E}_{X_{0}^{\mathsf{c}}}[{x}_{t_i}^{\mathsf{c}}-\alpha_{t_i}X_{0}^{\mathsf{c}}|X_{t_i}^{\mathsf{c}}={x}_{t_i}^{\mathsf{c}},X_{t_i}^{\mathsf{sc}}={x}_{t_i}^{\mathsf{sc}}, X_{t_i}^{\mathsf{s}}=\alpha_{t_i}\overline{Y}_{t_i} + \hat{\sigma}_{t_i}\overline{Z}_{t_i}] \notag\\
&\qquad \quad - \mathbb{E}_{X_{0}^{\mathsf{c}}}[{x}_{t_i}^{\mathsf{c}}-\alpha_{t_{i+1}}X_{0}^{\mathsf{c}}|X_{t_{i+1}}^{\mathsf{c}}={x}_{t_i}^{\mathsf{c}},X_{t_{i+1}}^{\mathsf{sc}}={x}_{t_{i+1}}^{\mathsf{sc}}, X_{t_{i+1}}^{\mathsf{s}}=\alpha_{t_{i+1}}\overline{Y}_{t_{i+1}} + \hat{\sigma}_{t_{i+1}}\overline{Z}_{t_{i+1}}]\Big\|_2\notag\\
&= \sigma_{t_{i+1}}^{-2}\gamma_i\Big\|\mathbb{E}_{X_{0}^{\mathsf{c}}}[\alpha_{t_i}X_{0}^{\mathsf{c}}|X_{t_i}^{\mathsf{c}}={x}_{t_i}^{\mathsf{c}},X_{t_i}^{\mathsf{sc}}={x}_{t_i}^{\mathsf{sc}}, X_{t_i}^{\mathsf{s}}=\alpha_{t_i}\overline{Y}_{t_i} + \hat{\sigma}_{t_i}\overline{Z}_{t_i}] \notag\\
&\qquad \quad - \mathbb{E}_{X_{0}^{\mathsf{c}}}[\alpha_{t_{i+1}}X_{0}^{\mathsf{c}}|X_{t_{i+1}}^{\mathsf{c}}={x}_{t_i}^{\mathsf{c}},X_{t_{i+1}}^{\mathsf{sc}}={x}_{t_{i+1}}^{\mathsf{sc}}, X_{t_{i+1}}^{\mathsf{s}}=\alpha_{t_{i+1}}\overline{Y}_{t_{i+1}} + \hat{\sigma}_{t_{i+1}}\overline{Z}_{t_{i+1}}]\Big\|_2\notag\\
&\le\sigma_{t_{i+1}}^{-2}\gamma_i (\alpha_{t_{i+1}}-\alpha_{t_{i}})\Big\|\mathbb{E}_{X_{0}^{\mathsf{c}}}[X_{0}^{\mathsf{c}}|X_{t_i}^{\mathsf{c}}={x}_{t_i}^{\mathsf{c}},X_{t_i}^{\mathsf{sc}}={x}_{t_i}^{\mathsf{sc}}, X_{t_i}^{\mathsf{s}}=\alpha_{t_i}\overline{Y}_{t_i} + \hat{\sigma}_{t_i}\overline{Z}_{t_i}]\Big\|_2 \notag\\
&\qquad \quad +\sigma_{t_{i+1}}^{-2}\gamma_i\alpha_{t_{i+1}}\Big\|\mathbb{E}_{X_{0}^{\mathsf{c}}}[X_{0}^{\mathsf{c}}|X_{t_i}^{\mathsf{c}}=x_{t_i}^{\mathsf{c}},X_{t_i}^{\mathsf{sc}}=x_{t_i}^{\mathsf{sc}}, X_{t_i}^{\mathsf{s}}=\alpha_{t_i}\overline{Y}_{t_i} + \hat{\sigma}_{t_i}\overline{Z}_{t_i}] \notag\\
&\qquad \quad - \mathbb{E}_{X_{0}^{\mathsf{c}}}[X_{0}^{\mathsf{c}}|X_{t_{i+1}}^{\mathsf{c}}=x_{t_i}^{\mathsf{c}},X_{t_{i+1}}^{\mathsf{sc}}=x_{t_{i+1}}^{\mathsf{sc}}, X_{t_{i+1}}^{\mathsf{s}}=\alpha_{t_{i+1}}\overline{Y}_{t_{i+1}} + \hat{\sigma}_{t_{i+1}}\overline{Z}_{t_{i+1}}]\Big\|_2.
\end{align}
Recall that $X_0$ is bounded. By using \eqref{eq:proof-diff-alpha}, the first term is bounded by
\begin{align}\label{eq:proof-5}
(\alpha_{t_{i+1}}-\alpha_{t_{i}})\Big\|\mathbb{E}_{X_{0}^{\mathsf{c}}}[X_{0}^{\mathsf{c}}|X_{t_i}^{\mathsf{c}}={x}_{t_i}^{\mathsf{c}},X_{t_i}^{\mathsf{sc}}={x}_{t_i}^{\mathsf{sc}}, X_{t_i}^{\mathsf{s}}=\alpha_{t_i}\overline{Y}_{t_i} + \hat{\sigma}_{t_i}\overline{Z}_{t_i}]\Big\|_2
= O(\delta_{{i+1}}).
\end{align}
We claim that the second term is bounded by the following inequality, whose proof is postponed to the end of this section.
\begin{align}\label{eq:proof-4}
&\Big\|\mathbb{E}_{X_{0}^{\mathsf{c}}}[X_{0}^{\mathsf{c}}|X_{t_i}^{\mathsf{c}}=x_{t_i}^{\mathsf{c}},X_{t_i}^{\mathsf{sc}}=x_{t_i}^{\mathsf{sc}}, X_{t_i}^{\mathsf{s}}=\alpha_{t_i}\overline{Y}_{t_i} + \hat{\sigma}_{t_i}\overline{Z}_{t_i}]\notag\\
&\quad - \mathbb{E}_{X_{0}^{\mathsf{c}}}[X_{0}^{\mathsf{c}}|X_{t_{i+1}}^{\mathsf{c}}=x_{t_i}^{\mathsf{c}},X_{t_{i+1}}^{\mathsf{sc}}=x_{t_{i+1}}^{\mathsf{sc}}, X_{t_{i+1}}^{\mathsf{st}}=\alpha_{t_{i+1}}\overline{Y}_{t_{i+1}} + \hat{\sigma}_{t_{i+1}}\overline{Z}_{t_{i+1}}]\Big\|_2\notag\\
&\qquad \lesssim \bigg(\|x_{t_i}^{\mathsf{c}}\|_2 +\|x_{t_i}^{\mathsf{sc}}\|_2 + \|Z\|_2+ 1\bigg)\delta_{{i+1}} + \|x_{t_{i+1}}^{\mathsf{sc}}-x_{t_i}^{\mathsf{sc}}\|_2.
\end{align}
Combining \eqref{eq:proof-4} and \eqref{eq:proof-5} together with \eqref{eq:proof-6}, we have
\begin{align}\label{eq:proof-lem-boundet-3}
\|e_{t_{i,3}}({x}_{t_i},{x}_{t_{i+1}}^{\mathsf{s}})\|_2\lesssim \bigg(\|x_{t_i}^{\mathsf{c}}\|_2 +\|x_{t_i}^{\mathsf{sc}}\|_2 + \|Z\|_2+ 1\bigg)\delta_{{i+1}}\gamma_i + \gamma_i\|x_{t_{i+1}}^{\mathsf{sc}}-x_{t_i}^{\mathsf{sc}}\|_2.
\end{align} 

For the last term in the right-hand-side of \eqref{eq:def-et}, we have
\begin{align}\label{eq:proof-lem-boundet-4}
&\quad\left\|\Bigg(1 - \sqrt{1-\eta(1 - e^{-2\delta_{i+1}})} - \frac{\eta(1 - e^{-2\delta_{i+1}})}{2}\Bigg)\sigma_{t_{i+1}}\epsilon_{t_{i+1}}^{\mathsf{c}}\right\|_2 \notag\\
&\le \Bigg|1 - \sqrt{1-\eta(1 - e^{-2\delta_{i+1}})} - \frac{\eta(1 - e^{-2\delta_{i+1}})}{2}\Bigg|\sigma_{t_{i+1}}\|\epsilon_{t_{i+1}}^{\mathsf{c}}\|_2\notag\\
&\lesssim \gamma_i^2\frac{\|\epsilon_{t_{i+1}}^{\mathsf{c}}\|_2}{\sigma_{t_{i+1}}^3} \le \frac{\gamma_i^2}{\sigma_{t_{i+1}}^4}(\|x_{t_i}^{\mathsf{c}}\|_2 + \alpha_{t_i}B).
\end{align}

Combining \eqref{eq:proof-lem-boundet-1}, \eqref{eq:proof-lem-boundet-2}, \eqref{eq:proof-lem-boundet-3}, and \eqref{eq:proof-lem-boundet-4}, we have
\begin{align*}
\|e_{t_i}(x_{t_i},x_{t_{i+1}}^{\mathsf{s}})\|_2^2 &\lesssim \bigg(\|x_{t_i}^{\mathsf{c}}\|_2 +\|x_{t_i}^{\mathsf{sc}}\|_2 + \|Z\|_2+ 1\bigg)^2\delta_{{i+1}}^2+ (\|x_{t_i}^{\mathsf{c}}\|_2^2+1)\gamma_i^4+\gamma_i^2\|x_{t_{i+1}}^{\mathsf{sc}}-x_{t_i}^{\mathsf{sc}}\|_2^2.
\end{align*}
It suffices to prove that
\begin{align*}
\mathbb{E}_{\widehat{X}_{t_{i+1}}^{\mathsf{sc}}|\widehat{X}_{t_i}=x_{t_i}}[\|\widehat{X}_{t_{i+1}}^{\mathsf{sc}}-\widehat{X}_{t_i}^{\mathsf{sc}}\|_2^2] =O(\delta_{{i+1}}^2)(\|x_{t_i}\|_2^2+1) + O(\delta_{i}).
\end{align*}
According to the update rule of $\widehat{X}_{t_{i+1}}^{\mathsf{sc}}$, we have
\begin{align*}
\mathbb{E}_{\widehat{X}_{t_{i+1}}^{\mathsf{sc}}|\widehat{X}_{t_i}=x_{t_i}}[\|\widehat{X}_{t_{i+1}}^{\mathsf{sc}}-\widehat{X}_{t_i}^{\mathsf{sc}}\|_2^2] &\lesssim \bigg(\frac{\sigma_{t_{i+1}}}{\sigma_{t_i}}{\rm e}^{-\delta_{i+1}}-1\bigg)^2\|x_{t_i}^{\mathsf{sc}}\|_2^2 + \delta_{{i+1}}^2 + \delta_{i+1}\mathbb{E}[\|z_{t_i}^{\mathsf{sc}}\|_2^2]\notag\\
& = \delta_{{i+1}}^2(\|x_{t_i}^{\mathsf{sc}}\|_2^2+1).
\end{align*}
and complete the proof.


\paragraph{Proof of \eqref{eq:proof-4}.}
For the second term, we introduce the following notation $q_{\lambda_{i}}(x_0|x')$ to simplify the expression:
\begin{align*}
q_{\lambda_{i}}(x_0|x') &\coloneqq p_{X_{0}|X_{t_i}}(x_0|x') = 
 \frac{\phi(x'|\alpha_{t_i}x_0,\sigma_{t_i})p_{X_0}(x_0)}{p_{X_{t_i}}(x')},
\end{align*}
where $\phi(\cdot|\mu,\sigma)$ denotes the Gaussian density functon with mean vector $\mu$ and covariance matrix $\sigma^2I$.
Armed with this notation, we can control the second term as
\begin{align}\label{eq:proof-2}
&\Big\|\mathbb{E}_{X_{0}^{\mathsf{c}}}[X_{0}^{\mathsf{c}}|X_{t_i}^{\mathsf{c}}=x_{t_i}^{\mathsf{c}},X_{t_i}^{\mathsf{sc}}=x_{t_i}^{\mathsf{sc}}, X_{t_i}^{\mathsf{s}}=\alpha_{t_i}\overline{Y}_{t_i} + \hat{\sigma}_{t_i}\overline{Z}_{t_i}]\notag\\
&\quad  - \mathbb{E}_{X_{0}^{\mathsf{c}}}[X_{0}^{\mathsf{c}}|X_{t_{i+1}}^{\mathsf{c}}=x_{t_i}^{\mathsf{c}},X_{t_{i+1}}^{\mathsf{sc}}=x_{t_{i+1}}^{\mathsf{sc}}, X_{t_{i+1}}^{\mathsf{s}}=\alpha_{t_{i+1}}\overline{Y}_{t_{i+1}} + \hat{\sigma}_{t_{i+1}}\overline{Z}_{t_{i+1}}]\Big\|_2\notag\\
&\quad=\left\|\int x_0 \left(q_{\lambda_{i}}(x_0|\tilde{x}_{t_i}) - q_{\lambda_{i+1}}(x_0|\tilde{x}_{t_{i+1}})\right)\mathrm{d} x_0\right\|_2 \le B\int |q_{\lambda_{i}}(x_0|\tilde{x}_{t_i}) - q_{\lambda_{i+1}}(x_0|\tilde{x}_{t_{i+1}})|\mathrm{d} x_0,
\end{align}
where $\tilde{x}_{t_i}\coloneqq [x_{t_i}^{\mathsf{c}},x_{t_i}^{\mathsf{sc}},\alpha_{t_{i}}\overline{Y}_{t_{i}}+\hat{\sigma}_{t_{i}}\overline{Z}_{t_{i}}]$ and $\tilde{x}_{t_{i+1}}\coloneqq [x_{t_i}^{\mathsf{c}},x_{t_{i+1}}^{\mathsf{sc}},\alpha_{t_{i+1}}\overline{Y}_{t_{i+1}}+\hat{\sigma}_{t_{i+1}}\overline{Z}_{t_{i+1}}]$.
According to the definition of $q_{\lambda_i}$, we have
\begin{align}\label{eq:proof-7}
\frac{1}{q_{\lambda_i}(x|x')}\left|\frac{\partial q_{\lambda_i}(x|x')}{\partial \alpha_{t_i}}\right| &= \frac{\partial \log q_{\lambda_i}(x|x')}{\partial \alpha_{t_i}} = \frac{\partial \log \phi(x'|\alpha_{t_i}x,\sigma_{t_i})}{\partial \alpha_{t_i}} - \frac{\partial \log p_{X_{t_i}}(x')}{\partial \alpha_{t_i}}\notag\\
&=-\frac12\frac{\partial}{\partial \alpha_{t_i}}\frac{1}{\sigma_{t_i}^2}\|x'-\alpha_{t_i}x\|_2^2 - \frac{\partial \log p_{X_{t_i}}(x')}{\partial \alpha_{t_i}}.
\end{align}
Moreover, we have
\begin{align*}
\left|\frac{\partial \log p_{X_{t_i}}(x')}{\partial \alpha_{t_i}}\right|
&= \frac{1}{2\sigma_{t_i}^2}\left|\int \frac{\partial}{\partial \alpha_{t_i}}\|x'-\alpha_{t_i}x\|_2^2\frac{\phi(x'|\alpha_{t_i}x,\sigma_{t_i})p_{X_0}(x)}{p_{X_{t_i}}(x')} \mathrm{d} x\right|\notag\\
&= \frac{1}{\sigma_{t_i}^2}\left|\int (x'-\alpha_{t_i}x)^{\top}x \frac{\phi(x'|\alpha_{t_i}x,\sigma_{t_i})p_{X_0}(x)}{p_{X_{t_i}}(x')} \mathrm{d} x\right|\notag\\
&= \frac{1}{\sigma_{t_i}^2}\left|\int (x'-\alpha_{t_i}x)^{\top}x q_{\lambda_i}(x|x') \mathrm{d} x\right|\notag\\
&\lesssim O(\|x'\|_2+1).
\end{align*}
Substituting into \eqref{eq:proof-7}, we have
\begin{align*}
\frac{1}{q_{\lambda_i}(x|x')}\left|\frac{\partial q_{\lambda_i}(x|x')}{\partial \alpha_{t_i}}\right| 
&\lesssim O(\|x'\|_2 + 1).
\end{align*}

Furthermore, we have the following inequality in a similar way
\begin{align}\label{eq:proof-8}
\frac{1}{q_{\lambda_i}(x|x')}\frac{\partial q_{\lambda_i}(x|x')}{\partial \sigma_{t_i}} 
&=-\frac12\frac{\partial}{\partial \sigma_{t_i}}\frac{1}{\sigma_{t_i}^2}\|x'-\alpha_{t_i}x\|_2^2 -\frac{d}{\sigma_{t_i}}- \frac{\partial \log p_{X_{t_i}}(x')}{\partial \sigma_{t_i}} \notag\\
&\overset{\text{(i)}}{=} \sigma_{t_i}^{-3}\|x'-\alpha_{t_i}x\|_2^2-\int \sigma_{t_i}^{-3}\|x'-\alpha_{t_i}x\|_2^2  q_{\lambda_i}(x|x') \mathrm{d} x,
\end{align}
where (i) comes from
\begin{align*}
\frac{\partial \log p_{X_{t_i}}(x')}{\partial \sigma_{t_i}}
&= \int \sigma_{t_i}^{-3}\|x'-\alpha_{t_i}x\|_2^2  q_{\lambda_i}(x|x') \mathrm{d} x - \frac{d}{\sigma_{t_i}}.
\end{align*}
By using the bounded property of $X_0$ again, we have
\begin{align*}
\left|\frac{1}{q_{\lambda_i}(x|x')}\frac{\partial q_{\lambda_i}(x|x')}{\partial \sigma_{t_i}} \right|
=O(\|x'\|_2+1).
\end{align*}
In addition, we have
\begin{align*}
\nabla_{x'} \log q_{\lambda_i}(x|x') &= -\nabla_{x'} \frac{1}{2\sigma_{t_i}^2}\|x'-\alpha_{t_i}x\|_2^2 -  \nabla_{x'} \log p_{X_{t_i}}(x') \notag\\
&= -\frac{1}{\sigma_{t_i}^2}(x'-\alpha_{t_i}x) + \frac{1}{\sigma_{t_i}^2}\mathbb{E}[x'-\alpha_{t_i}X_0|X_{t_i}=x'],\notag\\
\|\nabla_{x'} \log q_{\lambda_i}(x|x')\|_2&=O(\|x'\|_2).
\end{align*}
Moreover, we have
\begin{align*}
\|\tilde{x}_{t_{i+1}} - \tilde{x}_{t_{i}}\|_2 \le (\alpha_{t_{i+1}} - \alpha_{t_i})\|\overline{y}_{t_i}\|_2 + (\sigma_{t_{i+1}} - \sigma_{t_i})\|\overline{Z}_{t_i}\|_2 + \|x_{t_{i+1}}^{\mathsf{sc}}-x_{t_i}^{\mathsf{sc}}\|_2.
\end{align*}
Combining these together, we have
\begin{align*}
&\quad|q_{\lambda_{i}}(x|\tilde{x}_{t_i}) - q_{\lambda_{i+1}}(x|\tilde{x}_{t_{i+1}})|\notag\\
 &\lesssim q_{\lambda_{i}}(x|\tilde{x}_{t_i})(\|x_{t_i}^{\mathsf{c}}\|_2 +\|x_{t_i}^{\mathsf{sc}}\|_2 + \|Z\|_2+ 1) (|\alpha_{t_i}-\alpha_{t_{i+1}}| + |\sigma_{t_i} -\sigma_{t_{i+1}}|) \notag\\
&\quad +q_{\lambda_{i}}(x|\tilde{x}_{t_i})\|x_{t_{i+1}}^{\mathsf{sc}}-x_{t_i}^{\mathsf{sc}}\|_2+ O((|\alpha_{t_i}-\alpha_{t_{i+1}}| + |\sigma_{t_i} -\sigma_{t_{i+1}}|)^2) + O(\|x_{t_{i+1}}^{\mathsf{sc}}-x_{t_i}^{\mathsf{sc}}\|_2^2)
\end{align*}
According to \eqref{eq:proof-diff-alpha} and \eqref{eq:proof-diff-sigma}, we have
$|\alpha_{t_i}-\alpha_{t_{i+1}}|=O(\delta_{{i+1}})$ and $|\sigma_{t_i} -\sigma_{t_{i+1}}|=O(\delta_{{i+1}})$, and thus
\begin{align}\label{eq:proof-3}
|q_{\lambda_{i}}(x|\tilde{x}_{t_i}) - q_{\lambda_{i+1}}(x|\tilde{x}_{t_{i+1}})| &\lesssim q_{\lambda_{i}}(x|\tilde{x}_{t_i})\bigg(\|x_{t_i}^{\mathsf{c}}\|_2 +\|x_{t_i}^{\mathsf{sc}}\|_2 + \|Z\|_2+ 1\bigg)\delta_{{i+1}} \notag\\
&\quad +q_{\lambda_{i}}(x|\tilde{x}_{t_i})\|x_{t_{i+1}}^{\mathsf{sc}}-x_{t_i}^{\mathsf{sc}}\|_2 + O(\delta_{{i+1}}^2) + O(\|x_{t_{i+1}}^{\mathsf{sc}}-x_{t_i}^{\mathsf{sc}}\|_2^2).
\end{align}
Substituting \eqref{eq:proof-3} into \eqref{eq:proof-2}, we complete the proof.

\section{Extension to general $A$}
\label{subsec:extension}

For $V_{\mathcal{S}^{\mathsf{c}}}^{\top} x_{t_i}, V_{\Stci{t_i}}^{\top} x_{t_i}$, we aim to approximately generate them according to the following conditional distribution:
\begin{align*}
V_{\mathcal{S}^{\mathsf{c}}}^{\top} \widehat{X}_{t_i}, V_{\Stci{t_i}}^{\top} \widehat{X}_{t_i} \mymid V_{\Sti{t_i}}^{\top} \widehat{X}_{t_i}, V_{\Stci{t_i}}^{\top} \widehat{X}_{\tau}\sim p_{V_{\mathcal{S}^{\mathsf{c}}}^{\top} X_{t_i}, V_{\Stci{t_i}}^{\top} X_{t_i} \mymid V_{\Sti{t_i}}^{\top} X_{t_i}, V_{\Stci{t_i}}^{\top} X_{\tau}}
\end{align*}
Towards this, we need to calculate the score function, i.e., the log gradient of the above conditional density with respect to $V_{\mathcal{S}^{\mathsf{c}}}^{\top} x_{t_i}$ and $V_{\Stci{t_i}}^{\top} x_{t_i}$.
In particular, the gradient with respect to $V_{\mathcal{S}^{\mathsf{c}}}^{\top} x_{t_i}$ is given by
\begin{align*}
&\nabla_{V_{\mathcal{S}^{\mathsf{c}}}^{\top} x_{t_i}}\log p_{V_{\mathcal{S}^{\mathsf{c}}}^{\top} X_{t_i}, V_{\Stci{t_i}}^{\top} X_{t_i} \mymid V_{\Sti{t_i}}^{\top} X_{t_i}, V_{\Stci{t_i}}^{\top} X_{\tau}}(V_{\mathcal{S}^{\mathsf{c}}}^{\top} x_{t_i}, V_{\Stci{t_i}}^{\top} x_{t_i} \mymid V_{\Sti{t_i}}^{\top} x_{t_i}, V_{\Stci{t_i}}^{\top} x_{\tau})\notag\\
&\quad =\nabla_{V_{\mathcal{S}^{\mathsf{c}}}^{\top} x_{t_i}}\log p_{V_{\mathcal{S}^{\mathsf{c}}}^{\top} X_{t_i}, V_{\Stci{t_i}}^{\top} X_{t_i}, V_{\Sti{t_i}}^{\top} X_{t_i}}(V_{\mathcal{S}^{\mathsf{c}}}^{\top} x_{t_i}, V_{\Stci{t_i}}^{\top} x_{t_i}, V_{\Sti{t_i}}^{\top} x_{t_i}) = -\frac{1}{\sigma_{t_i}}V_{\mathcal{S}^{\mathsf{c}}}^{\top}\epsilon_{t_i}(x_{t_i}),
\end{align*}
which is the same as previously used score function defined in \eqref{eq:conditional-data-predictor}.
the gradient with respect to $V_{\Stci{t_i}}^{\top} x_{t_i}$ is given by
\begin{align*}
&\nabla_{V_{\Stci{t_i}}^{\top} x_{t_i}}\log p_{V_{\mathcal{S}^{\mathsf{c}}}^{\top} X_{t_i}, V_{\Stci{t_i}}^{\top} X_{t_i} \mymid V_{\Sti{t_i}}^{\top} X_{t_i}, V_{\Stci{t_i}}^{\top} X_{\tau}}(V_{\mathcal{S}^{\mathsf{c}}}^{\top} x_{t_i}, V_{\Stci{t_i}}^{\top} x_{t_i} \mymid V_{\Sti{t_i}}^{\top} x_{t_i}, V_{\Stci{t_i}}^{\top} x_{\tau})\notag\\
&\quad =\nabla_{V_{\Stci{t_i}}^{\top} x_{t_i}}\log p_{V_{\mathcal{S}^{\mathsf{c}}}^{\top} X_{t_i}, V_{\Stci{t_i}}^{\top} X_{t_i}, V_{\Sti{t_i}}^{\top} X_{t_i}}(V_{\mathcal{S}^{\mathsf{c}}}^{\top} x_{t_i}, V_{\Stci{t_i}}^{\top} x_{t_i}, V_{\Sti{t_i}}^{\top} x_{t_i}) \notag\\
&\qquad + \nabla_{V_{\Stci{t_i}}^{\top} x_{t_{i}}} \log p_{V_{\Stci{t_i}}^{\top} X_{\tau}\mymid V_{\Stci{t_i}}^{\top} X_{t_{i}}}(V_{\Sti{t_i}}^{\top} x_{\tau},V_{\Sti{t_i}}^{\top} x_{t_i})\notag\\
&\quad =-\frac{1}{\sigma_{t_i}}V_{\Stci{t_i}}^{\top}\epsilon_{t_i}(x_{t_i})
+\frac{1}{\sigma_{t_i}^2({\rm e}^{2(\lambda_{t_i}-\lambda_{\tau})}-1)}\left(\frac{\alpha_{t_i}}{\alpha_{\tau}} V_{\Stci{t_i}}^{\top} x_{\tau}- V_{\Stci{t_i}}^{\top} x_{t_i}\right).
\end{align*}
Compared to previously used score function, an additional term is introduced to capture the condition information between $V_{\Stci{t_i}}^{\top}\widehat{X}_{t_i}$ and $V_{\Stci{t_i}}^{\top}\widehat{X}_{\tau}$.
Inserting this score function into the DDIM update rule yields the update formula in \eqref{eq:ddim-prior-dominated-extension}.

\subsection{Proof of Theorem \ref{cor:main}}
\label{sec:proof-cor:main}

Akin to the proof of Theorem \ref{thm:main}, we first construct the following auxiliary sequence:
\begin{subequations}\label{eq:def-auxiliary-sequence-extension}
\begin{align}
\overline{X}_{k+1}^{\mathsf{st}} &=\alpha_{t_{k+1}}\overline{Y}_{t_{k+1}} + \widehat{\sigma}_{t_{k+1}}\overline{Z}_{t_{k+1}},\quad
\overline{X}_{k+1}^{\mathsf{sc}},\overline{X}_{k+1}^{\mathsf{c}} \sim p_{X_{k+1}^{\mathsf{sc}},X_{k+1}^{\mathsf{c}}\mid X_{k+1}^{\mathsf{st}},X_{\tau_s}^{\mathsf{sc},t_{k+1}}}(\cdot,\cdot\mid \overline{X}_{k+1}^{\mathsf{st}},\overline{X}_{k_{\tau_s}}^{\mathsf{sc},t_{k+1}}),\label{eq:def-auxiliary-sequence-bar-extension}\\
    \widetilde{X}_{k+1}^{\mathsf{st}} &=\alpha_{t_{k+1}}\overline{Y}_{t_{k+1}} + \widehat{\sigma}_{t_{k+1}}\overline{Z}_{t_{k+1}},\\
\widetilde{X}_{k+1}^{\mathsf{sc}} &= \widetilde{X}_{k}^{\mathsf{sc}} +  \gamma_k \nabla_{\widetilde{X}_k^{\mathsf{sc}}} \log q_{\lambda_{k+1}}(\widetilde{X}_{k}^{\mathsf{c}},\widetilde{X}_{k}^{\mathsf{sc}},\widetilde{X}_{k+1}^{\mathsf{st}}) + \sqrt{2\gamma_k} Z_k^{\mathsf{sc}},\label{eq:def-auxiliary-sequence-extension-2}\\
\widetilde{X}_{k+1}^{\mathsf{c}} &= \widetilde{X}_{k}^{\mathsf{c}} +  \gamma_k \nabla_{\widetilde{X}_k^{\mathsf{c}}} \log q_{\lambda_{k+1}}(\widetilde{X}_{k}^{\mathsf{c}},\widetilde{X}_{k}^{\mathsf{sc}},\widetilde{X}_{k+1}^{\mathsf{st}}) + \sqrt{2\gamma_k} Z_k^{\mathsf{c}},\label{eq:def-auxiliary-sequence-extension-3}
\end{align}
\end{subequations}
where $q_{\lambda_k}(x_{k}^{\mathsf{c}},x_{k}^{\mathsf{sc}},x_{k+1}^{\mathsf{st}})$ is defined as 
\begin{align*}
    q_{\lambda_{k+1}}(x_{k}^{\mathsf{c}},x_{k}^{\mathsf{sc}},x_{k}^{\mathsf{st}})\coloneqq p_{X_{t_{k+1}}^{\mathsf{sc}},X_{t_{k+1}}^{\mathsf{c}}\mymid X_{t_{k+1}}^{\mathsf{st}},X_{\tau_s}^{\mathsf{sc},t_{k+1}}}(x_{k}^{\mathsf{c}},x_{k}^{\mathsf{sc}}\mymid x_{k+1}^{\mathsf{st}},x_{k_{\tau_s}}^{\mathsf{sc},t_{k+1}}),
\end{align*}
and $X_{\tau_s}^{\mathsf{sc},t_{k+1}} = \{v_s^{\top}X_{\tau_s}\}_{s\in\mathcal{S}_{t_{k+1}}^{\mathsf{c}}}$, $\overline{X}_{k_{\tau_s}}^{\mathsf{sc},t_{k+1}} = \{v_s^{\top}\overline{X}_{k_{\tau_s}}\}_{s\in\mathcal{S}_{t_{k+1}}^{\mathsf{c}}}$
By a similar argument with Lemma \ref{lem:convergence-auxiliary-sequence}, we have
\begin{align}
\mathsf{KL}(p_{\widetilde{X}_{0:M}\mymid Y,Z}\Vert p_{\widehat{X}_{t_{0:M}}\mymid Y,Z})\to 0,\qquad \mathrm{as}\quad \eta^3\delta^2\to 0, \quad \mathrm{and}\quad \eta\to\infty.
\end{align}
Moreover, the sequence defined in \eqref{eq:def-auxiliary-sequence-bar} satisfies $p_{\overline{X}_k \mymid V_{\mathcal{S}_{t_k}}^{\top}\overline{X}_{k}, v_{s}^{\top}\overline{X}_{k_{\tau_s}}} = p_{X_{t_k} \mymid V_{\mathcal{S}_{t_k}}^{\top}X_{t_k}, v_{s}^{\top}X_{\tau_s}}$.
By a similar argument as in \eqref{eq:proof-thm-1}, it suffices to prove that for any $y$, $z$,
\begin{align*}
\mathsf{TV}\big(p_{\widetilde{X}_k\mymid Y,Z}(\cdot\mymid y,z),p_{\overline{X}_{k}\mymid Y,Z}(\cdot\mymid y,z)\big)\to 0.
\end{align*}
By decomposition, we have
\begin{align*}
\mathsf{TV}\big(p_{\widetilde{X}_k\mymid Y,Z}(\cdot\mymid y,z),p_{\overline{X}_{k}\mymid Y,Z}(\cdot\mymid y,z)\big) &\le \mathsf{TV}\big(p_{\widetilde{X}_k^{\mathsf{st}},\widetilde{X}_{k_{\tau_s}}^{\mathsf{sc},t_{k}}\mymid Y,Z}(\cdot\mymid y,z),p_{\overline{X}_{k}^{\mathsf{st}},\overline{X}_{k_{\tau_s}}^{\mathsf{sc},t_{k}}\mymid Y,Z}(\cdot\mymid y,z)\big) \notag\\
&\qquad + \mathsf{TV}\big(p_{\widetilde{X}_k\mymid \widetilde{X}_{k}^{\mathsf{st}},\widetilde{X}_{k_{\tau_s}}^{\mathsf{sc},t_{k}},Y,Z}(\cdot\mymid x_k^{\mathsf{st}},x_{\tau},y,z),p_{\overline{X}_{k}\mymid \overline{X}_{k}^{\mathsf{st}},\overline{X}_{k_{\tau_s}}^{\mathsf{sc},t_k},Y,Z}(\cdot\mymid x_k^{\mathsf{st}},x_{\tau},y,z)\big)\notag\\
&=\mathsf{TV}\big(p_{\widetilde{X}_k\mymid \widetilde{X}_{k}^{\mathsf{st}},\widetilde{X}_{k_{\tau_s}}^{\mathsf{sc},t_{k}},Y,Z}(\cdot\mymid x_k^{\mathsf{st}},x_{\tau},y,z),p_{{X}_{t_k}\mymid {X}_{t_k}^{\mathsf{st}},{X}_{\tau_s}^{\mathsf{sc},t_{k}}}(\cdot\mymid x_k^{\mathsf{st}},x_{\tau})\big),
\end{align*}
where $x_k^\mathsf{st} = \alpha_{t_k}y + \widehat{\sigma}_{t_k}z$, and the last equation uses $p_{\overline{X}_{k}\mymid \overline{X}_{k}^{\mathsf{st}},\overline{X}_{k_{\tau_s}}^{\mathsf{sc},t_k},Y,Z}(\cdot\mymid x_k^{\mathsf{st}},x_{\tau},y,z) = p_{{X}_{t_k}\mymid {X}_{t_k}^{\mathsf{st}},{X}_{\tau_s}^{\mathsf{sc},t_{k}}}(\cdot\mymid x_k^{\mathsf{st}},x_{\tau})$.
The remaining proof is similar to the proof of Lemma \ref{lem:distri-xc}.
Specificially, we define a new random variable $\widetilde{X}_k^{\mathsf{stc}} \coloneqq  [\widetilde{X}_k^{\mathsf{sc}},\widetilde{X}_k^{\mathsf{c}}]$ and ${X}_{t_k}^{\mathsf{stc}} \coloneqq  [{X}_{t_k}^{\mathsf{sc}},{X}_{t_k}^{\mathsf{c}}]$.
Then the update rule in \eqref{eq:def-auxiliary-sequence-extension-2} - \eqref{eq:def-auxiliary-sequence-extension-3} can be integrated as 
\begin{align*}
\widetilde{X}_{k+1}^{\mathsf{stc}} &= \widetilde{X}_{k}^{\mathsf{stc}} +  \gamma_k \nabla_{\widetilde{X}_k^{\mathsf{stc}}} \log q_{\lambda_{k+1}}(\widetilde{X}_{k}^{\mathsf{c}},\widetilde{X}_{k}^{\mathsf{sc}},\widetilde{X}_{k+1}^{\mathsf{st}}) + \sqrt{2\gamma_k} [Z_k^{\mathsf{sc}},Z_k^{\mathsf{c}}].
\end{align*}
For convenience, denote distributions
$$
p_k^{\mathsf{stc}}({x}_{t_k}^{\mathsf{stc}})\coloneqq p_{{X}_{t_k}^{\mathsf{stc}}\mymid {X}_{t_k}^{\mathsf{st}}}({x}_{t_k}^{\mathsf{stc}}\mymid {x}_{t_k}^{\mathsf{st}}),\qquad \psi_k({x}_{t_k}^{\mathsf{stc}})\coloneqq p_{\widetilde{X}_{k}^{\mathsf{stc}}\mymid \widetilde{X}_{k}^{\mathsf{st}},\widetilde{X}_{k_{\tau_s}}^{\mathsf{sc},t_{k}},Y,Z}({x}_{t_k}^{\mathsf{stc}}\mymid x_k^{\mathsf{st}},x_{\tau},y,z)
$$
and denote $P_{\gamma_k}$ as the transition probability of the following SDE from $t$ to $t+\gamma_k$:
\begin{align}
\mathrm{d} X(t) = \nabla \log p_{k+1}^{\mathsf{stc}}(X(t)) \mathrm{d} t + \sqrt{2} \mathrm{d} W_t,
\end{align}
where $W_t$ denotes the Brownian motion.
According to the analysis of Langevin, we have
\begin{align*}
\mathsf{KL}(\psi_{k+1}\Vert p_{k+1}^{\mathsf{stc}}) &\le \mathsf{KL}(P_{\gamma_k}\psi_{k}\Vert p_{k+1}^{\mathsf{stc}}) + \mathsf{KL}(\psi_{k+1}\Vert p_{k+1}^{\mathsf{stc}}) - \mathsf{KL}(P_{\gamma_k}\psi_{k}\Vert p_{k+1}^{\mathsf{stc}})\notag\\
&\le {\rm e}^{-2\kappa\gamma_k}\mathsf{KL}(\psi_{k}\Vert p_{k}^{\mathsf{stc}}) + {\rm e}^{-2\kappa\gamma_k}\left(\mathsf{KL}(\psi_{k}\Vert p_{k+1}^{\mathsf{stc}})-\mathsf{KL}(\psi_{k}\Vert p_{k}^{\mathsf{stc}}) \right)\notag\\
&\quad +\mathsf{KL}(\psi_{k+1}\Vert p_{k+1}^{\mathsf{stc}}) - \mathsf{KL}(P_{\gamma_k}\psi_{k}\Vert p_{k+1}^{\mathsf{stc}})
\end{align*}
where $\kappa<\infty$ is a constant.
We claim that 
\begin{align}\label{eq:proof-cor-main-1}
    \mathsf{KL}(\psi_{k}\Vert p_{k+1}^{\mathsf{stc}})-\mathsf{KL}(\psi_{k}\Vert p_{k}^{\mathsf{stc}})  = O(\delta_{k+1}).
\end{align}
By similar argument as in \eqref{eq:proof-11} - \eqref{eq:proof-15}, we have
\begin{align*}
\mathsf{KL}(\psi_{k+1}\Vert p_{k+1}^{\mathsf{stc}}) - \mathsf{KL}(P_{\gamma_k}\psi_{k}\Vert p_{k+1}^{\mathsf{stc}}) = O(\gamma_k^2).
\end{align*}
Therefore, we have
\begin{align*}
\mathsf{KL}(\psi_{k+1}\Vert p_{k+1}^{\mathsf{stc}}) & \le {\rm e}^{-2\kappa\gamma_k}\mathsf{KL}(\psi_{k}\Vert p_{k}^{\mathsf{stc}}) + O(\delta_{k+1} + \gamma_k^2)
&\lesssim \frac{1}{\eta} + \max_k \gamma_k \to 0,\qquad \mathrm{as}\quad \eta\to\infty, \delta\eta \to 0.
\end{align*}
Finally applying Pinsker's inequality, we complete the proof.

\paragraph{Proof of \eqref{eq:proof-cor-main-1}.}
According to the definition of KL divergence, we have
\begin{align*}
\mathsf{KL}(\psi_k\Vert p_{k+1}^{\mathsf{stc}}) - \mathsf{KL}(\psi_k\Vert p_{k}^{\mathsf{stc}}) = \int \log \frac{p_{k}^{\mathsf{stc}}(x)}{p_{k+1}^{\mathsf{stc}}(x)}\psi_{k+1}(x)\mathrm{d} x.
\end{align*}
It suffices to prove that 
$$
\log p_{k}^{\mathsf{stc}}(x) - \log p_{k+1}^{\mathsf{stc}}(x) = O(\delta_{k+1}),
$$
which is implied by
\begin{align*}
\left|\frac{\partial}{\partial\alpha_k}\log p_{k}^{\mathsf{stc}}(x)\right| &= O(\|x\|_2+1),\qquad \left|\frac{\partial}{\partial\sigma_k}\log p_{k}^{\mathsf{stc}}(x)\right| = O(\|x\|_2+1),\notag\\
\|\nabla_{{x}_{t_k}^{\mathsf{st}}}\log p_{{X}_{t_k}^{\mathsf{stc}}\mymid {X}_{t_k}^{\mathsf{st}}}({x}_{t_k}^{\mathsf{stc}}\mymid {x}_{t_k}^{\mathsf{st}})\|_2 &= O(\|{x}_{t_k}^{\mathsf{st}}\|_2+1).
\end{align*}
To this end, by the definition of $p_{k}^{\mathsf{stc}}(x)$, we have
\begin{align*}
\left|\frac{\partial}{\partial\alpha_{t_k}}\log p_{k}^{\mathsf{stc}}(x)\right| &\le \left|\frac{\partial}{\partial\alpha_k}\log p_{{X}_{t_k}^{\mathsf{stc}}, {X}_{t_k}^{\mathsf{st}}}({x}^{\mathsf{stc}}, {x}^{\mathsf{st}})\right| + \left|\frac{\partial}{\partial\alpha_k}\log p_{{X}_{t_k}^{\mathsf{st}}}( {x}^{\mathsf{st}})\right| \notag\\
&= \frac{\big|\mathbb{E}[(x^{\mathsf{stc}}-\alpha_{t_k}X_0^{\mathsf{stc}})^{\top}X_0^{\mathsf{stc}}\mymid X_{t_k}=x]\big|}{\sigma_{t_k}^2} + \frac{\big|\mathbb{E}[(x^{\mathsf{st}}-\alpha_{t_k}X_0^{\mathsf{st}})^{\top}X_0^{\mathsf{st}}\mymid X_{t_k}=x]\big|}{\sigma_{t_k}^2}\notag\\
&\qquad  + \frac{\big|\mathbb{E}[(x^{\mathsf{st}}-\alpha_{t_k}X_0^{\mathsf{st}})^{\top}X_0^{\mathsf{st}}\mymid X_{t_k}^{\mathsf{st}}=x^{\mathsf{st}}]\big|}{\sigma_{t_k}^2}\notag\\
\left|\frac{\partial}{\partial\sigma_{t_k}}\log p_{k}^{\mathsf{stc}}(x)\right| &\le 
\frac{\mathbb{E}[\|x^{\mathsf{stc}}-\alpha_{t_k}X_0^{\mathsf{stc}}\|_2^2\mymid X_{t_k}=x]}{\sigma_{t_k}^3} + \frac{\mathbb{E}[\|x^{\mathsf{st}}-\alpha_{t_k}X_0^{\mathsf{st}}\|_2^2\mymid X_{t_k}=x]}{\sigma_{t_k}^3}\notag\\
&\qquad  + \frac{\mathbb{E}[\|x^{\mathsf{st}}-\alpha_{t_k}X_0^{\mathsf{st}}\|_2^2\mymid X_{t_k}^{\mathsf{st}}=x^{\mathsf{st}}]}{\sigma_{t_k}^3} - 
\frac{|\mathcal{S}_t^{\mathsf{c}}\cup \mathcal{S}^{\mathsf{c}}|}{\sigma_{t_k}},\notag\\
\|\nabla_{{x}_{t_k}^{\mathsf{st}}}\log p_{{X}_{t_k}^{\mathsf{stc}}\mymid {X}_{t_k}^{\mathsf{st}}}({x}_{t_k}^{\mathsf{stc}}\mymid {x}_{t_k}^{\mathsf{st}})\|_2& \le \|\nabla_{{x}_{t_k}^{\mathsf{st}}}\log p_{{X}_{t_k}^{\mathsf{stc}}, {X}_{t_k}^{\mathsf{st}}}({x}_{t_k}^{\mathsf{stc}},{x}_{t_k}^{\mathsf{st}})\|_2+ \|\nabla_{{x}_{t_k}^{\mathsf{st}}}\log p_{{X}_{t_k}^{\mathsf{st}}}( {x}_{t_k}^{\mathsf{st}})\|_2\notag\\
&\le \frac{\mathbb{E}[\|x^{\mathsf{st}}-\alpha_{t_k}X_0^{\mathsf{st}}\|_2\mymid X_{t_k}=x]}{\sigma_{t_k}^2} + \frac{\mathbb{E}[\|x^{\mathsf{st}}-\alpha_{t_k}X_0^{\mathsf{st}}\|_2\mymid X_{t_k}^{\mathsf{st}}=x^{\mathsf{st}}]}{\sigma_{t_k}^2}.
\end{align*}
Therefore, we complete the proof.

\bibliographystyle{apalike}
\bibliography{bibfileDF2}

\end{document}